\documentclass[lettersize,journal]{IEEEtran}
\usepackage{amsmath,amsfonts,amsthm,amssymb, mathtools}
\usepackage{algorithmic}
\usepackage{algorithm}
\usepackage{array}
\usepackage[caption=false,font=scriptsize,labelfont=sf,textfont=sf]{subfig}
\usepackage{textcomp}
\usepackage{stfloats}
\usepackage{url}
\usepackage{verbatim}
\usepackage{graphicx}
\usepackage{cite}
\usepackage{diagbox}
\usepackage{multirow}
\usepackage[dvipsnames]{xcolor}
\usepackage{xcolor, soul}
\sethlcolor{yellow}

\begin{document}
\newtheorem{lemma}{Lemma}
\newtheorem{theorem}{Theorem}
\newtheorem{definition}{Definition}
\newtheorem{corollary}{Corollary}
\newtheorem{claim}{Claim}
\newtheorem{proposition}{Proposition}

\theoremstyle{remark}
\newtheorem{rmark}{Remark}
\newtheorem{assumption}{Assumption}

\newcommand{\R}{\mathbb{R}}
\newcommand{\C}{\mathbb{C}}
\newcommand{\onevect}{\ensuremath{\mathbf{1}}}
\newcommand{\zeros}{\ensuremath{\mathbf{0}}}
\newcommand{\ones}{\ensuremath{\mathbf{1}}}
\newcommand{\eye}{\ensuremath{\mathbf{I}}}
\newcommand{\tr}{\ensuremath{\text{tr}}}
\newcommand{\ts}{\textsuperscript}
\newcommand{\poly}{\ensuremath{\text{poly}}}
\newcommand{\vectorize}[1]{\ensuremath{\textrm{vec}\left(#1\right)}}


\newcommand{\G}{\ensuremath{\mathcal{G}}}
\newcommand{\V}{\ensuremath{\mathcal{V}}}
\newcommand{\E}{\ensuremath{\mathcal{E}}}
\newcommand{\Lap}{\ensuremath{\mathbf{L}}}
\newcommand{\LG}{\ensuremath{\Lap_\G}}
\newcommand{\LGs}{\ensuremath{\Lap_{\G_s}}}
\newcommand{\W}{\ensuremath{  \mathbf{W} }}
\newcommand{\D}{\ensuremath{  \mathbf{D} }}
\newcommand{\crd}[1]{\ensuremath{  \left|#1\right| }}
\newcommand{\uv}{\ensuremath{\mathbf{u}}}
\newcommand{\UG}{\ensuremath{\mathbf{U}_\G }}
\newcommand{\UGs}{\ensuremath{\mathbf{U}_{\G_s} }}
\newcommand{\LamG}{\ensuremath{\mathbf{\Lambda}_ \G}}
\newcommand{\lamG}{\ensuremath{\mathbf{\lambda}_ \G}}
\newcommand{\covMat}{\ensuremath{{\mathbf{C}}}}

\newcommand{\x}{\ensuremath{\mathbf{x}}}
\newcommand{\xk}{\ensuremath{\x_k}}
\newcommand{\xm}{\ensuremath{\x_m}}
\newcommand{\y}{\ensuremath{\mathbf{y}}}
\newcommand{\Cx}{\ensuremath{{\covMat_\x}}}
\newcommand{\hatCx}{\ensuremath{{\hat{\covMat}_\x}}}
\newcommand{\estCx}{\ensuremath{{{\covMat}^*_\x}}}
\newcommand{\estC}{\ensuremath{{{\covMat}^*}}}
\newcommand{\Cxk}{\ensuremath{{\covMat_{\xk}}}}
\newcommand{\Ckm}{\ensuremath{{\covMat_{\x_k \x_m}}}}
\newcommand{\Cdistx}{\ensuremath{{\rho_\x}}}
\newcommand{\w}{\ensuremath{\mathbf{w}}}
\newcommand{\psdCone}{\ensuremath{{\mathbb{S}}}}
\newcommand{\gPar}{\ensuremath{\mathcal{P}}}
\newcommand{\tv}{\ensuremath{T}}
\newcommand{\hatz}{\ensuremath{\hat{\mathbf{z}}}}
\newcommand{\z}{\ensuremath{\mathbf{z}}}


\newcommand{\Gk}{\ensuremath{\mathbf{G}_k}}
\newcommand{\Gm}{\ensuremath{\mathbf{G}_m}}
\newcommand{\invGk}{\ensuremath{\Gk^\dag}}
\newcommand{\invGm}{\ensuremath{\Gm^\dag}}
\newcommand{\g}{\ensuremath{\mathbf{g}}}
\newcommand{\gk}{\ensuremath{\mathbf{g}_k}}
\newcommand{\h}{\ensuremath{\mathbf{h}}}
\newcommand{\hk}{\ensuremath{\mathbf{h}_k}}
\newcommand{\Hmat}{\ensuremath{\mathbf{H}}}
\newcommand{\M}{\ensuremath{\mathbf{M}}}
\newcommand{\Smat}{\ensuremath{\mathbf{S}}}
\newcommand{\Gmat}{\ensuremath{\mathbf{G}}}
\newcommand{\bv}{\ensuremath{\mathbf{b}}}
\newcommand{\bk}{\ensuremath{\bv_k}}
\newcommand{\Gam}{\ensuremath{\mathbf{\Gamma}}}
\newcommand{\B}{\ensuremath{\mathbf{B}}}
\newcommand{\mub}{\ensuremath{\gamma}}
\newcommand{\epsh}{\ensuremath{\epsilon}}
\newcommand{\Kcut}{\ensuremath{\kappa_C}}


\newcommand{\Amat}{\ensuremath{\mathbf{A}}}
\newcommand{\Bmat}{\ensuremath{\mathbf{B}}}
\newcommand{\Dmat}{\ensuremath{\mathbf{D}}}
\newcommand{\Emat}{\ensuremath{\mathbf{E}}}
\newcommand{\Umat}{\ensuremath{\mathbf{U}}}
\newcommand{\Vmat}{\ensuremath{\mathbf{V}}}
\newcommand{\Ymat}{\ensuremath{\mathbf{Y}}}
\newcommand{\Zmat}{\ensuremath{\mathbf{Z}}}
\newcommand{\uvect}{\ensuremath{\mathbf{u}}}
\newcommand{\vvect}{\ensuremath{\mathbf{v}}}
\newcommand{\cvect}{\ensuremath{\mathbf{c}}}

\title{Locally Stationary Graph Processes}

\author{Abdullah Canbolat and Elif Vural
\thanks{The authors are with the Dept. of Electrical and Electronics Engineering,
METU, Ankara. This work was supported by the Scientific and Technological
Research Council of Turkey (T\"UB\.ITAK) under grant 120E246.}
\thanks{}}

\markboth{}%
{}

\maketitle

\begin{abstract}

Stationary graph process models are commonly used in the analysis and inference of data sets collected on irregular network topologies. While most of the existing methods represent graph signals with a single stationary process model that is globally valid on the entire graph, in many practical problems, the characteristics of the process may be subject to local variations in different regions of the graph. In this work, we propose a locally stationary graph process (LSGP) model that aims to extend the classical concept of local stationarity to irregular graph domains. We characterize local stationarity by expressing the overall process as the combination of a set of component processes such that the extent to which the process adheres to each component varies smoothly over the graph. We propose an algorithm for computing LSGP models from realizations of the process, and also study the approximation of LSGPs locally with WSS processes. Experiments on signal interpolation problems show that the proposed process model provides accurate signal representations competitive with the state of the art. 

\end{abstract}


\section{Introduction}

Graph signal models provide effective solutions for analyzing data collections acquired on irregular network topologies in many modern applications. The probabilistic modeling of graph signals has been a topic of interest in the recent years. In particular, the concept of stationary random processes has been extended to graph domains in several recent works \cite{Girault2015,Marques2017,Perraudin2017,Loukas2019}. The wide sense stationary  (WSS) graph process models in these works are based on the assumption that the correlations between different graph nodes can be captured via a single global model coherent with the topology of the whole graph. Meanwhile, in many real-world problems,  the nature of the interactions between nearby graph nodes may vary throughout the graph, e.g., in a social network the correlation patterns within a group of users may show substantial diversity among different communities of the network. In this paper, we propose a new graph signal model that extends the concept of local stationarity to graph processes, hence permitting the statistics of the process to be locally-varying.
 
While local stationarity is a well-studied subject in classical time series analysis  \cite{Silverman1957, Dahlhaus2012}, a comprehensive and detailed treatment of local stationarity remains absent in the graph signal processing literature. Several previous works have briefly touched upon the notion of local stationarity within the context of graph processes. The studies in \cite{Hasanzadeh2019} and \cite{Scalzo2023} consider a piecewise stationary process model, which is linked to but not the same as local stationarity; while the works \cite{Girault2017,Girault2017a} propose a definition of local power spectrum for graph processes, however, without introducing a locally stationary process model. In this paper, we propose a locally stationary graph process (LSGP) model where the overall graph process is expressed through the combination of a set of individual stationary graph processes, each of which is generated through a different spectral kernel. The overall process value at each graph node is related to each component process through a smoothly varying membership function. This model also allows the definition of the vertex-frequency spectrum for graph processes, which is analogous to the concept of time-frequency spectrum for time processes. We then propose an algorithm for computing locally stationary graph process models from a set of possibly partially observed realizations of the process, which is formulated as learning the membership functions and the spectral kernels identifying the component processes.

We then proceed to a connected problem and study whether LSGPs on large graphs can be locally approximated with simpler models. We theoretically show that an LSGP can be locally approximated with a WSS graph process on a subgraph, provided that its membership functions are well-localized on the subgraphs and the kernels generating the component processes are sufficiently separated from each other in the spectral domain. These theoretical findings motivate a second algorithm that partitions a given graph into subgraphs by meeting the above conditions and computes a local approximation of the original process on each subgraph.  Experiments on synthetic and real data show that the  proposed algorithms lead to quite competitive performance compared to recent approaches in graph signal interpolation applications.

The rest of the paper is organized as follows: In Section \ref{sec_rel_work}, we discuss the related literature. In Section \ref{sec:preliminary}, we overview some preliminary concepts in graph signal processing. In Section \ref{sec:vlsp_intro}, we present our LSGP model, and propose a method for learning LSGPs in Section \ref{sec:learning}. In Section \ref{sec:wss_lsp}, we study the local approximation of LSGPs with WSS processes.  We present our experimental results in Section \ref{sec_experiments}, and conclude in Section \ref{sec_concl}.

\section{Related Work}
\label{sec_rel_work}

With the emergence of the theory of graph signal processing (GSP) in the recent years, it has been possible to extend classical signal processing concepts such as the Fourier transform, filtering, stationarity and sampling  to irregular graph domains \cite{Shuman2012,Sandryhaila2013,Perraudin2017,Tanaka2020}. The definition of frequency analysis tools on graphs has enabled the construction of graph filters via functions of shift operators \cite{Girault2015isometric,Shuman2016vertex}, following which  the MA, AR, and ARMA filter models widely used in signal processing have been generalized to graph domains as well \cite{Marques2017}, \cite{Liu2019}. The extension of  stochastic process models to graph domains has been the focus of several studies in the last few years. The works in \cite{Girault2015, Marques2017, Perraudin2017, Loukas2019}  have proposed to generalize the notion of wide sense stationarity (WSS) to irregular topologies through the joint diagonalizability of the process covariance matrix with the graph shift operator. The consequent extension of AR, MA, and ARMA process models to graph domains have found efficient usage in problems related to the inference of graph signals \cite{IsufiLPL19,  Marques2017, MeiM17}. 

Most graph process models in the GSP literature are based on a rather strict global stationarity assumption, and non-stationary models with locally-varying statistics as in our work are relatively uncommon. Several studies have considered node-varying graph filters \cite{Hua2018, Gama2022}, which in fact correspond to non-stationary graph filter kernels. A non-stationary graph process model is briefly hinted at in \cite{Segarra2016}. The piecewise stationary process models proposed in \cite{Hasanzadeh2019} and \cite{Scalzo2023}  are some of the other efforts towards dealing with non-stationarity in graph signals. An important difference between these works and ours is that, they present a more restrictive approach as the global graph is broken into subgraphs and each subgraph is constrained to admit a different individual model with its own frequency content. In contrast, in our model the spectral contents of the component processes are valid on the whole graph and are blended smoothly. Even when partitioning a large graph, we do not disregard this holistically defined frequency content, but rather make use of the assumption that a different spectral kernel is dominant on each subgraph. 

The term locally stationary process (LSP) has previously been used in the works of Girault et al.~\cite{Girault2017,Girault2017a}. An intrinsic stationarity definition at local and global scales has been proposed in \cite{Serrano2018}  based on the local graph variogram. However, these studies briefly present a local spectrum definition for graph processes rather than putting forward a stochastic process model. In our work, we elaborate on the concept of local spectrum through our \textit{vertex-frequency spectrum} definition. This is somewhat related to the vertex-frequency analysis concepts introduced in some earlier works \cite{Segarra2016, Shuman2016vertex, Stankovic2020, Gama2022}.  However, the scopes of these works are limited to vertex-frequency kernels and operators, and they do not focus on stochastic graph processes in particular.

While the aforementioned studies consider local stationarity in graph settings, in a wider scope, the notion of local stationarity has been investigated in the classical signal processing literature first \cite{Silverman1957}, \cite{Dahlhaus2012}. Dahlhaus defines a locally stationary process as a time-varying MA process with smoothly varying coefficients over time \cite{Dahlhaus2012}. Our LSGP model, which describes local stationarity through smoothly varying membership functions, is rather aligned with this definition and extends it to irregular finite-dimensional domains in some way. As for the inference of second-order statistics of non-stationary time processes, current methods mostly rely on Silverman's model \cite{SilvermanWahlberg2007,Mallat1998,Hansson2005} or Dahlhaus's model \cite{Dahlhaus1997,Donoho1996,Roueff2019}. Locally varying statistics are often captured through data tapers and windowing techniques in the time domain  \cite{Pitton2000,Palma2010,Dahlhaus2012}.  However, these works address the non-stationary covariance estimation problem within the traditional time series setting, while we consider the problem over irregular graph topologies.

In this work, we demonstrate the proposed methods in signal interpolation applications. A large body of solutions exist for this problem, such as reconstruction via regularization \cite{Jung2019}, deep algorithm unrolling methods with learnable regularization  parameters \cite{Nagahama2022}, and graph neural network methods \cite{WuPCLZY21}.  Graph attention networks \cite{Vaswani2017}  bear some resemblance to our study on a conceptual level, as the locally defined attention coefficients \cite{Velickovic2018,Brody2021} can be compared to our node-varying membership functions. Lastly, a preliminary version of our study was presented in \cite{Canbolat2022}, which has been significantly extended in the current paper by detailing the theoretical results and experimental evaluation.

\section{Preliminaries}
\label{sec:preliminary}

We write matrices with boldface capital letters (e.g.~$\mathbf{A}$), vectors with boldface lowercase  letters (e.g.~$\mathbf{x}$), and sets with calligraphic letters (e.g.~$\mathcal{V}$). Indexing is shown in parentheses (e.g.~$\mathbf{A}(i,j)$). The notation $\| \cdot\|_F$ stands for the Frobenius norm of a matrix, $\circ $ refers to the Hadamard (elementwise) product, $(\cdot)^T$ denotes the transpose, and $(\cdot)^{\dag}$ denotes the pseudo-inverse of a matrix. The notation $\vectorize{\cdot}$ represents the vectorized form of a matrix. Vectors and matrices consisting of only zeros and only ones are denoted as $\zeros_{N} \in \R^{N\times 1}$, $\ones_{N} \in \R^{N \times 1}$ and  $\zeros_{N \times M} \in \R^{N\times M}$, $\ones_{N \times M} \in \R^{N \times M}$, respectively.  $\eye_N \in \R^{N\times N}$ denotes the identity matrix. Curly inequalities for matrices represent element-wise inequality, e.g.,  $\Amat \preccurlyeq \Bmat$ means $\Amat (i,j ) \leq \Bmat(i,j) $ for all entries $(i,j)$. The notation $| \cdot |$ stands for absolute value for scalars; and element-wise absolute value when the argument is a matrix, i.e., $|\Amat| (i,j) \triangleq | \Amat(i,j) |$ for a matrix $\Amat$. When its argument is a set, $| \cdot |$ denotes cardinality. The covariance matrix of a random vector $\x$ is shown as $\covMat_{\x}$.

\subsection{Graph Signal Processing}

In this work, we consider signals defined on an undirected graph $\G$ consisting of a single connected component. The topology of the graph $\mathcal{G} = (\mathcal{V}, \mathcal{E},\W)$ is defined by the vertex set $\mathcal{V}$, edge set $\mathcal{E}$, and edge weight matrix $\W \in \mathbb{R}^{N \times N}$, where $N = \crd{\mathcal{V}}$ is the number of graph nodes (vertices). We write $i \sim j$ when $\W(i,j) \neq 0$, i.e., the nodes $i,j \in \V$ are neighbors.  A graph signal $x$ is defined as a vertex function $x : \mathcal{V} \rightarrow \mathbb{R}$, which can alternatively be represented as a vector $\x \in \mathbb{R}^N$. The normalized graph Laplacian is defined as $\LG= \D^{-1/2} (\D-\W) \D^{-1/2}$, where $\D$ is the diagonal degree matrix given by $\D(i,i)=\sum_{j}\W(i,j)$.  The graph Laplacian has an eigenvalue decomposition $\LG = \UG \LamG \UG^T$, where the columns $\uv_i \in \mathbb{R}^N$ of $\UG = [\uv_1 \, \uv_2 \, \dots \, \uv_N]$ are the Fourier modes of the graph,  and the diagonal entries of $\LamG$ are regarded as graph frequencies. Also, we stack graph frequencies into a vector $\lamG \in \mathbb{R}^N$ such that $\lambda_\mathcal{G}(i) = \LamG(i,i)$. 

Filtering is defined via frequency domain functions called as kernels on graphs. An input graph signal $\x \in \mathbb{R}^N$ can be filtered with a graph filter kernel $h : \mathbb{R} \rightarrow \mathbb{R}$ as \cite{Shuman2012}

\begin{equation}
\y = \UG h(\LamG) \UG^T \x
\end{equation}
where $ h(\LamG)$ is a diagonal matrix whose entries are obtained by evaluating the kernel function $h(\cdot)$ at the graph frequencies $\lamG(i) $; and $\y$ is the output signal. Hence, a graph filter can be represented as a matrix $h(\LG) = \UG h(\LamG) \UG^T$.

\subsection{Wide Sense Stationary Processes on Graphs}

In classical time series analysis, a wide sense stationary (WSS) process is defined as a constant-mean stochastic process such that the covariance between the process values at two distinct time instants depends only on their time difference. In recent studies, the concept of wide sense stationarity has been extended to graph domains as follows  \cite{Perraudin2017}.

\begin{definition}[Wide Sense Stationary Process]
\label{defn_wssp}
A stochastic graph signal $\x \in \mathbb{R}^N$ on graph $\mathcal{G}$ is called WSS if the following conditions are satisfied \cite{Perraudin2017}:

\begin{itemize}
\item The mean of the process is $E[\x] = c \, \mathbf{1}_N$ for a constant $c\in \mathbb{R}$. 
\item The covariance matrix $\Cx$ of the process is in the form of a graph filter $\Cx = \UG h(\LamG) \UG ^T$, where $h$ is a positive graph kernel.
\end{itemize}
\end{definition}

The second condition above imposes the covariance between any two graph nodes to be given by the response of a graph kernel $h(\lambda)$ localized (centered) at one node and evaluated at the other node. Therefore, a graph process is WSS if there exists a graph kernel $h(\lambda)$ that can fully characterize its correlation pattern on the whole graph.

\subsection{Time-Varying Moving Average Processes}

In this section we give a brief overview of locally stationary processes in classical time series analysis  \cite{Dahlhaus2012}. Locally stationary time processes have spectral characteristics that are close to a stationary process when restricted to a fixed time interval. The validity of the approximation of a locally stationary process by a globally stationary one decreases as one moves away from the anchor time instant where the approximation is made. A locally stationary process is in turn a time-varying process.

Let $x_{t,T} $ be a finite-dimensional random process, where $t \in \mathbb{Z}$ is a time instant and $T \in \mathbb{Z}^+$ is the total number of time instants in the process.  Also, let $w_{t}$ represent the value of a white noise process at time instant $t$. The process $x_{t,T}$ is then called  a linear locally stationary process if it can be expressed as a time-varying $MA(\infty)$ process as \cite{Dahlhaus2012}
\begin{equation}
\label{eq:time_lsp}
x_{t,T} = \sum_{i = -\infty}^{\infty} a_{t,T}(i) \ w_{t-i} + \mu\left(\frac{t}{T}\right).
\end{equation} 
Here $\mu:[0,1] \rightarrow \mathbb{R}$ is the mean function of the process and $a_{t,T}:\mathbb{Z}\rightarrow \mathbb{R}$ are time-varying moving average (MA) filter coefficients. For each time shift value $i$, the filter coefficient $a_{t,T}(i) $  is approximately given through a kernel $f$ as  $a_{t,T}(i) \approx f(\frac{t}{T},i)$ where $f(.,i):[0,1]\rightarrow \mathbb{R}$.

The process $x_{t,T}$ has the spectral decomposition \cite{Dahlhaus2012}
\begin{equation}
\label{eq:time_lsp_spec}
x_{t,T} = \mu\left(\frac{t}{T}\right) + \frac{1}{\sqrt{2\pi}} \int_{-\pi}^{\pi} A_{t,T}(\lambda) e^{j\lambda t} d\xi(\lambda)
\end{equation}
where $\xi(\lambda)$ is an orthogonal increment process and $A_{t,T}(\lambda) \approx F(\frac{t}{T},\lambda)$, with $A_{t,T}$ and $F(\frac{t}{T},.)$ being the frequency domain representations of the filters generated by $a_{t,T}$ and $f(\frac{t}{T},.)$, respectively. These are also called as the time-frequency spectrum of the process.

\section{Locally Stationary Graph Processes}
\label{sec:vlsp_intro}

\subsection{Proposed Process Model}
We propose a locally stationary graph process model of the form 
\begin{equation}
    \x = \sum_{k=1}^K \Gk \UG h_k(\LamG) \UG^T  \w
    \label{eq:nvkernel_graphvalid}
\end{equation}
where the process $\x\in \mathbb{R}^N$ of interest is obtained by filtering a white noise process $\w \in \mathbb{R}^N$ of unit variance. Here, the overall filtering operation on the white process is expressed in terms of $K$ graph filters of the form $\UG h_k(\LamG) \UG^T$, each of which is defined by the kernel $h_k(\lambda)$. For each $k$, the term $\UG h_k(\LamG) \UG^T \w$ is an individual WSS graph process, which forms a ``component'' of the locally stationary graph process model. The matrices $\Gk \in \mathbb{R}^{N \times N}$ are diagonal matrices representing the ``membership'' of the overall process $\x$ with respect to each component process: For each $n$-th node, the $n$-th diagonal entry $\Gk(n,n)$ indicates how much the $k$-th component process contributes to the overall process $\x$. Note that an analogy can be drawn between \eqref{eq:nvkernel_graphvalid} and the classical locally stationary time process model in \eqref{eq:time_lsp}, such that the filter coefficients $a_{t,T}$ in \eqref{eq:time_lsp} would correspond to the graph filters $\UG h_k(\LamG) \UG^T$ in \eqref{eq:nvkernel_graphvalid}. While the local behavior of the process model is captured by the dependence of the filters $a_{t,T}$ on the time instant $t$ in  \eqref{eq:time_lsp}, in  \eqref{eq:nvkernel_graphvalid}  it is captured by the membership matrices $\Gk$ whose $n$-th entries define the local behavior of the process at the $n$-th graph node. In our process definition, memberships are allowed to take negative values as well as positive ones, which permits a more compact model representation.

In  \eqref{eq:nvkernel_graphvalid} we do not make any general assumptions about the spectral characteristics of the graph filters $h_k(\cdot)$ defining the model. In return, in order to provide the model with a meaningful notion of local stationarity, one must ensure that the statistics of the process change smoothly between neighboring nodes. We achieve this by imposing that the membership functions of the process vary slowly on the graph: Defining the vectorized form $\gk \in \mathbb{R}^N$ of the diagonal matrices $\Gk$ such that $\gk(n) = \Gk (n,n)$, the vector $\gk $ can be regarded as a membership function on the graph, which identifies how much each graph node conforms to the $k$-th component process model. An upper bound on the term $\gk^T \LG \gk$ then determines the rate at which the membership function $\gk$ varies on the graph. This motivates the following definition of locally stationary graph processes (LSGP): 
\begin{definition}[Locally Stationary Graph Process]
\label{def:local_filter}
A stochastic graph signal $\x$ is called a locally stationary graph process (LSGP) with variation rate $C$, if it can be expressed as
\begin{equation}
    \x = \sum_{k=1}^K \Gk \UG h_k(\LamG) \UG^T \w
    \label{eq_locstat_model}
\end{equation}
such that $\gk^T \LG  \gk \leq C$ and $\|  h_k(\LamG) \|_F^2=1$, for $k=1, \dots, K$. 
\end{definition}

Let $\hk \in \mathbb{R}^N$ denote the vectorized form of the diagonal entries of  $h_k(\LamG)$ such that $\hk(n)=h_k(\LamG)(n,n)$.  In the above definition, the normalization condition  $\|   h_k(\LamG) \|_F^2=1$ (or equivalently $\| \hk \|=1$) is imposed on the kernels in order to avoid the scale ambiguity arising from the product of the membership values $\gk$ with the amplitudes of the kernels $\hk$.

\subsection{Vertex-Frequency Spectrum}
\label{sec:vf_spec}

In classical time series analysis, a locally stationary time process is associated with a time-frequency spectrum $A_{t,T}$ as in \eqref{eq:time_lsp_spec}.  We now show that, in analogy, a vertex-frequency spectrum can be defined for LSGPs.

\begin{theorem}[Vertex-Frequency Spectrum]
A locally stationary graph process with variation rate $C$ can be expressed as $\x=\Hmat \w$, where the filter  $\Hmat = \sum_{k=1}^K \Gk \UG h_k(\LamG) \UG^T $ can be written in the form
\begin{equation}
    \Hmat = (\UG \circ \M)  \UG^T
    \label{eq:compact_local}
\end{equation}
with $\M= \sum_{k=1}^K  \gk \hk^T$.  Here the matrix $\M \in \mathbb{R}^{N\times N}$ identifies the vertex-frequency spectrum of the process, such that $\M(n,i)$ gives the local spectrum for the $i$-th graph frequency $\lamG(i)$ at graph node $n$. For each graph frequency $\lamG(i)$, the local spectrum $\M(\cdot, i)$ has bounded variation over the whole graph when regarded as a graph signal, such that the total variation of the local spectra is upper bounded as
\[ \textup{\tr}(\M^T  \LG \M)\leq K^2 C.\]
\label{th:vertex_freq_spec}
\end{theorem}

The proof of Theorem \ref{th:vertex_freq_spec} is given in Appendix B. An immediate implication of Theorem \ref{th:vertex_freq_spec} is that the vertex-frequency spectrum $\M$ of an LSGP is the counterpart of the time-frequency spectrum $A_{t,T}$ for locally stationary time processes; hence, $\M$ defines a vertex-varying spectrum for LSGPs. Noticing that a WSS graph process is always an LSGP due to Definition \ref{defn_wssp}, it is easy to verify that the vertex-frequency spectrum $\M$ of a WSS graph process takes the trivial form of a rank-1 matrix with identical row vectors.

\subsection{Extension and Restriction of LSGPs}
\label{sec:extension_lsp}

In the analysis of graph signals, a question of interest is whether a graph signal model can be extended to larger graphs in a mathematically consistent way, e.g., due to the addition of new nodes. Similarly, one may want to restrict the model to a subgraph of the original graph as well, e.g., due to node removal. In Appendix C, we discuss the extension and restriction of LSGPs. We show that the family of LSGPs are sufficiently comprehensive so as to permit their extension and restriction to larger and smaller graphs by still remaining in the class of LSGP models.

 \section{Learning LSGP Models}
 \label{sec:learning}

In Section \ref{sec:vlsp_intro}, we proposed a model for locally stationary graph processes and examined some of its properties. We next study the problem of learning LSGPs in this section. We propose a framework for computing an LSGP model from a set $\{\x^l \}_{l=1}^L$  of $L$ graph signals that are considered as realizations of a process $\x$. Our aim is then to compute a process model of the form \eqref{eq_locstat_model}.  In many practical applications, the graph signals of interest can be only partially observed; hence, we address a flexible setting where the given realizations $\x^l \in \mathbb{R}^{N}$ of the process may contain missing values.

We consider that the value $\x^l(i)$ of the realization is known at a subset of the graph nodes $i \in \{1, \dots, N\}$. Let us denote as $I^l$ the index set of the nodes $i$ where $\x^l(i)$ is known. In order to compute an LSGP model, we first obtain an estimate $\hatCx$ of the covariance matrix $\Cx$ of the process from the available observations, which can be chosen as the sample covariance estimate. Our approach is then based on learning a process model by fitting the parameters of the LSGP in  \eqref{eq_locstat_model} to the covariance estimate $\hatCx$. The covariance matrix of the process $\x=\Hmat \w$ as defined in Theorem \ref{th:vertex_freq_spec} is 
\begin{equation}
\label{eq:covariance}
\Cx = E[\x \x^T] = \Hmat \Hmat^T
\end{equation}
assuming that $\w$ is a zero-mean white process with unit variance. We then propose to learn the process model by solving the following optimization problem:

\begin{equation}
\label{eq:optimization_coarse}
\begin{split}
\underset{\{\g_k\}_{k = 1}^{K},\{\h_k\}_{k = 1}^{K}}{\text{minimize}} ||\hatCx - \Hmat \Hmat^T ||^2_F+ \mu_1 \, \tr(\Gmat^T \LG \Gmat) \\
\textrm{subject to } \Gmat = [\g_1 \ \dots \ \g_K], \quad 
\Hmat = \sum_{k=1}^K \Gk  \UG h_k(\LamG) \UG^T.\\
\end{split}
\end{equation}

The above optimization problem is motivated by the local stationarity definition in Definition \ref{def:local_filter}. The first term in the objective function enforces the covariance matrix $\Cx=\Hmat \Hmat^T$ of the learnt process model to fit the initial empirical estimate $\hatCx$. Gathering all membership functions $\g_k$ in the matrix $\Gmat \in \R^{N \times K}$, the second term aims to reduce the variation rate $C$ of the locally stationary process as much as possible, so that the process statistics change smoothly over the graph.

The optimization problem in \eqref{eq:optimization_coarse} is difficult to solve as it is nonconvex and it involves $2NK$ optimization variables. In order to put the problem in a more tractable form, we first constrain the graph kernels $h_k(\cdot)$ to be polynomial functions, which is a common choice due to its various convenient properties such as good vertex-domain localization \cite{Segarra2016}. The entries of the spectrum vector $\h_k$ are then of the form
\begin{equation}
\label{eq:polynomial_kernel}
\h_k(i) = \sum_{q = 0}^{Q-1} b_{q,k} \lamG^q (i)
\end{equation}
where $\lamG^q (i)$ denotes the $q$-th power of the $i$-th graph frequency $ \lamG(i)$; and $b_{q,k} $ are polynomial coefficients.  Let us define the polynomial coefficient vector $\bk = [b_{0,k} \ b_{1,k} \ \dots \ b_{Q-1,k}]^T \in \R^{Q}$ of the $k$-th kernel, as well as the overall coefficient vector  $\bv = [\bv_1^T \ \bv_2^T \ \dots \ \bv_K^T]^T \in \R^{QK}$. Let also $\g = [\g_1^T \ \g_2^T \ \dots \ \g_K^T]^T \in \R^{NK}$ denote the vectorized form of the matrix $\Gmat$. Then we propose to relax the nonconvex problem \eqref{eq:optimization_coarse} into a convex one, by introducing the new optimization variables
\begin{equation}
\label{eq_defn_B_Gamma}
 \Gam=\g \g^T \in \R^{NK \times NK}, \qquad \qquad \B= \bv \bv^T \in \R^{QK \times QK}.
\end{equation}
In Appendix D, we show that the term $\Hmat \Hmat^T$ can be directly expressed as a function of $\Gam$ and  $\B$; and the term $\tr(\Gmat^T \LG \Gmat)$ can be written as a linear function of $\Gam$. We can thus define the functions $f_1(\Gam, \B)=||\hatCx - \Hmat \Hmat^{T}||^2_F$ and $f_2(\Gam)=\tr(\Gmat^T \LG \Gmat)$ representing the first and the second terms of the objective in \eqref{eq:optimization_coarse}. Meanwhile, for the decompositions in \eqref{eq_defn_B_Gamma} to be valid, the matrices $\Gam$ and $\B$ need to be rank-1 and positive semi-definite. We thus propose to relax the original problem \eqref{eq:optimization_coarse} into the optimization problem
\begin{equation}
\label{eq:optimization_poly}
\begin{split}
\underset{\Gam,\B}{\text{minimize}} \ f_1(\Gam,\B)+ \mu_1  f_2(\Gam)+ \mu_2 \tr(\B) + \mu_3 \tr(\Gam) \\
\textrm{subject to } \Gam \in \psdCone_+^{NK},  \ \ \B \in \psdCone_+^{QK}
\end{split}
\end{equation}
where $\psdCone_+^{n}$ denotes the cone of $n\times n$ positive semi-definite matrices and $\mu_1, \mu_2, \mu_3$ are positive weight parameters. The terms $\tr(\B)$  and $ \tr(\Gam) $ in the objective function give the nuclear norms of the positive semi-definite matrices $\B$ and $\Gam$, aiming to minimize their ranks by providing a relaxation of the rank-1 constraint.

The term $f_1(\Gam,\B)$ in \eqref{eq:optimization_poly} is quadratic individually in $\Gam$ and in $\B$, and the term $f_2(\Gam)$ is linear in $\Gam$, whose explicit forms can be found in Appendix D. While the overall objective function is not jointly convex in $\Gam$ and $\B$, it is convex in only $\Gam$ and only $\B$. We propose to minimize the objective iteratively with an alternating optimization procedure. In each iteration, first fixing $\Gam$, and then fixing $\B$, the problem \eqref{eq:optimization_poly} can be rewritten respectively in the forms \eqref{eq:explicit_fd_b} and \eqref{eq:explicit_fd_gamma}. Here the positive semi definite matrices $\mathbf{Q}_{\Gam}, \mathbf{Q}_{\B} $ and the vectors $\cvect_{\Gam},  \cvect_{\B}$ depend respectively on the fixed variables $\Gam$ and $\B$ of each problem, as well as the problem constants $\LG$, $\hatCx $, $\mu_1$, $\mu_2$, and $\mu_3$.

\begin{subequations}

  \begin{equation}
    \label{eq:explicit_fd_b}
    \begin{split}
      \underset{\B}{\text{minimize}} \ \vectorize{\B}^T \mathbf{Q}_{\Gam} \vectorize{\B} + \cvect_{\Gam}^T \vectorize{\B} \\
      \textrm{subject to } \B \in \psdCone_+^{QK}
      \end{split}
  \end{equation}
  
  \begin{equation}
    \label{eq:explicit_fd_gamma}
    \begin{split}
      \underset{\Gam}{\text{minimize}} \ \vectorize{\Gam}^T \mathbf{Q}_{\B} \vectorize{\Gam} + \cvect_{\B}^T \vectorize{\Gam} \\
      \textrm{subject to } \Gam \in \psdCone_+^{NK}
      \end{split}
  \end{equation}

\end{subequations}

The objectives in \eqref{eq:explicit_fd_b} and \eqref{eq:explicit_fd_gamma} are quadratic and convex in $\B$ and $\Gam$, respectively. These problems can be solved with semi-definite programming (SDP) by linearizing the quadratic functions \cite{Toh2006, Yang2015, Yurtsever2021} or solving the quadratic semidefinite programming problem directly with specialized solvers \cite{Li2018}. Once  $\Gam$ and $\B$ are found by solving  \eqref{eq:explicit_fd_b} and \eqref{eq:explicit_fd_gamma} in an alternating way, the model parameter vectors $\g$ and $\bv$ are computed through their rank-1 approximations via SVD. The filter $\Hmat$ can then be computed from $\g$ and $\bv$ using the relations in \eqref{eq:compact_local} and \eqref{eq:polynomial_kernel}. The proposed method for learning LSGP models is summarized in Algorithm  \ref{alg:lsp_learning}.

 \begin{algorithm}[t]
  {\small
    \caption{\small Learning LSGP models}\label{alg:lsp_learning}
      \hspace*{\algorithmicindent} \textbf{Input:} Graph $\G$, initial covariance estimate $\hatCx$, number of process components $K$ \\
   \hspace*{\algorithmicindent} \textbf{Output:} LSGP model parameters $\gk$, $\hk$, $\M$, $\Hmat$\\
    \vspace{-0.2cm}\\
    \textbf{procedure} $\text{\texttt{learnLSGP}}(\G, \hatCx, K)$ \\
    \vspace{-0.5cm}
    \begin{algorithmic}[1]
    \STATE \hspace{0.5cm}Estimate $\Gam$ and $\B$ by solving \eqref{eq:explicit_fd_b}-\eqref{eq:explicit_fd_gamma} alternatingly
    \STATE \hspace{0.5cm}Find $\g$ and $\bv$ with rank-1 decompositions of $\Gam$ and  $\B$
    \STATE \hspace{0.5cm}Compute kernels $\{\hk\}_{k = 1}^K$ using \eqref{eq:polynomial_kernel}
    \STATE \hspace{0.5cm}From $\{\g_k\}_{k=1}^K$ and $\{\h_k\}_{k = 1}^K$, compute $\M = \sum_{k=1}^K \g_k \h_k^T$
    \STATE \hspace{0.5cm}Estimate $\Hmat$ from \eqref{eq:compact_local} \\
     \hspace*{0.5cm}\textbf{return}  $\{\gk\}$, $\{\hk\}$, $\M$, $\Hmat$
         \end{algorithmic} }
\end{algorithm}

In Appendix E, we present a complexity analysis of Algorithm \ref{alg:lsp_learning}, which can be summarized as $O(\poly(N) K^2)$ with $\poly(\cdot)$ denoting polynomial complexity.

 \section{Locally Approximating LSGPs with WSS Processes}
 \label{sec:wss_lsp}
 
 In Section \ref{sec:learning}, we have proposed a method for learning an LSGP model from realizations of the process. Data statistics on large networks are likely to vary gradually throughout the network, which justifies the assumption of local stationarity. On the other hand, a common problem in graph signal processing is the potential complexity of learning graph signal models over a whole network as the network size increases. Motivated by these observations, in this section we explore a constructive approach for handling local stationarity in large graphs. Our approach is based on partitioning a given graph $\G$ into a set of $K$ disjoint subgraphs $\{\G_k\}_{k = 1}^K$. We consider the LSGP model in \eqref{eq_locstat_model} and express the process $\x$ on the original graph $\G$ as $ \x = \sum_{k=1}^K \Gk \x_k$, where each component process 
 \[
 \x_k = \UG h_k(\LamG) \UG^T \w
 \]
 is a WSS graph process. We then would like to partition $\G$ such that the process $\x$ can be approximated through only the component process $\x_k$ over each subgraph $\G_k$. The feasibility of such an approximation of course depends on the specific LSGP at hand; in particular, the characteristics of the membership functions $\gk$ and the kernels $\hk$. In Section \ref{sec:cov_bounds}, we explore the conditions on  $\gk$ and $\hk$ that permit accurate local approximations of $\x$ with WSS processes as above. We then study the covariance matrix $\Cx$ of the process $\x$ and show that, under these conditions, $\x$ is weakly correlated across different subgraphs $\G_k$. Finally, these theoretical findings give rise to an algorithm in Section \ref{sec:subgraph_algorithm} for suitably partitioning a graph $\G$ based on the covariance of $\x$, and locally approximating $\x$ with a WSS process $\x_k$ on each subgraph.

 \subsection{Covariance Analysis of LSGPs}
 \label{sec:cov_bounds}
Let $\G= (\mathcal{V}, \E, \W)$ be a graph with $N$ nodes, and let $\{\G_k\}_{k = 1}^K$ be disjoint subgraphs of $\G$ with vertex sets $\{\V_k\}_{k = 1}^K$ such that $\bigsqcup_{k = 1}^K \mathcal{V}_k = \mathcal{V}$. For each subgraph $\G_k$, let $\Smat_k \in \{0,1\}^{|\mathcal{V}_k |\times N}$ denote a binary selection matrix representing an inclusion map between $\G_k$ and $\G$, such that $\Smat_k (i,j)=1$ if and only if the $i$-th node in $\G_k$ corresponds to the $j$-th node in $\G$. We consider an LSGP $ \x = \sum_{k=1}^K \Gk \UG h_k(\LamG) \UG^T \w$ on graph $\G$ whose membership functions $\gk$'s have the following property:
\begin{assumption}
\label{ass:vertex_separation}
Each membership function $\gk$ is localized over the subgraph $\G_k$ such that there exist constants $\delta,\mu,\mub >0$ with $| \Smat_k \Gmat_m | \preccurlyeq \delta \Smat_k $ for $k \neq m$, and $  \mu \Smat_k \preccurlyeq \Smat_k \Gmat_k \preccurlyeq  \mub \Smat_k$ for $k=1, \dots, K$.
 \end{assumption} 
According to the above assumption, each membership function $\gk$ must be relatively strong on the corresponding subgraph $\G_k$ as lower bounded by the parameter $\mu$, while it should take weaker values on the other subgraphs $\G_m$ as upper bounded by the parameter $\delta$. The parameter $\mub$ stands for an upper bound that prevents the membership functions from taking unbounded values at an arbitrary node.

We first wish to determine how well the process $\x$ can be approximated by the component processes $\xk$, which we characterize in terms of their second-order statistics. Let 
\begin{equation*}
\Ckm = E[ \xk \xm^T ]
\end{equation*}
denote the cross-covariance matrix of the component processes $\xk$ and $\xm$. In the next main result, we provide an upper bound on the deviation between $\Cx$ and the cross-covariances $\Ckm $ of the component processes.

 \begin{theorem}
 \label{theorem:arbitrary_clusters}
Let Assumption \ref{ass:vertex_separation} hold for the LSGP $\x$. Then for all $k, m \in \{ 1, \dots, K \}$, the cross-covariance $\Ckm $ of $\xk$ and $\xm$ approximates the overall covariance $\Cx$ on $\G_k$ and $\G_m$ according to the following bound:
 \begin{equation}
  \begin{split}
    &| \Smat_k \invGk \Cx  (\invGm)^T  \Smat_m^T - \Smat_k \Ckm \Smat_m ^T| \\
    &\preccurlyeq  \left(2 (K-1) \frac{\delta}{\mu} + (K-1)^2 \left(\frac{\delta}{\mu} \right)^2 \right)  \mathbf{1}_{\crd{\mathcal{V}_k} \times \crd{\mathcal{V}_m }} \\
  \end{split}
 \end{equation}
 \label{thm_dev_Cx_Ckm}
 \end{theorem}

The proof of Theorem  \ref{thm_dev_Cx_Ckm}  is given in Appendix F. The theorem compares the covariance $\Cx$ of the process $\x$ (normalized by the inverse membership functions for appropriate scaling) with the cross-covariance $\Ckm $ of the component processes, when locally restricted to the nodes on the subgraphs $\G_k$ and $\G_m$. Note that by choosing $k=m$, the statement of the theorem pertains to the approximation of $\Cx$ by the covariance $\Cxk$ of the component process $\xk$ on the subgraph $\G_k$. The theorem implies that as the ratio $\delta/\mu$ decreases, which is a measure of how well the supports of the memberships $\gk$'s are restricted to the subgraphs $\G_k$'s, the covariance of $\x$ can be more accurately approximated by the covariances of $\xk$'s. 

This result brings about the possibility of identifying suitable subgraphs $\G_k$'s such that the LSGP $\x$ can be approximated with the WSS process $\xk$ on each subgraph $\G_k$. However, in order to achieve this, the processes $\{\xk\}$ must be weakly correlated with each other as well; i.e.,  $\Cx$ must have negligible entries over its off-diagonal blocks corresponding to $\Ckm$ for $k\neq m$. In order to establish this condition, in addition to Assumption \ref{ass:vertex_separation}, the kernels $\{\hk\}$ must also be sufficiently different from each other. This is characterized via their spectral separation in the following assumption:

 \begin{assumption}
  \label{ass:frequency_separation}
 For all $k, m \in \{1,2, \dots, K\}$ with $k \neq m$, the spectral supports of the kernels $\h_k, \h_m$ are separated from each other such that $\sum_{i = 1}^N \left | \h_k (i) \, \h_m(i)\right| \leq \epsh$.
 \end{assumption}
The parameter $\epsh \geq 0$ in Assumption \ref{ass:frequency_separation} is thus a spectral separation parameter such that small values of $\epsh$ ensure the incoherence of the kernels $\hk$. Assumptions \ref{ass:vertex_separation} and \ref{ass:frequency_separation} then guarantee an upper bound on the process cross-covariance across different subgraphs, which is stated in the next result.

 \begin{theorem}\label{cor:average_cov_bound}
Let Assumptions \ref{ass:vertex_separation} and \ref{ass:frequency_separation} hold. Then, the average squared cross-covariance of $\x$ across different subgraphs $\G_k$, $\G_m$ is upper bounded as 
 \begin{equation}\label{eq:cor_avg}
    \begin{split}
      & \frac{1}{\mub ^4 N^2}  \sum_{k=1}^K    \sum_{\substack{m=1\\ m \neq k}}^K \sum_{(i,j) \in \mathcal{V}_k \times \mathcal{V}_m} \left| \Cx(i,j) \right|^2 \leq  \frac{2}{N}K(K-1)\epsh^2 \\
      & + \frac{2}{N^2} \crd{\bigcup_{k=1}^K    \bigcup_{\substack{m=1\\ m \neq k}}^K    \mathcal{V}_k \times \mathcal{V}_m} \left(2 (K-1) \frac{\delta}{\mu} + (K-1)^2 \left(\frac{\delta}{\mu}\right)^2 \right)^2. \\
    \end{split}
\end{equation}
 \end{theorem}

Theorem \ref{cor:average_cov_bound} is proved in Appendix G. In the theorem, the first term in the right hand side decreases with $\epsh$; hence, when the kernels $\hk$ have lesser frequency content in common, the cross-correlation between different subgraphs $\G_k$, $\G_m$  weakens.  Then, the second term reflects the effect of the localization of the membership functions on the process cross-covariance. As each membership function $\gk$ attains better localization on $\G_k$, the ratio $\delta/\mu$ decreases due to Assumption \ref{ass:vertex_separation}, reducing the cross-covariance across different subgraphs. We notice from the second term that the cross-covariance magnitudes increase at a rate of $O(K^2 \, (\delta/\mu)^2)$ with $K$. This practically suggests that the number $K$ of subgraphs must be at most of $O\big( (\delta/\mu)^{-1}\big)$, so that the process models on different subgraphs remain distinguishable. In addition to the  cross-covariance upper bound presented in Theorem \ref{cor:average_cov_bound}, in Appendix H we also show that a lower bound can be derived on the within-subgraph covariance values under the assumption that the process varies sufficiently slowly on each subgraph.

 \subsection{Proposed Algorithm for Locally Approximating LSGPs}
 \label{sec:subgraph_algorithm}

The analysis in Sec. \ref{sec:cov_bounds} shows that, under certain assumptions, the second-order statistics of an LSGP can be locally approximated by that of a WSS graph process. In this section, inspired by these results, we propose an algorithm for partitioning a given graph $\G$ such that an LSGP $\x$ defined on $\G$ can be approximated by an individual WSS process $\xk$ on each subgraph $\G_k$. We have seen in Sec. \ref{sec:cov_bounds}  that if the LSGP model $\x$ admits a local approximation, the cross-covariance of $\x$ across different subgraphs must be relatively weak, while ensuring a lower bound on the covariance of the process on each individual subgraph. Assuming that an initial estimate $\hatCx$ of the covariance matrix is available, we thus propose to inspect $\hatCx$ along with the graph topology $\G$ in order to identify a set of subgraphs $\{\G_k\}_{k=1}^K$, such that the weak entries in $\hatCx$ are associated with between-subgraph cross-covariance values, and the strong entries in $\hatCx$ correspond to within-subgraph covariance values.

In order to determine the subgraphs $\{\G_k\}$, we first use the covariance estimate $\hatCx$ for defining a distance function $\Cdistx:  \E \rightarrow \R^+$ on the edges $\E$ of the graph $\G$. The distance function $\Cdistx(i,j)=d(\hatCx(i,j))$ is computed through a continuous and even kernel $d: \R \rightarrow \R^+$  that is strictly decreasing on $\R^+ \cup \{0\}$. A suitable choice for $d(x)$ is the Gaussian function $\exp\left(-x^2/\theta\right)$, where the parameter $\theta$ adjusts the mapping between the covariance values and the distances. Once the distance function  $\Cdistx(i,j)$ is obtained, we employ a graph partitioning algorithm $\gPar(\G, \Cdistx, K)$ that partitions the graph $\G$ into $K$ disjoint subgraphs $\{ \G_k \}_{k=1}^K$ by cutting the edges $(i,j)$ associated with high distance values $\Cdistx(i,j)$ and retaining those with low distances. Many alternatives exist in the literature for the choice of the partitioning algorithm $\gPar$; an example method can be found in the study \cite{Chien2019}, where the edges in $\G$ are progressively removed in a geometry-dependent manner based on the Ricci curvature induced by the distance $\Cdistx$. The resulting graph partitioning procedure is shown as \texttt{partitionGraph} in Algorithm \ref{alg:subgraph_clustering}. Once the subgraphs $\{ \G_k \}_{k=1}^K$ are determined, we construct their selection matrices $\{\Smat_k\}_{k = 1}^K$, and restrict the covariance estimate $\hatCx$ to each subgraph as $ \Smat_k \hatCx \Smat_k^T$. One can then compute a WSS graph process on each subgraph $\G_k$ through the method described in Algorithm \ref{alg:lsp_learning} by setting $K=1$ in the LSGP model. The overall procedure is outlined in Algorithm \ref{alg:subgraph_clustering}.

\begin{algorithm}
  {\small
    \caption{\small Local Approximation of LSGPs}\label{alg:subgraph_clustering}
    \hspace*{\algorithmicindent}
    \textbf{Input:} 
    Graph $\G = (\mathcal{V}, \E, \W)$, 
    initial covariance estimate $\hatCx$,
    number of subgraphs $K$,
    distance function $d $,
    graph partitioning method $\gPar$ \\
    \hspace*{\algorithmicindent} \textbf{Output:} Subgraphs $\{\G_k\}$, an individual process model $(\gk, \hk)$ on each subgraph $\G_k$.   
        \begin{algorithmic}[1]
    \STATE $\{\G_k\}_{k=1}^K \gets  \texttt{partitionGraph}\left(\G, \hatCx, \gPar, d\right)$
    \STATE \textbf{for} $k = 1, 2, \dots K$ 
    \STATE \hspace{0.5cm}Construct selection matrix $\Smat_k$ for subgraph $\G_k$
    \STATE \hspace{0.5cm}$ (\g,\h, \sim, \sim) \ \gets \ \text{\texttt{learnLSGP}}(\G_k, \Smat_k \hatCx \Smat_k^T, 1)$
    \STATE \hspace{0.5cm} $\gk=\g$;   $\ \hk=\h$
    \STATE \textbf{end}
    \end{algorithmic}}
        \vspace{0.15cm}  
        \small
    \textbf{procedure} \texttt{partitionGraph}$(\G, \hatCx, \gPar, d)$ \\
    \hspace*{0.5cm}$\Cdistx: \E \rightarrow  \R^+$  \\
    \hspace*{0.5cm}\textbf{for} $(i,j) \in \E$ \\
    \hspace*{1cm}$\Cdistx  (i,j) = d\left( \hatCx (i,j)\right)$ \\
    \hspace*{0.5cm}\textbf{end} \\ 
    \hspace*{0.5cm}$\{\G_k\}_{k=1}^K \gets \gPar(\G, \Cdistx, K)$ \\
    \hspace*{0.5cm}\textbf{return} $\{\G_k \}_{k=1}^K$
    \vspace{-0.2cm} \\
\end{algorithm}

\section{Experimental Results}
\label{sec_experiments}
In this section, we evaluate the performance of the presented algorithms on synthetic and real datasets.

\subsection{Performance Analysis of the Proposed Methods}

\subsubsection{Sensitivity of Algorithm \ref{alg:lsp_learning} to noise level}
\label{sec:perf_noise}

In order to study the performance of the proposed LSGP algorithm under noise, we construct a synthetic $5$-NN graph with $N=36$ nodes from 2D points with random locations, where the edge weights are determined with a Gaussian kernel. We then generate realizations of an LSGP $\x$ with parameters $K = 3, Q = 4$ on this graph according to the model $\eqref{eq_locstat_model}$. The realizations are corrupted with additive white Gaussian noise at different signal-to-noise ratio (SNR) levels.

We study two problems. In the first problem, we initially compute a sample covariance (SC) estimate $\hatCx$ of the process from all available realizations. We then learn a model by giving $\hatCx$ as input to the proposed  Algorithm \ref{alg:lsp_learning}. We finally obtain an estimate $\estCx= \Hmat \Hmat^T$ of the covariance according to the learnt model, which is denoted as (SC+LSGP) in the results. We determine the covariance discrepancy (CD) between the learnt process and the true process as CD $= \left|\left| \Cx - \estCx \right|\right|_F / \|  \Cx \|_F$, 
where $\Cx $ denotes the true covariance matrix of the process. The CD of the sample covariance $\hatCx$ is computed similarly. The CD values of the methods (SC), (SC+LSGP) are plotted in Fig.~\ref{fig_cd_variation}.

In the second problem, we address a scenario where 10000 realizations $\x^l$ of the process are given, with half of the values in the realizations being missing.  We first obtain the initial covariance estimate with two alternative approaches: the sample covariance estimate (SC), and the sparse correction of the sample covariance with the shrinkage estimator in \cite{Lounici2014} ($\ell$1). Then these two covariance estimates are provided as input to Algorithm \ref{alg:lsp_learning}, whose results are respectively denoted as (SC+LSGP) and ($\ell$1+LSGP). Once the covariance matrix $\estCx $ of the proposed LSGP method is found,  the LMMSE estimates of the missing values in each realization $\x^l$ are obtained as 
\begin{equation}
\label{eq_lmmse_est_z}
	\hatz^l = (\covMat_{\z \y}^*)^l \, ((\covMat_{\y}^*)^l)^{-1} \ \y^l 
\end{equation}
where the vectors $\y^l$ and $\z^l$ respectively contain the available and the initially missing entries of each realization $\x^l$, and $\hatz^l$ is the estimate of $\z^l$. The matrix $(\covMat_{\z \y}^*)^l$ denotes the estimated cross-covariance of $\z^l$ and  $\y^l$, and $(\covMat_{\y}^*)^l$ is the estimated covariance of $\y^l$, which can be obtained by extracting the corresponding entries of $\estCx$ for each realization $\x^l$. Defining a concatenated vector $\z$ that consists of the missing values $\z^l$  in all realizations and its estimate $\hatz$, we evaluate the estimation error with respect to the normalized mean error NME $=||\z - \hatz||_2 / ||\z||_2$, the mean absolute error MAE $= ||\z - \hatz||_1/ L_\z$, and the mean absolute percentage error MAPE $=1 / L_\z \sum_{i = 1}^{L_\z}  | \z(i) - \hatz(i) | / | \z(i) | $ metrics, where $L_\z$ denotes the length of $\z$. The errors of the other estimates (SC), ($\ell$1) are computed similarly. The estimation errors of all methods are plotted in Fig.~\ref{fig_noise_nme_variation}.

In Fig.~\ref{fig_cd_variation}, the CD of the sample covariance estimate $\hatCx$ remains above that of the LSGP algorithm output $\estCx$, with the gap reaching around $0.05$ at $-3$ dB and around $0.04$ at infinite SNR. Despite the seemingly minor difference in the CD values of $\hatCx$ and $\estCx$ in Fig.~\ref{fig_cd_variation}, the proposed method (LSGP) significantly improves the estimation performance of the sample covariance (SC) and the shrinkage estimates ($\ell$1) in the interpolation problem in Fig.~\ref{fig_noise_nme_variation}. In particular, in Fig.~\ref{fig_noise_nme_variation} the sample covariance estimate yields quite high errors at high SNR values. This is because the diagonal loading effect of the noise covariance is lost at high SNR values, which results in $\hatCx$ estimates with negative eigenvalues in this scenario with missing observations. This undesired artifact is efficiently corrected by the proposed algorithm, providing a substantial improvement over the initial estimates of the process statistics.

\begin{figure}[t]
  \begin{center}
    \hspace{-0.7cm}
       \subfloat[]
       {\label{fig_cd_variation}\includegraphics[height=3.8cm]{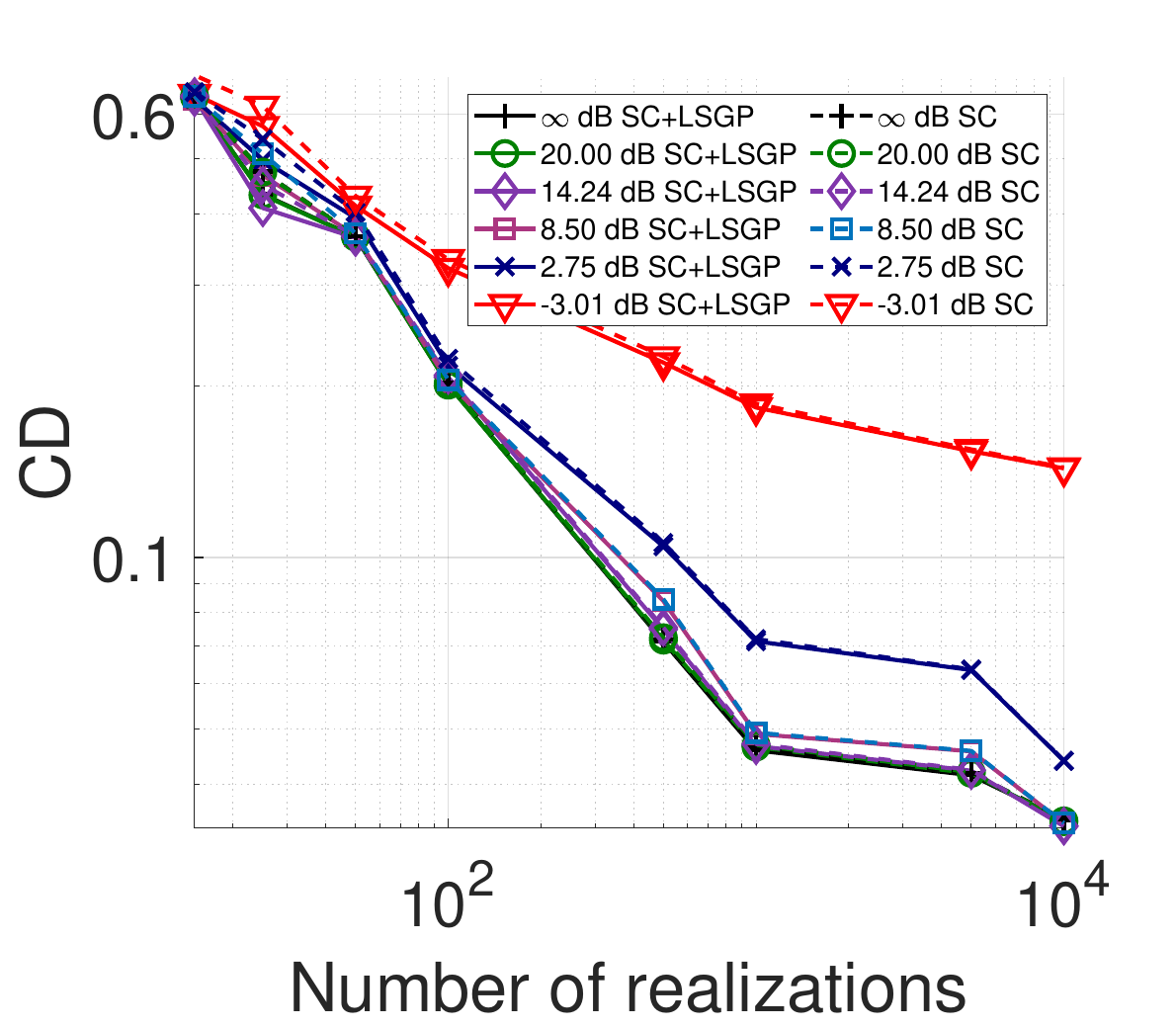}}
       \hspace{0.1cm}
      \subfloat[]
      {\label{fig_noise_nme_variation}\includegraphics[height=3.8cm]{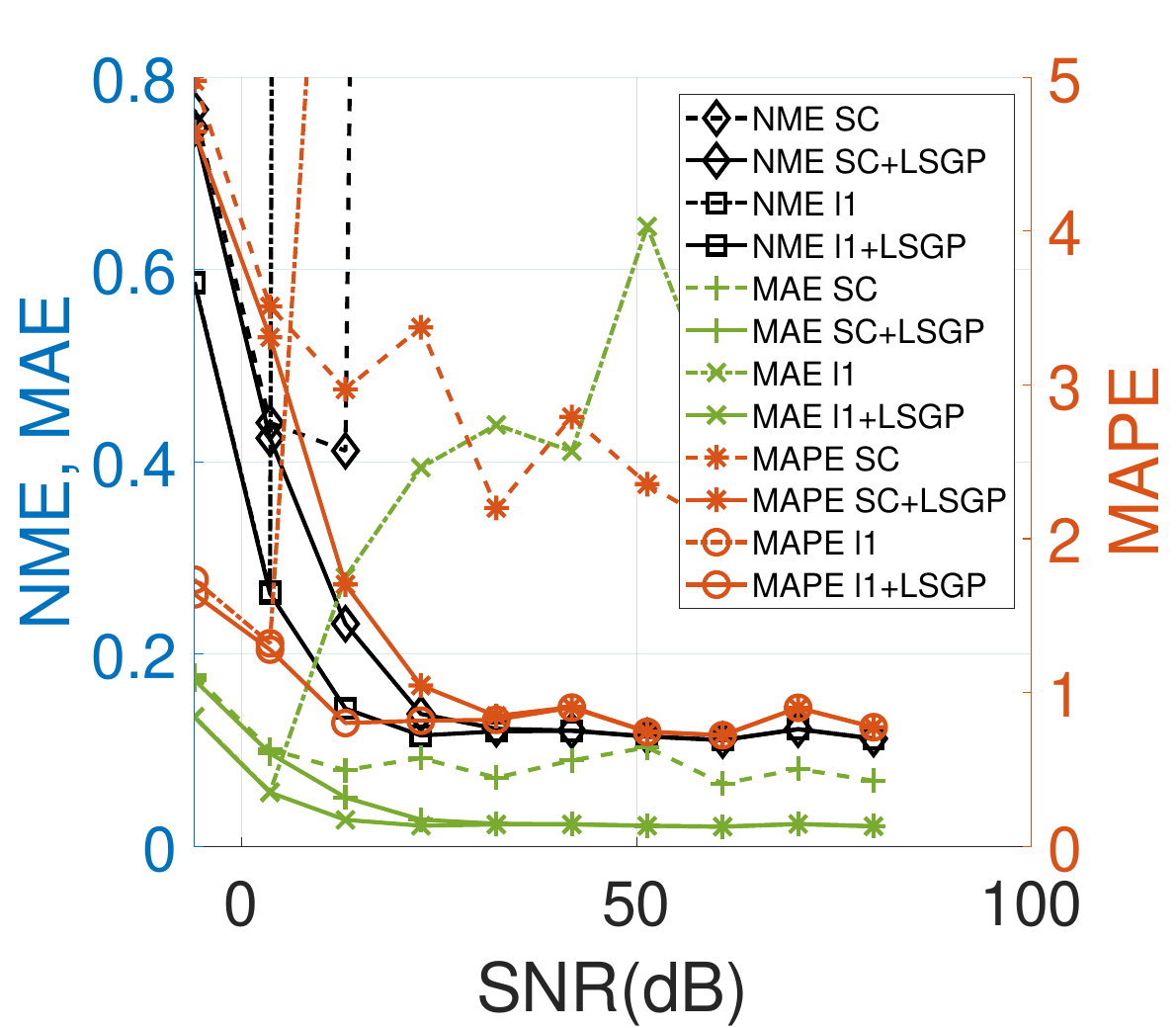}}\\
   \end{center}
    \caption{(a) Variation of the covariance discrepancy with respect to the number of realizations. (b) Variation of the estimation errors with the SNR.}

   \label{fig_noise_errors}
  \end{figure}

\subsubsection{Additional analyses}
In Appendix I, we provide additional performance analysis results, where we first study the effect of the model parameters $K$ and $Q$, and the regularization parameters $\mu_1, \mu_2, \mu_3$  on the performance of Algorithm \ref{alg:lsp_learning}. The results show that as the model complexity increases, the number of realizations required to attain a certain performance level also increases as expected, while the performance of Algorithm \ref{alg:lsp_learning} is rather stable with respect to the variations in the regularization parameters in a relatively wide region. We then evaluate the theoretical findings of Section \ref{sec:wss_lsp} by conducting a performance analysis of Algorithm \ref{alg:subgraph_clustering}.

\subsection{Comparative Experiments on Real Data Sets}
\label{sec:comparative_experiments}

In this section, we evaluate our signal estimation performance on the following real data sets:

\textit{COVID-19 pandemic data set.} The COVID-19 data set consists of the number of daily new COVID-19 cases in $N = 37$  European countries of highest populations between February 15, 2020 and July 5, 2021 \cite{COVID-cases}.  A $4$-NN graph is constructed by considering each country as a graph node. Edge weights are determined with a Gaussian kernel based on a hybrid distance measure that combines geographical distances and numbers of flights accessed via \cite{COVID-flights}. Normalized by country populations and smoothed out with a moving average filter over one week, the numbers of daily new cases are taken as graph signals. 

 \textit{Mol\`ene weather data set.} This data set consists of hourly temperature measurements taken in $N = 37$ measurement stations in the Brittany region of France in January 2014  \cite{Girault2015}. The graph is constructed with a $5$-NN topology by considering each station as a graph node, with Gaussian edge weights based on the geographical distance between the stations. Experiments are done on $744$ graph signals.

\textit{NOAA weather data set.} The NOAA data set contains hourly temperature measurements for one year taken in weather stations across the United States averaged over the years 1981-2010 \cite{NOAA-dataset}. We construct a  $7$-NN graph from $N = 246$ weather stations  with Gaussian edge weights. The experiments are done on $8760$ graph signals.

\textit{USA COVID-19 data set.} This data set \cite{Dong2020} consists of the number of COVID-19 patients recorded in the United States. We construct a  $10$-NN graph with Gaussian edge weights from $N = 1238$ locations in the east of the United States and experiment on $1044$ graph signals.

\subsubsection{Comparative performance evaluation of the LSGP algorithm} \label{sec:comparative_exps} We first study a signal interpolation problem on the Mol\`ene and COVID-19 data sets by considering two scenarios: 

\begin{itemize}
  \item (Random data loss) Missing observations of the graph signals occur at nodes selected uniformly at random.
  \item (Structured data loss) Missing observations occur at particular regions of the graph over a local clique of neighboring graph nodes. 
\end{itemize}

When testing the proposed method (LSGP), each graph signal is treated as the realization of a locally stationary graph process, an LSGP model is learnt with Algorithm \ref{alg:lsp_learning}, and LMMSE estimates of the missing observations are computed as in \eqref{eq_lmmse_est_z}. The proposed LSGP model is compared to two other stochastic graph process models; namely, wide sense stationary graph processes (WSS) \cite{Perraudin2017} and graph ARMA processes (Graph-ARMA) \cite{Marques2017}. We also include three reference non-stochastic graph signal interpolation approaches in our comparisons, based on the total variation regularization of graph signals (TV-minimization) \cite{Jung2019}, the deep algorithm unrolling method (Nest-DAU) recently proposed in \cite{Nagahama2022}, and graph attention networks with dynamic attention coefficients (GATv2) \cite{Brody2021}.  All algorithms employing the process covariance matrix have been provided the sample covariance estimate as input. Algorithm hyperparameters  are determined with validation for all methods that require parameter tuning. 

\begin{figure}[t]
  \begin{center}
    \hspace{-0.7cm}
       \subfloat[COVID-19 (Random data loss)]
       {\label{fig_nme_covid_sp}\includegraphics[height=3.5cm]{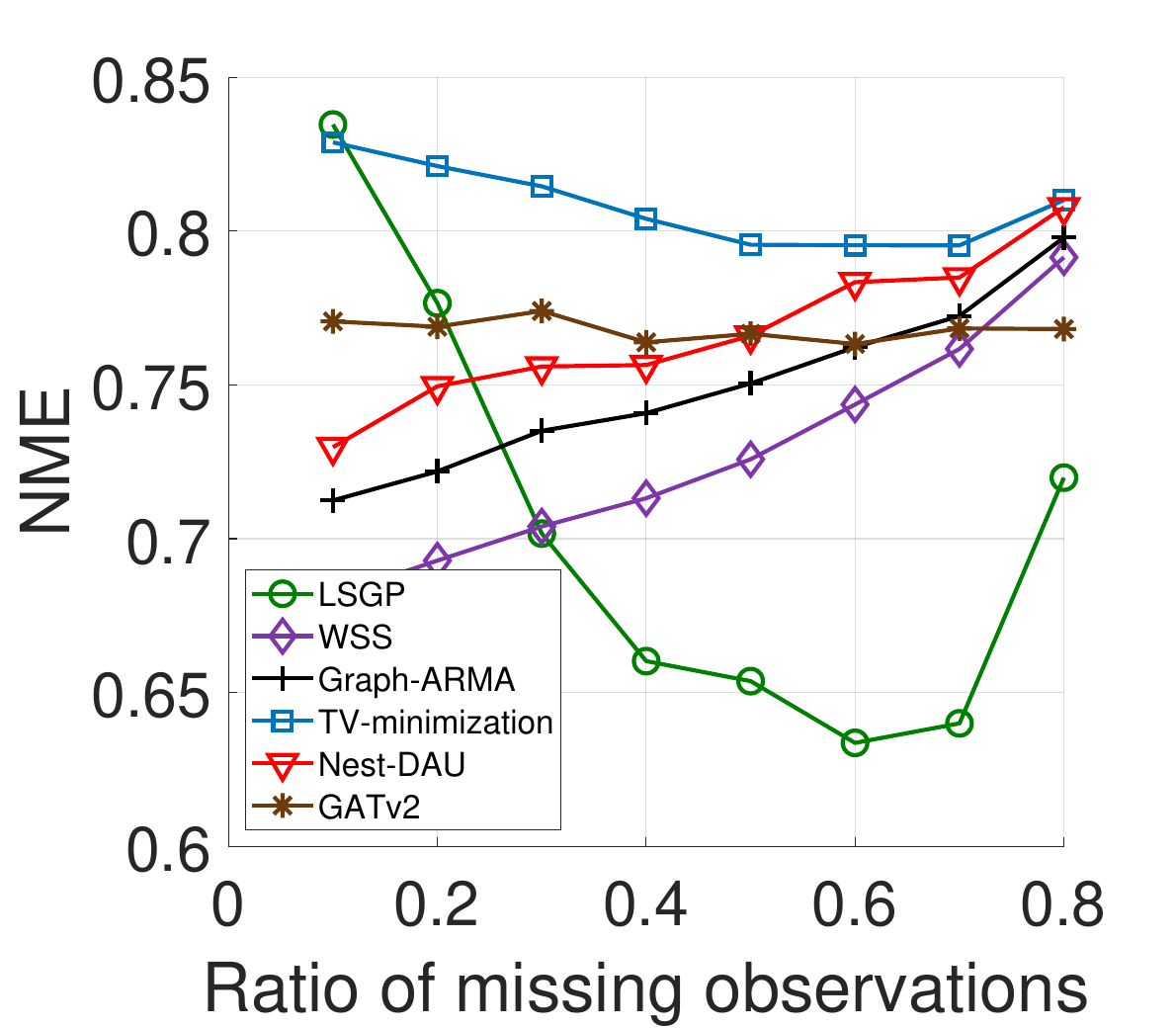}}
       \hspace{0.1cm}
      \subfloat[Mol\`ene (Random data loss)]
      {\label{fig_nme_molene_sp}\includegraphics[height=3.5cm]{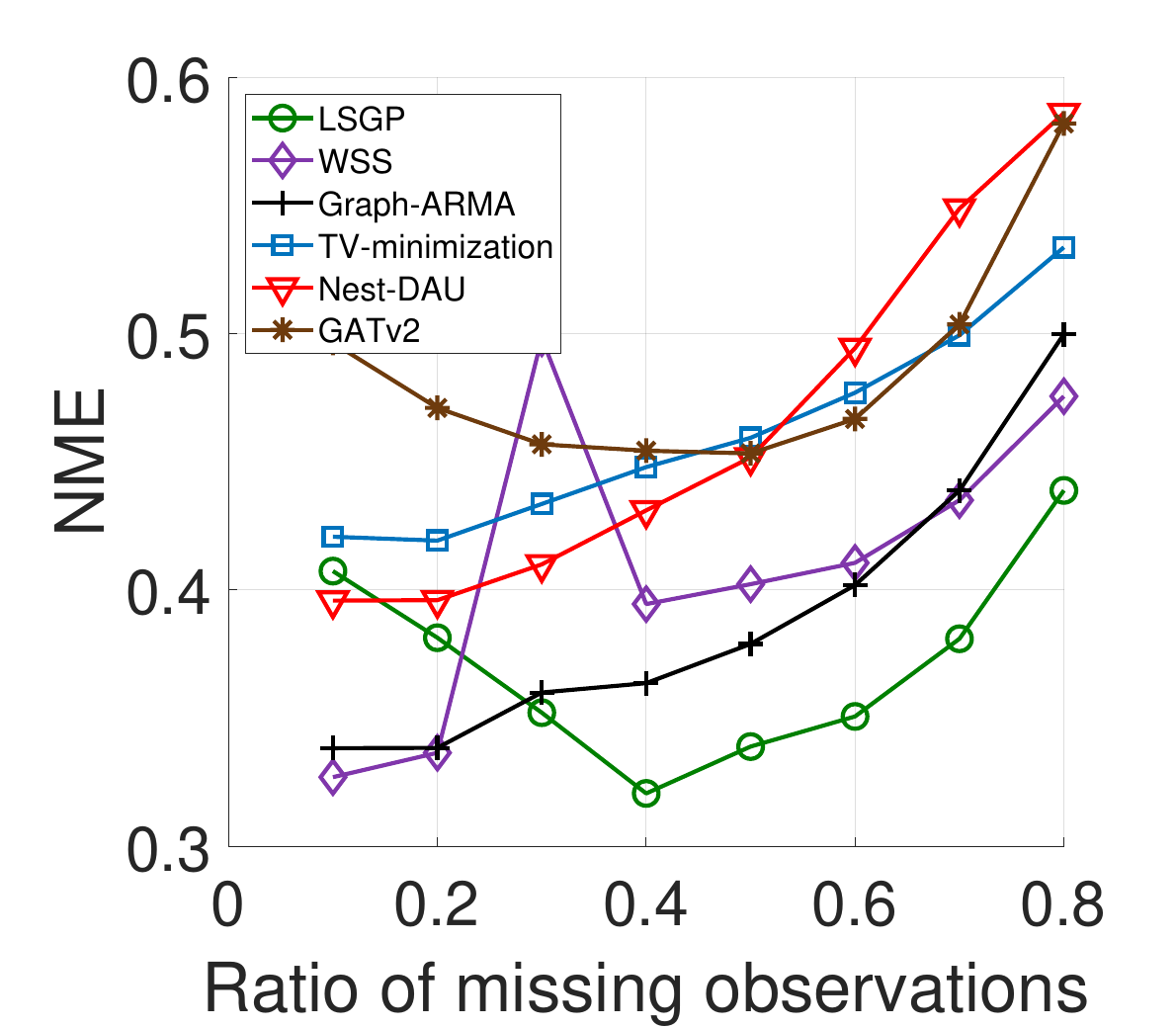}}\\
         \subfloat[COVID-19 (Structured data loss)]
         {\label{fig_nme_covid_nm}\includegraphics[height=3.5cm]{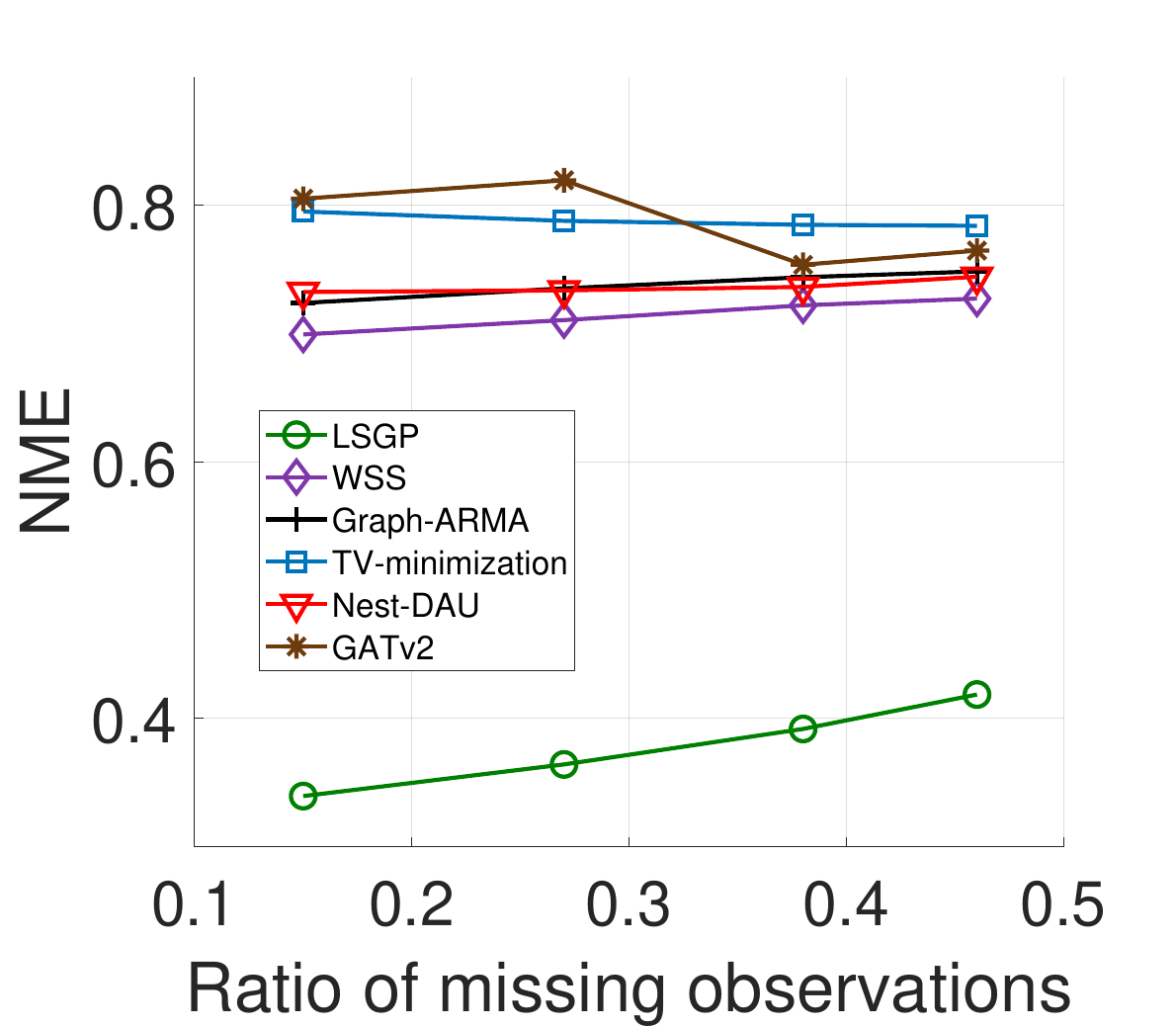}}
         \hspace{0.1cm}
        \subfloat[Mol\`ene (Structured data loss)]
        {\label{fig_nme_molene_nm}\includegraphics[height=3.5cm]{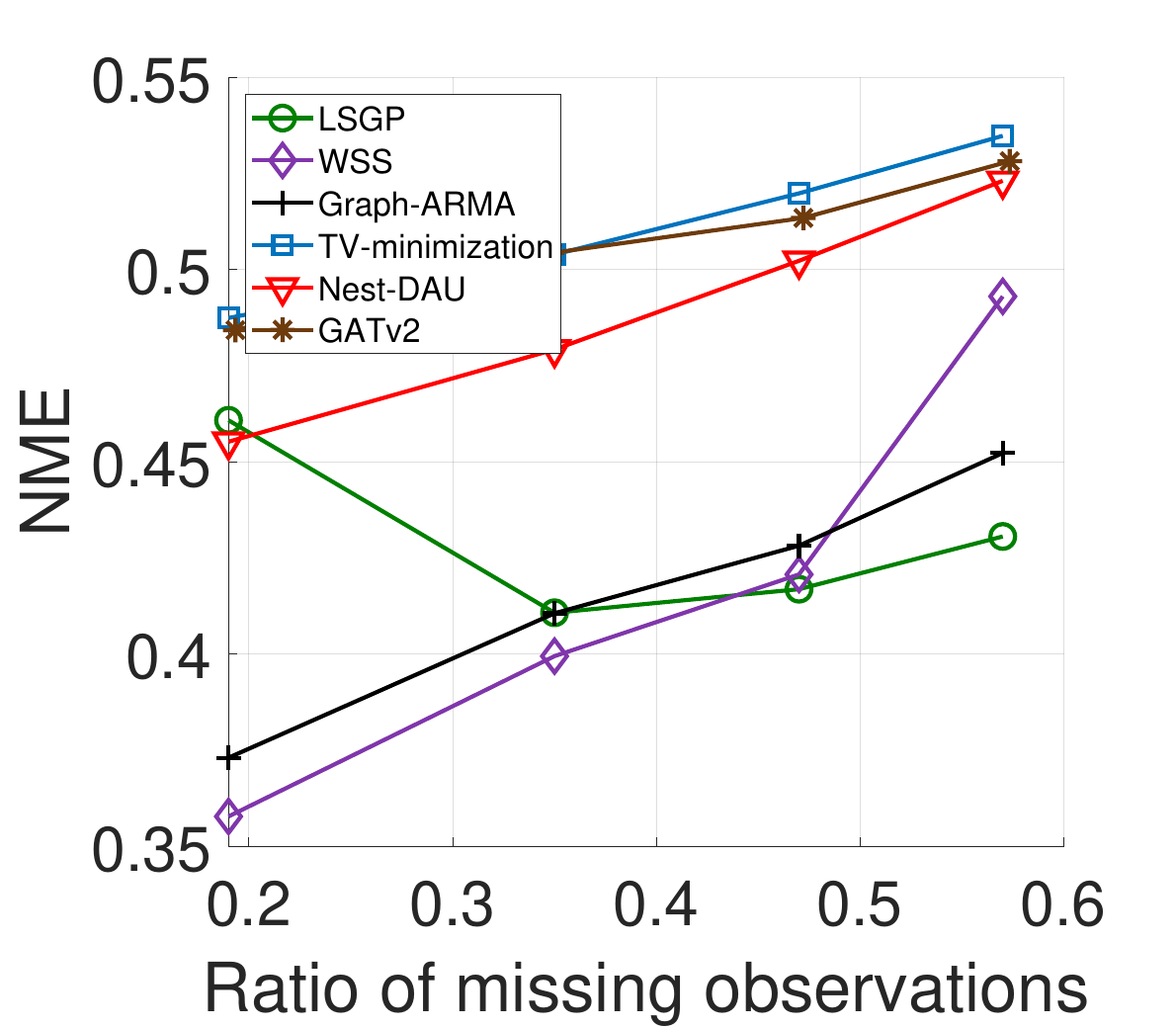}}\\
   \end{center}
    \caption{NME of compared algorithms on COVID-19 and Mol\`ene data sets}
   \label{fig_nmes}
  \end{figure}

The variation of the NME of the methods with the ratio of missing observations is shown in Figure \ref{fig_nmes} (the results with the MAE and the MAPE metrics can also be found in Appendix J). The signal estimation performance of the proposed LSGP method is seen to be competitive with the other methods, often outperforming them especially at middle-to-high missing observation ratios. In most instances, the stochastic process based methods LSGP, WSS, and Graph-ARMA provide smaller error than the non-stochastic TV-minimization, Nest-DAU and GATv2 methods. The error of the proposed LSGP algorithm often shows a non-monotonic variation with the missing observation ratio, which is a somewhat surprising finding. We interpret this in the following way: When selecting the algorithm hyperparameters $K,Q$, $\mu_1,\mu_2,\mu_3$ via validation, the ratio of missing observations must fall within a suitable interval so that the error obtained on the validation data is well representative of that obtained on the missing data. While the performance of  LSGP is similar between the COVID-19 and the Mol\`ene data sets in the random data loss scenario, its behavior is quite different among the two data sets for structured data loss. The COVID-19 data has weaker vertex stationarity than Mol\`ene \cite{Guneyi2023}, indicating that the process characteristics show higher diversity across different graph regions.  While the other algorithms learn a global model for the whole graph and therefore find it harder to compensate for the loss of information in a local region through the average signal statistics on the whole graph, the proposed LSGP algorithm can learn a model whose local statistics are successfully adapted to different neighborhoods, achieving a substantial performance improvement over the other methods.

\subsubsection{Local approximation of LSGPs}\label{sec:cluster_exps} We finally study the performance of locally approximating LSGPs with smaller processes in problems where one needs to analyze data acquired on large network topologies. We experiment on the relatively larger NOAA weather and the USA COVID-19 data sets in order to test the methods on both the whole graphs and their partitioned versions.

\begin{figure}[t]
  \begin{center}
      \subfloat[Partitioning of the NOAA  graph]
      {\label{fig_noaa_clusters}\includegraphics[height=3.5cm]{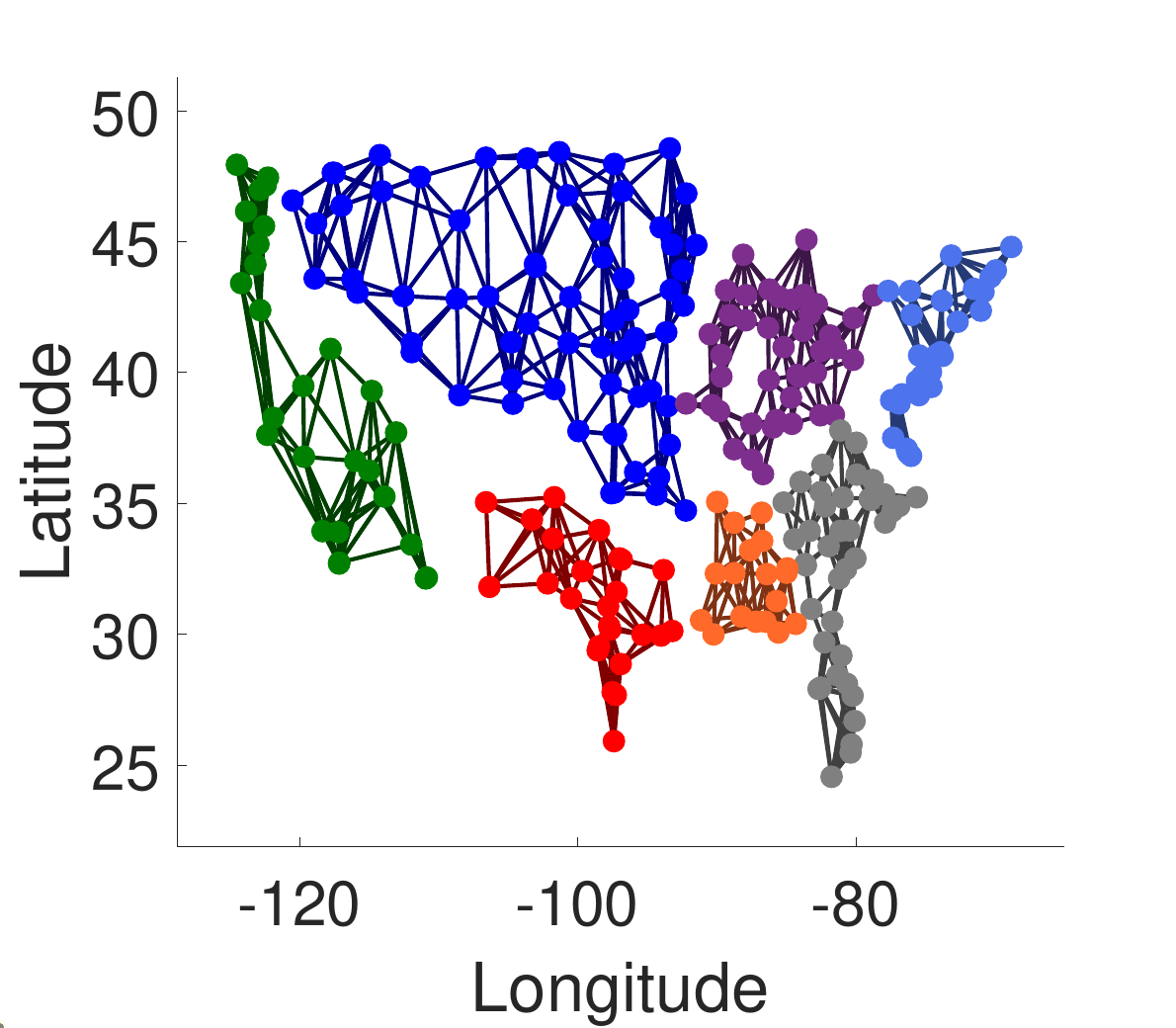}}
      \subfloat[Partitioning of USA COVID-19]
      {\label{fig_eus_clusters}\includegraphics[height=3.5cm]{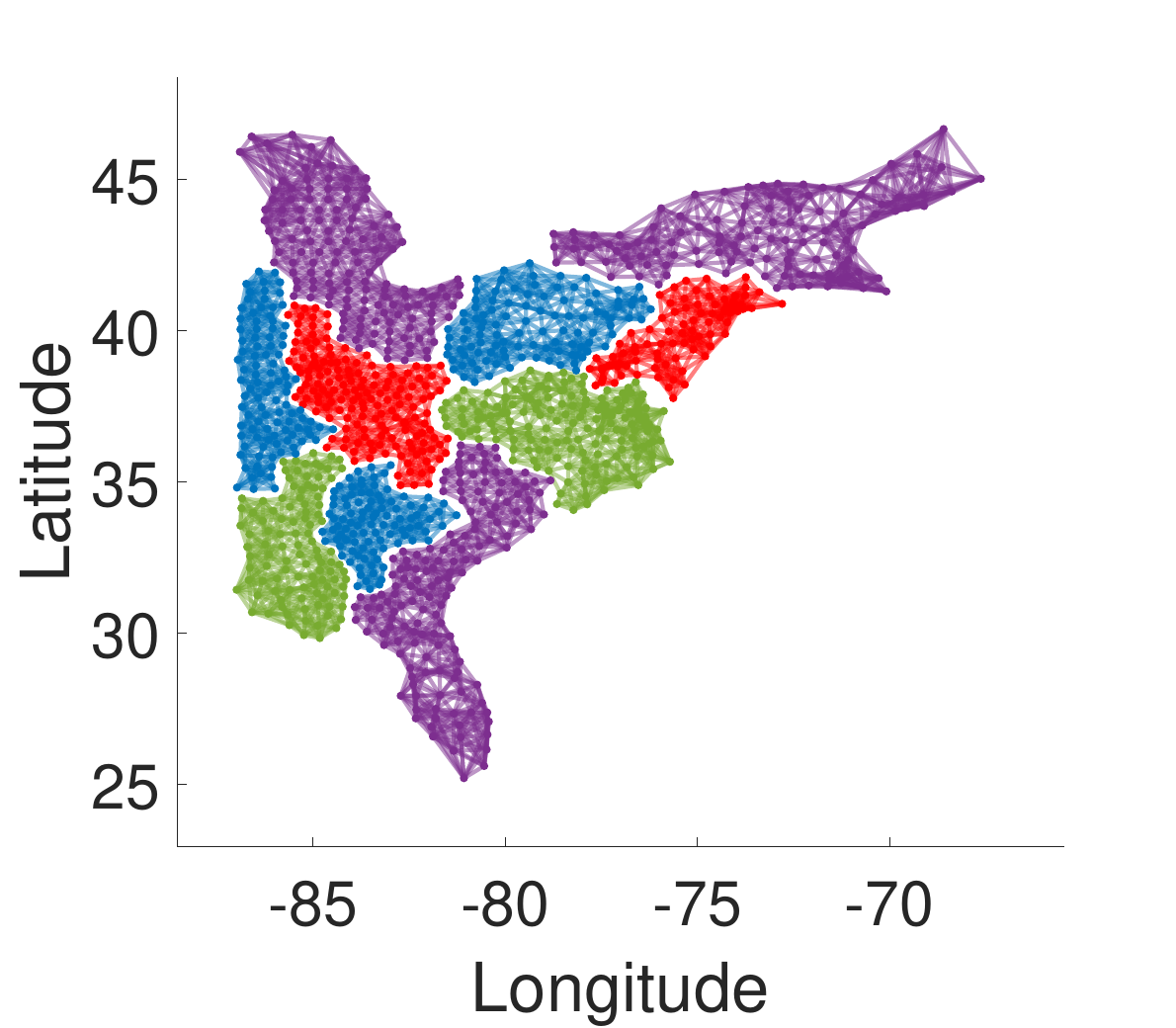}}\\
       \subfloat[NOAA data set]
      {\label{fig_nme_noaa_sp}\includegraphics[height=3.5cm]{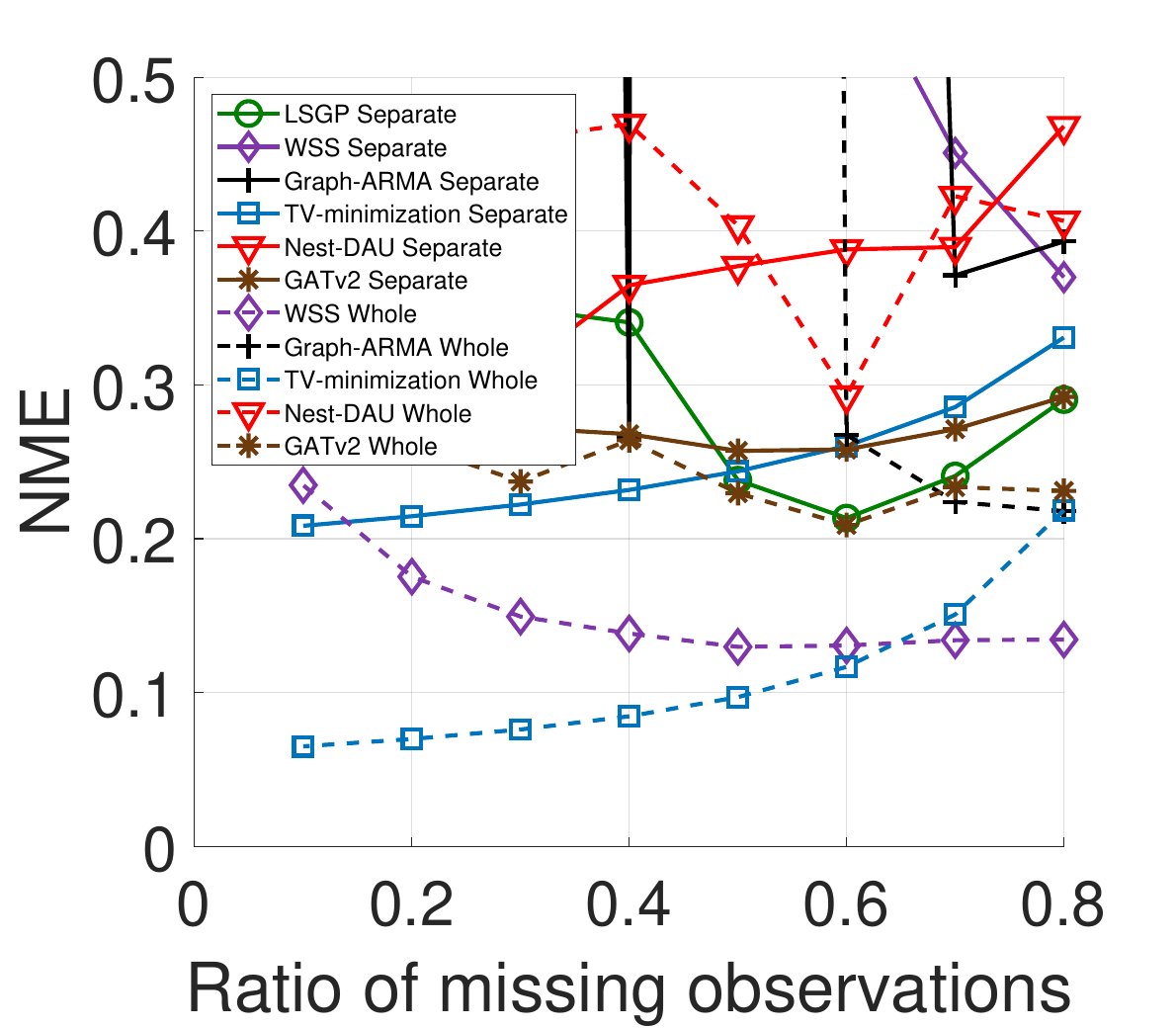}}
      \subfloat[USA COVID-19 data set]
      {\label{fig_nme_eus_sp}\includegraphics[height=3.5cm]{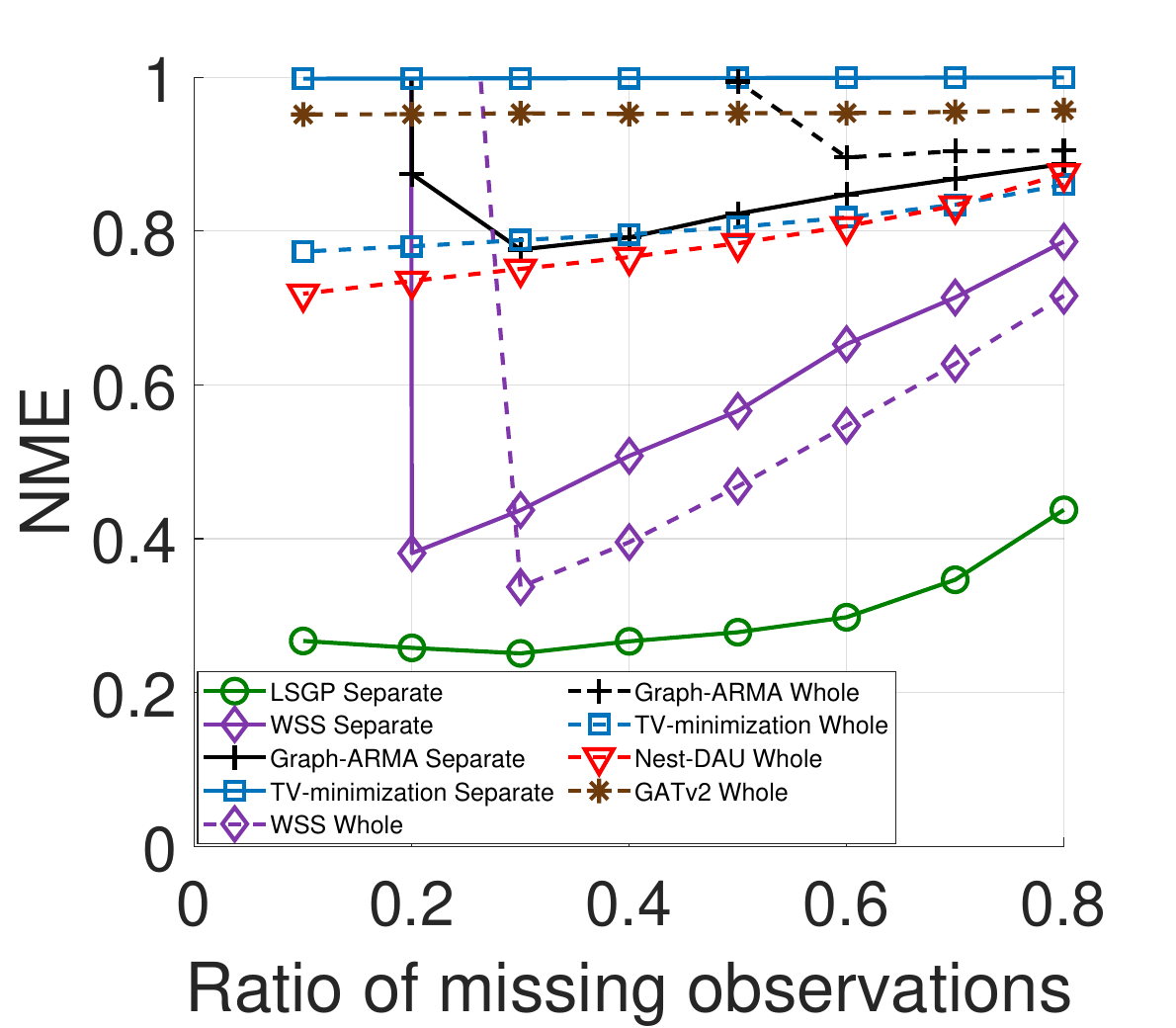}}
   \end{center}
    \caption{Results obtained on the NOAA and the USA COVID-19 data sets}
   \label{fig_noaa_partitioning}
  \end{figure}

We first partition the graphs using Algorithm \ref{alg:subgraph_clustering} where we set the number of subgraphs as $K=7$ for NOAA and $K=10$ for USA COVID-19. The partitioning results are shown in Fig.~\ref{fig_noaa_clusters}-\ref{fig_eus_clusters}. We select the missing observations uniformly at random and consider two signal estimation settings for the compared methods: In the first setting we learn distinct models on separate subgraphs (Separate), while in the second setting we learn a single model on the whole graph (Whole).  The algorithms are tested in both settings and compared with respect to their NME in Fig.~\ref{fig_nme_noaa_sp}-\ref{fig_nme_eus_sp}. The estimation of the signals on the whole graph often provides higher accuracy than on separate subgraphs, as expected\footnote{The Nest-DAU and the GATv2 methods have been excluded from the USA COVID-19 separate subgraphs setting, and the LSGP method has been excluded from the whole graph settings due to their complexities.}. However, the additional results with the MAE and MAPE metrics given in Appendix J present an interesting exception to this, where the estimation on separate subgraphs results in smaller MAE for most methods. This is because the error in separate modeling concentrates sparsely along subgraph boundaries and therefore has relatively small $\ell_1$-norm. The performance of LSGP is quite competitive with the other methods, offering a promising solution in the separate subgraphs setting. We also compare the time complexities of the methods in Appendix J by reporting their average runtimes, which show that the  runtimes of stochastic process methods differ by a factor of at least 4 between the separate subgraph and the whole graph settings. This situation illustrates a common challenge faced by many graph signal processing algorithms, whose complexities are typically around $O(N^3)$, making them impractical to use in a straightforward way as the graph size $N$ grows. The experiments in this section have aimed to provide some insight for handling the scalability issue in large graphs via the local modeling and processing of graph signals on suitably identified neighborhoods.

\section{Conclusion}
\label{sec_concl}
In this paper, we have proposed a graph signal model that
extends the classical concept of local stationarity to
irregular graph domains. In contrast to globally stationary processes,
the proposed locally stationary graph process (LSGP)
model permits the process statistics to vary locally over the
graph. After a theoretical discussion of some useful properties
of LSGPs such as vertex-frequency spectrum, we have presented an algorithm for learning LSGP
models from realizations of the process. Then,
considering potential scalability issues
regarding the computation of process models on large graphs,
we have studied the problem of locally approximating LSGPs
on smaller subgraphs. Experimental results on graph signal
interpolation applications suggest that the proposed graph
signal model provides promising performance in comparison
with reference approaches. Some possible future directions
of our study consist of the extension of the proposed
model to incorporate the temporal dimension as well, and the
investigation of alternative representations of LSGP models
towards developing the scalability of learning algorithms.

{\appendices

\section*{Appendix A. Useful Lemmas}

In this section, we present several lemmas that will be useful in the proofs of our results.

\begin{lemma}
    \label{lemma:absolute_matrix}
For any matrices $\Umat, \Amat, \Vmat $ of compatible size such that $\Umat \Amat \Vmat^T \in \R^{M\times N}$, the inequality $|\Umat \Amat \Vmat^T|\preccurlyeq|\Umat| |\Amat|  |\Vmat^T|$ is satisfied.
\end{lemma}
    
\begin{proof}
  
Let us denote the $n$-th rows of $\Umat$ and $\Vmat$ respectively as $\uvect_n^T$ and $\vvect_n^T$. We then have 
 \[
 \Umat \Amat \Vmat^T =\begin{bmatrix}
  \uvect_1^T \Amat \vvect_1 & \uvect_1^T \Amat \vvect_2 & \cdots & \uvect_1^T \Amat \vvect_N \\ 
  \uvect_2^T \Amat  \vvect_1 & \uvect_2^T \Amat \vvect_2 & \cdots & \vdots \\
  \vdots & \vdots & \ddots & \vdots \\ 
  \uvect_M^T \Amat \vvect_1 & \cdots & \cdots & \uvect_M^T  \Amat \vvect_N \end{bmatrix}
.  \]
For any $(i,j)$ index pair,  we obtain
\begin{equation*}
\begin{split}
& | \uvect_i^T \Amat \vvect_j | 
= \left|\sum_{k }\sum_{l } \Amat(k,l) \uvect_i^T (k) \vvect_j(l) \right| \\
 &\leq 
 \sum_{k }\sum_{l } |\Amat(k,l)| \,  |\uvect_i^T (k)| \, |\vvect_j(l)| 
 = |\uvect_i^T | \, |\Amat| \, |\vvect_j|
 .
 \end{split}
 \end{equation*}
It follows that $ |\Umat \Amat \Vmat^T | \preccurlyeq |\Umat| |\Amat| |\Vmat^T|$.
\end{proof}
    
\begin{lemma}
    \label{lemma:absolute_matrix_2}
Let $\Umat, \Amat, \Bmat, \Vmat $ be real matrices of compatible size such that $\Umat \Amat \Vmat^T \in \R^{M\times N}$ and $\Umat \Bmat \Vmat^T \in \R^{M\times N}$.    Assume that $ \Umat, \Vmat \succcurlyeq 0 $ and $ \Bmat \preccurlyeq \Amat $. Then $\Umat \Bmat \Vmat^T \preccurlyeq \Umat \Amat \Vmat^T$.
\end{lemma}
    
\begin{proof}
Fixing any index pair $(i,j)$, and using the same notation for row vectors as in the proof of Lemma \ref{lemma:absolute_matrix},  we have
\begin{equation*}
\begin{split}
& \uvect_i^T \Bmat \vvect_j = \sum_k \sum_l \Bmat(k,l) \uvect_i^T(k) \vvect_j^T(l) \\
& \leq  \sum_k \sum_l  \Amat(k,l) \uvect_i^T(k) \vvect_j^T(l)
 = \uvect_i^T \Amat \vvect_j
 \end{split}
 \end{equation*} 
since $\Bmat \preccurlyeq \Amat.$ 
\end{proof}

The following lemma is due to \cite{Styan}.
\begin{lemma}
    \label{lemma:basic_lemma}
    For any $\Amat, \Bmat \in \mathbb{R}^{N\times M}$ and diagonal matrices $\Dmat, \Emat$, the following equality holds
    \begin{center}
    $\Dmat (\Amat \circ \Bmat) \Emat = \Amat \circ (\Dmat \Bmat \Emat)$.
    \end{center}
\end{lemma}

\section*{Appendix B. Proof of Theorem \ref{th:vertex_freq_spec}}

\begin{proof}
Taking $\Dmat= \Gmat_k$, $\Amat=\UG$, $\Bmat = \textbf{1}_{N \times N}$ and $\Emat=h_k( \LamG)$ in Lemma \ref{lemma:basic_lemma}, we have 
  \begin{equation}
  \begin{split}
  \Hmat &= \sum_{k=1}^K \Gmat_{k} \UG h_k(\LamG ) \UG^T \\
  &= \sum_{k=1}^K \Gmat_{k} (\UG  \circ \textbf{1}_{N \times N})   h_k(\LamG ) \UG^T \\
  &=\sum_{k=1}^K \big( \UG \circ ( \Gmat_{k} \ \textbf{1}_{N \times N} \ h_k(\LamG ) )\big)\ \UG ^T \\
  &=  \left( \UG \circ \left( \sum_{k=1}^K \Gmat_{k} \ \textbf{1}_{N \times N} \ h_k(\LamG ) \right) \right) \ \UG ^T 
  \end{split}
  \end{equation}
  where the third and the fourth equalities follow respectively from Lemma \ref{lemma:basic_lemma} and the linearity of the Hadamard product. Defining 
  \begin{equation}
  \label{eq_defn_M}
  \M=\sum_{k=1}^K \Gmat_{k} \ \textbf{1}_{N \times N} \ h_k(\LamG) 
  \end{equation}
  we arrive at the equality in \eqref{eq:compact_local}. 
  
 The matrix $\M$ in \eqref{eq_defn_M} can equivalently be written as $\M= \sum_{k=1}^K  \gk \hk^T$. We then observe that each $i$-th row of $\M$ is given by $\sum_k \gk(i) \hk^T$, hence represents the overall spectrum at node $i$ resulting from the contributions of all individual spectra $\hk$ weighted by the membership $\gk(i)$ of the $i$-th node to the $k$-th model. Therefore, the matrix $\M$ provides the vertex-frequency spectrum of the locally stationary graph process $\x$.
  
  We next bound the variation of the local spectra on the graph as follows. Writing 
  \begin{equation}
  \begin{split}
  \tr(\M^T \LG \M) & = \sum_{i,j = 1}^K \tr(\h_i \g_i^T \LG \g_j {\h_j}^T) \\
  & = \sum_{i,j = 1}^K  \g_i^T \LG \g_j \tr(\h_i \h_j^T)
  \end{split}
  \end{equation}
  and noting that $\tr(\h_i \h_j^T) = \left< \h_i, \h_j \right> \triangleq \h_i^T \h_j$, we have
  \begin{equation}
  \begin{split}
  &\tr(\M^T \LG  \M)  = |\tr(\M^T \LG \M)  | = \left | \sum_{i,j = 1}^K  \g_i^T \LG  \g_j \left< \h_i, \h_j \right>  \right | \\
  & \leq \sum_{i,j = 1}^K   \left| \g_i^T \LG \g_j  \right| 
  \left| \left< \h_i, \h_j \right>   \right|
   \leq C  \sum_{i,j = 1}^K \left | \left< \h_i, \h_j \right> \right | \leq K^2 C
  \end{split}
  \end{equation}
  where the last two inequalities follow from the Cauchy-Schwarz inequality and the fact that $\LG$ is a positive semi-definite matrix.\\
\end{proof}

\section*{Appendix C. Extension and Restriction of LSGPs}

Here we present some additional results on the extension and restriction of locally stationary graph processes. Let us consider a graph $\G = (\mathcal{V}, \E, \W)$ with a given subgraph $ \G_s = (\mathcal{V}_s, \E_s, \W_s) $. The relation between the subgraph $\G_s$ and the supergraph $\G$ can be represented through an inclusion map $\iota = (\iota_v,\iota_e,\iota_w): \G_s \hookrightarrow  \G$, where $\iota_v$, $\iota_e$, and $\iota_w$ denote the inclusion maps defined over the vertices, edges, and edge weights, respectively. One can then define a binary selection matrix $\Smat_s \in \{0,1\}^{\crd{\mathcal{V}_s}\times\crd{\mathcal{V}}}$ such that $\Smat_s(v_s,v) = 1$ if and only if $\iota_v(v_s) = v$ for a given enumeration of the vertices in $\mathcal{V}_s$ and  $\mathcal{V}$.

\subsubsection{Extension of LSGPs} We consider an LSGP $\x_s$ on the subgraph $\G_s$ defined by the kernels $ \h_{s,k}$ and the membership functions $\g_{s,k}$ for $k=1, \dots, K$. Due to Theorem \ref{th:vertex_freq_spec}, the process $\x_s$ can be expressed in terms of its vertex-frequency spectrum as
\begin{equation}
    \x_s = \left(\UGs \circ \sum_{k = 1}^K  \g_{s,k} \h_{s,k}^T \right) \UGs^T \ \w_s
\end{equation}
where $\w_s \in \mathbb{R}^{\crd{\mathcal{V}_s}}$ is a unit-variance white process. In order to extend the process  $\x_s$ to the supergraph $\G$, in addition to the inclusion map $\iota $ in the vertex domain, we will also make use of an inclusion map in the frequency domain, defined as $\hat{\iota}:\{1,2,\dots,\crd{\mathcal{V}_s}\} \hookrightarrow \{1,2,\dots,\crd{\mathcal{V}}\}$. The inclusion map $\hat{\iota}$ determines how the frequencies in the spectrum of $\LGs$ should relate to those of $\LG$. We maintain a generic setting  by treating  $\hat{\iota}$ as an arbitrary injection, whose selection is in practice a matter of choice among several possible strategies. The spectral selection matrix $\hat{\Smat} \in \{0,1\}^{\crd{\mathcal{V}_s}\times\crd{\mathcal{V}}}$ associated with $\hat{\iota}$ is given by  $\hat{\Smat}(v_s,v) = 1$ if and only if $\hat \iota(v_s) = v$.
 
 We can then define the extension of the process model $\x_s$ to the supergraph $\G$ as
\begin{equation}
\label{eq:vertex_frequency_extension}
    \x = \left(\UG \circ \sum_{k = 1}^K \Smat_s^T \g_{s,k} \h_{s,k}^T \hat{\Smat}\right) \UG^T \w.
\end{equation}
The extended process model is thus obtained by extending the membership functions and the kernels to the graph $\G$, respectively as $\gk = \Smat_s^T \g_{s,k}$ and $\h_k = \hat{\Smat}^T \h_{s,k}$. Noticing that the choice of $\hat{\iota}$, and thus $\hat{\Smat}$, determines the vertex-frequency spectrum  of the extended process $\x$, the extension operation in \eqref{eq:vertex_frequency_extension} is observed to define an injective map between the set of LSGPs on the subgraph $\G_s$ and the set of LSGPs on $\G$ by retaining the model order $K$. The adoption of the normalized graph Laplacian in our process model provides a convenient basis for the extension procedure in \eqref{eq:vertex_frequency_extension} due to the boundedness of its eigenvalues.  It can be verified that the extended process $\x$ takes the value $0$ on the nodes in $\mathcal{V} \setminus \mathcal{V}_s$.

\subsubsection{Restriction of LSGPs} We next consider an LSGP $\x$ on the supergraph $\G$ given by 
\begin{equation}
    \x = \left(\UG \circ \sum_{k = 1}^K  \gk \hk^T \right) \UG^T \w.   
\end{equation}
Based on the vertex and frequency mappings $\iota$ and $\hat \iota$, we define the restriction of $\x$ to the subgraph $\G_s$ as
\begin{equation}
    \label{eq:vertex_frequency_restriction}
    \x_s = \left(\UGs \circ \sum_{k = 1}^K \Smat_s \gk \hk^T \hat{\Smat}^T\right) \UGs^T \w_s.
\end{equation}
Similarly to the extension procedure, we define the model parameters of the restricted process $\x_s$ as $\g_{s,k} = \Smat_s \g_k$ and $\h_{s,k} = \hat{\Smat} \h_k$, which identify its vertex-frequency spectrum. Hence, the restriction operation defines a surjective map between the set of LSGPs on the supergraph $\G$ and the set of LSGPs on the subgraph $\G_s$. The formulation in \eqref{eq:vertex_frequency_restriction} reveals that the model order of the restricted process is at most $K$. If the restriction leads to a loss of information, it might be possible to represent the resulting process with a smaller order. 

\begin{rmark} Following the definitions of the extension and restriction operations, a pertinent question is whether the extension of a process $\x_s$ from $\G_s$ to $\G$, followed by its restriction back to $\G_s$ preserves it. Note that for irregular graph topologies, the matrices $\UG$ and $\UGs$ typically do not contain any zero entries. In this case, the matrix  $(\Smat_s \UG \hat{\Smat}^T) \circ \UGs $ has no zero entries as well. It is then easy to show that the consequent application of the operations in \eqref{eq:vertex_frequency_extension} and  \eqref{eq:vertex_frequency_restriction} result in an identity morphism on the set of LSGPs on $\G_s$ with model order $K$, hence ensures that the process $\x_s$ is preserved. 
\end{rmark}

\section*{Appendix D. Optimization Problem in Explicit Form}

Here we derive the explicit expressions for the terms $f_1(\Gam, \Bmat)$ and $f_2(\Gam)$ appearing in our problem formulation in Section \ref{sec:learning}. Defining $\Hmat_k= \sum_{q = 0}^{Q-1} b_{q,k} \LG^q$, we have $\Hmat = \sum^{K}_{k = 1} \Gmat_k \Hmat_k$, which gives
 \begin{equation} 
 \begin{split}
 \Hmat \Hmat^T &= \sum_{k=1}^K \sum_{l=1}^K \Gmat_k \Hmat_k \Hmat_l^T \Gmat_l^T \\
&= \sum_{k=1}^K \sum_{l=1}^K \sum_{q=0}^{Q-1}  \sum_{r=0}^{Q-1}  \Gmat_k b_{q,k}b_{r,l} \LG^{q+r} \Gmat_l.
 \end{split}
 \end{equation}
 We proceed by defining $\Zmat_k = [\zeros_{N\times N} \ \dots \ \eye_{N \times N} \ \dots \ \zeros_{N \times N}]$ $\in \mathbb{R}^{N \times NK}$ which contains the identity matrix in its $k$-th block. We then have $\Zmat_k \Gam \Zmat_l^T = \g_k \g_l^T$. Using Lemma \ref{lemma:basic_lemma}, we set $\Dmat = \Gmat_k$, $\Emat = \Gmat_l$, $\Amat = b_{q,k}b_{r,l} \LG^{q+r}$, $\Bmat = \mathbf{1}_{N \times N}$, and manipulate the resulting equation to obtain
 \begin{equation} 
 \begin{split}
 \Hmat \Hmat^T  =& \sum_{k=1}^K \sum_{l=1}^K \sum_{q=0}^{Q-1}  \sum_{r=0}^{Q-1}  \Gmat_k b_{q,k} b_{r,l} \LG^{q+r} \Gmat_l  \\
=& \sum_{k=1}^K \sum_{l=1}^K \sum_{q=0}^{Q-1}  \sum_{r=0}^{Q-1}   (\Zmat_k \Gam \Zmat_l^T)\circ (b_{q,k}b_{r,l} \LG^{q+r}).
 \end{split}
 \end{equation}
 
 Hence, $\Hmat \Hmat^T$ is shown to be a function of $\Gam$. In order to obtain the dependence of $\Hmat \Hmat^T$ on $\Bmat$, we define the matrix
 \[
 \Ymat_k = [\zeros_{Q\times Q} \ \dots  \ \eye_{Q\times Q} \ \dots \ \zeros_{Q \times Q}]^T \in \R^{QK \times Q}
 \] 
 which contains the identity matrix in its $k$-th block, and the vector $\w_q = [0 \ \dots \ 1 \ \dots 0]^T \in \R^{Q \times 1}$,  which contains the value 1 in its $q$-th entry. We thus obtain the function 
 \begin{equation}
 \label{eq:Gamma_B}
 \begin{split}
 f_1(\Gam, \Bmat) = ||\hatCx 
 &-  \sum_{k=1}^K \sum_{l=1}^K \sum_{q=0}^{Q-1}  \sum_{r=0}^{Q-1}  
 (\Zmat_k \Gam \Zmat_l^T) \\
 &\circ (\w_{q+1}^T \Ymat_k^T \Bmat \Ymat_l \w_{r+1} \LG^{q+r})||_F^2
\end{split}
 \end{equation}
 in terms of $\Gam$ and $\Bmat$.
 
Next, for the term $f_2(\Gam)$, we first observe that $f_2(\Gam)= \tr(\Gmat^T \LG  \Gmat) = \sum_{k=1}^K \g_k^T \LG  \g_k$. Writing $\LG = \sum_{i = 1}^N \lamG(i) \, \uv_i \uv_i^T$ and using the equality $\g_k \g_k^T = \Zmat_k \Gam \Zmat_k^T$, we obtain the explicit form of the term $f_2(\Gam)$ in \eqref{eq:optimization_poly} as
  \begin{equation}
 f_2(\Gam)= \tr(\Gmat^T \LG  \Gmat) =  \sum_{k = 1}^K\sum_{i = 1}^N \lamG(i) \, \uv_i^T \Zmat_k \Gam \Zmat_k^T \uv_i.
  \end{equation}

\section*{Appendix E. Complexity Analysis of Algorithm \ref{alg:lsp_learning}}

Here we analyze the complexity of  the LSGP learning method proposed in Algorithm \ref{alg:lsp_learning}. In Algorithm \ref{alg:lsp_learning}, the complexity of the preliminary step of finding the eigenvalue decomposition of $\LG$ is $O(N^3)$. The most significant  stage of the algorithm is Step-1, where we compute the $\Gam$ and $\B$  matrices by solving the optimization problems \eqref{eq:explicit_fd_b}-\eqref{eq:explicit_fd_gamma} based on semidefinite programming (SDP). The commonly used HKM algorithm can be taken as reference for the solution of SDP problems \cite{Toh1999}, whose complexity is $O(mn^3 + m^2 n^2)$ with respect to the number of equality constraints  $m$ and the number of variables $n$. Hence, the complexity of the alternating stages of solving for $\Gam$ and $\B$ can be obtained as $O(\poly(N) K^2)$ and $O(\poly(KQ))$ respectively, where $\poly(\cdot)$ denotes at least cubic polynomial complexity. Next, the complexity of computing rank-1 decompositions for obtaining the model parameters in Step-2 is $O(N^3 K^3)$ for $\g$ and $O(K^3 Q^3)$ for $\bv$. The evaluation of the polynomial functions in Step-3 requires $\Theta(NKQ)$ operations. Once the filter kernels $\h_k$ are found, Step-4 can be executed with a complexity of $O(N^2 K)$. Finally,  the estimation of the operator $\Hmat$ from $\M$ in Step-5 has complexity $O(N^3)$. Hence, assuming that $K, Q \ll N$, the overall complexity of the algorithm can be reported as $O(\poly(N) K^2)$.

 \begin{rmark}
In the above analysis, the primary computational bottleneck of Algorithm \ref{alg:lsp_learning}  is seen to be  Step-1, which has polynomial complexity in the number of nodes $N$. The high complexity in $N$ stems mainly from the nonparametric formulation of the membership functions in our model, which results in $N$ optimization variables to solve for, for each $\gk$ vector. In applications involving large networks, the computational complexity of the algorithm can be alleviated through several strategies. For instance, the membership functions $\gk$  can be formulated in a parametric form with a relatively small number of parameters, e.g., in terms of a linear combination of a small set of localized and smoothly varying graph signal prototypes, such as graph wavelets \cite{Hammond2011} and heat kernels  \cite{Thanou2017}. This would significantly reduce the complexity of Step-1, while preserving the locality and smoothness properties of the membership functions over the graph. Another strategy would be to locally approximate the graph with a smaller subgraph and learn a simpler model on the subgraph. This idea is elaborated in detail in Section \ref{sec:wss_lsp}.

 \end{rmark}

\section*{Appendix F. Proof of Theorem \ref{theorem:arbitrary_clusters}}

\begin{proof}

The cross-covariance matrix of the component processes $\xk$ and $\xm$ is given by 
  \begin{equation}
\Ckm = \UG  h_k(\LamG) h_m(\LamG) \UG^T.
  \end{equation}
The element-wise magnitude of the matrix $\Ckm$ can be bounded as
  \begin{equation}
    \label{eq:ckl_norm1}
  \begin{split}
 &   | \Ckm  | =  \left| \sum_{i = 1}^N \h_k(i) \h_m(i)  \uv_i \uv_i^T \right|
     \preccurlyeq  \sum_{i = 1}^N | \h_k(i) \h_m(i) \uv_i \uv_i^T | \\
     & \preccurlyeq  \sum_{i = 1}^N | \h_k(i) \h_m(i) |  \mathbf{1}_{N \times N}
     \preccurlyeq  \frac{1}{2} \sum_{i = 1}^N \left(| \h_k^2(i)| + |\h_m^2(i)| \right) \mathbf{1}_{N \times N} \\
   &  \preccurlyeq \mathbf{1}_{N \times N}
  \end{split}
  \end{equation}
where the second and the fourth inequalities follow respectively from the fact that the vectors $\uv_i$ and $\hk$ are unit-norm.

Next, we write the covariance matrix $\Cx$ of the process $\x$ as a weighted average of the cross-covariances $\Ckm$ as
\begin{equation}
  \label{eq:lsp_cx}
    \Cx = \sum_{k=1}^K \sum_{m = 1}^K \Gmat_k \Ckm  \Gmat_m^T.
\end{equation}
We can then bound the deviation between $\Ckm$ and the restriction of $\Cx$ to the subgraphs $\G_k$ and $\G_m$ as
  \begin{equation}
    \label{eq:bound_coarse}
    \begin{split}
    &\left| \Smat_k \invGk \Cx (\invGm)^T \Smat_m^T - \Smat_k \Ckm  \Smat_m^T\right| \\
    & = \left| \sum_{ (i,j) \neq (k,m)} \Smat_k \invGk  \Gmat_i \covMat_{\x_i \x_j} \Gmat_j^T (\invGm)^T \Smat_m^T \right|\\
    &  {\preccurlyeq} \sum_{(i,j) \neq (k,m)} | \Smat_k \invGk  \Gmat_i |   |\covMat_{\x_i \x_j}|  | \Gmat_j^T (\invGm)^T \Smat_m^T | \\
    &  {\preccurlyeq} \sum_{ (i,j) \neq (k,m)} |\Smat_k \invGk  \Gmat_i| \, \mathbf{1}_{N\times N} \, |\Gmat_j^T (\invGm)^T \Smat_m^T| 
    \end{split}
  \end{equation}
  where the first and the second inequalities are due to Lemmas \ref{lemma:absolute_matrix} and  \ref{lemma:absolute_matrix_2}, respectively. In order to bound the expression in \eqref{eq:bound_coarse} in terms of $\delta$ and $\mu$, we next examine the product $|\Smat_k \invGk \Gmat_i |$ for the cases $i\neq k$ and $i=k$. Due to Assumption \ref{ass:vertex_separation}, we have $|\Smat_k \invGk \Gmat_i| \preccurlyeq \frac{1}{\mu} \Smat_k |\Gmat_i| \preccurlyeq \frac{\delta}{\mu} \Smat_k$ for $i \neq k$; and $|\Smat_k \invGk \Gmat_i| \preccurlyeq \Smat_k$ for $i=k$. Using these inequalities in \eqref{eq:bound_coarse}, we get
  \begin{equation}
    \begin{split}
    &\left| \Smat_k \invGk \Cx (\invGm)^T \Smat_m^T - \Smat_k \Ckm  \Smat_m^T\right| \\
    & \preccurlyeq \sum_{ (i,j) \neq (k,m)} |\Smat_k \invGk  \Gmat_i| \, \mathbf{1}_{N\times N} \, |\Gmat_j^T (\invGm)^T \Smat_m^T| \\
    & = \sum_{i\neq k, j = m}  |\Smat_k \invGk  \Gmat_i| \, \mathbf{1}_{N\times N} \, |\Gmat_j^T (\invGm)^T \Smat_m^T| \\
    & \sum_{i = k, j \neq m}  |\Smat_k \invGk  \Gmat_i| \, \mathbf{1}_{N\times N} \, |\Gmat_j^T (\invGm)^T \Smat_m^T| \\
    & + \sum_{i \neq k, j \neq m}  |\Smat_k \invGk  \Gmat_i| \, \mathbf{1}_{N\times N} \, |\Gmat_j^T (\invGm)^T \Smat_m^T|\\  
    & \preccurlyeq \left(2 (K-1) \frac{\delta}{\mu} + (K-1)^2 \left(\frac{\delta}{\mu} \right)^2 \right) \Smat_k \mathbf{1}_{N \times N} \Smat_m^T\\
     &=  \left(2 (K-1) \frac{\delta}{\mu} + (K-1)^2 \left(\frac{\delta}{\mu} \right)^2 \right)  \mathbf{1}_{|\mathcal{V}_k | \times |\mathcal{V}_m |} 
    \end{split}
  \end{equation}
which concludes the proof. 
\end{proof}

\section*{Appendix G. Proof of Theorem \ref{cor:average_cov_bound}}

\begin{proof}
We begin by observing that for any $a,b \in \mathbb{R}$, we have
\begin{equation}
    |a + b|^2 \leq |a|^2 + |b|^2 + 2|ab| \leq 2 \left(|a|^2 + |b|^2\right).
\end{equation}
We can bound the average squared cross-covariance as 
\begin{equation}
\begin{split}
\label{eq:average_overall_sq}
        & \frac{1}{\mub ^4} \sum_{k=1}^K    \sum_{\substack{m=1\\ m \neq k}}^K \ \sum_{(i,j) \in \mathcal{V}_k \times \mathcal{V}_m} 
        \left| \Cx (i,j) \right|^2 \\
        & \leq  \sum_{k=1}^K    \sum_{\substack{m=1\\ m \neq k}}^K \ \sum_{(i,j) \in \mathcal{V}_k \times \mathcal{V}_m}  \g_k(i)^{-2} \left| \Cx (i,j) \right|^2 \g_m(j)^{-2}\\
        & \leq 2  \sum_{k=1}^K    \sum_{\substack{m=1\\ m \neq k}}^K \ \sum_{(i,j) \in \mathcal{V}_k \times \mathcal{V}_m} 
         \big| \g_k(i)^{-1} \Cx(i,j) \g_m(j)^{-1} \\
        & \quad \quad - \Ckm(i,j) \big|^2 
        + 2  \sum_{k=1}^K    \sum_{\substack{m=1\\ m \neq k}}^K \  \sum_{(i,j) \in \mathcal{V}_k \times \mathcal{V}_m} 
        \left|  \Ckm(i,j)\right|^2 \\
        &\leq 2  \sum_{k=1}^K    \sum_{\substack{m=1\\ m \neq k}}^K \ \sum_{(i,j) \in \mathcal{V}_k \times \mathcal{V}_m}       
         \left(2 (K-1) \frac{\delta}{\mu} + (K-1)^2 \left(\frac{\delta}{\mu}\right)^2 \right)^2 \\
        &+ 2 \sum_{k=1}^K    \sum_{\substack{m=1\\ m \neq k}}^K \ \sum_{(i,j) \in \mathcal{V}_k \times \mathcal{V}_m}     
        \left|   \Ckm(i,j) \right|^2 \\
\end{split}
\end{equation}
where the last inequality is due to Theorem \ref{theorem:arbitrary_clusters}. We proceed by upper bounding the cross-covariance sum as
\begin{equation}
\begin{split}
        & \sum_{k=1}^K    \sum_{\substack{m=1\\ m \neq k}}^K \ \sum_{(i,j) \in \mathcal{V}_k \times \mathcal{V}_m} \left|  \Ckm(i,j) \right|^2 \\
        & =  \sum_{k=1}^K    \sum_{\substack{m=1\\ m \neq k}}^K \ \sum_{(i,j) \in \mathcal{V}_k \times \mathcal{V}_m}   \left|\sum_{l = 1}^N \h_k(l) \h_m(l) \uv_l(i) \uv_l(j)\right|^2 \\
        & \leq \sum_{k=1}^K    \sum_{\substack{m=1\\ m \neq k}}^K \ \sum_{(i,j) \in \mathcal{V}_k \times \mathcal{V}_m}  \left(\sum_{l = 1}^N \h_k(l)^2 \h_m(l)^2 \right) \\
        & \hspace{5cm}
        \cdot \left(\sum_{n = 1}^N \uv_n(i)^2 \uv_n(j)^2 \right) \\
        &= \sum_{l=1}^N \sum_{n=1}^N   \sum_{k=1}^K    \sum_{\substack{m=1\\ m \neq k}}^K   \h_k(l)^2 \h_m(l)^2
         \sum_{(i,j) \in \mathcal{V}_k \times \mathcal{V}_m} \uv_n(i)^2 \uv_n(j)^2 \\ 
        &= \sum_{l=1}^N \sum_{n=1}^N
        \sum_{k=1}^K    \sum_{\substack{m=1\\ m \neq k}}^K
         \h_k(l)^2 \h_m(l)^2
         \sum_{i \in \mathcal{V}_k} \uv_n(i)^2  \sum_{j \in \mathcal{V}_m} \uv_n(j)^2 \\
        &\leq N \sum_{l=1}^N \sum_{k=1}^K    \sum_{\substack{m=1\\ m \neq k}}^K  
        \h_k(l)^2 \h_m(l)^2 \leq N K(K-1) \epsilon^2 \\
\end{split}
\end{equation}
where the last inequality is due to Assumption \ref{ass:frequency_separation}. Using this result in \eqref{eq:average_overall_sq}, we get the bound stated in the theorem

\begin{equation}
    \begin{split}
        &\frac{1}{\mub^4 N^2} 
        \sum_{k=1}^K    \sum_{\substack{m=1\\ m \neq k}}^K \ \sum_{(i,j) \in \mathcal{V}_k \times \mathcal{V}_m} 
         \left|  \Cx(i,j)   \right|^2\\
        & \leq \frac{2}{N}K(K-1)\epsilon^2\\
        & + \frac{2}{N^2} 
         \sum_{k=1}^K    \sum_{\substack{m=1\\ m \neq k}}^K \ \sum_{(i,j) \in \mathcal{V}_k \times \mathcal{V}_m} 
         \left(2 (K-1) \frac{\delta}{\mu} + (K-1)^2 \left(\frac{\delta}{\mu}\right)^2 \right)^2\\
        & = \frac{2}{N}K(K-1)\epsilon^2 \\
        & + \frac{2}{N^2} \left|  \bigcup_{k=1}^K \bigcup_{\substack{m=1\\ m \neq k}}^K  \mathcal{V}_k \times \mathcal{V}_m \right| \left(2 (K-1) \frac{\delta}{\mu} + (K-1)^2 \left(\frac{\delta}{\mu}\right)^2 \right)^2 .\\
\end{split}
\end{equation}
\end{proof}

\section*{Appendix H. Lower bound on the within-subgraph covariance magnitudes }

While Theorem \ref{cor:average_cov_bound} guarantees an upper bound on the cross-covariance, in order for this bound to be meaningful, it should be assessed relatively to the covariance of $\x$ on each subgraph $\G_k$. Therefore, in this section we aim to get a lower bound for the average strength of $\Cx$ over the individual subgraphs $\G_k$. In order to ensure such a lower bound, the process $\x$ must change sufficiently slowly on each subgraph $\G_k$, which can be imposed through a restriction on the bandwidths of the kernels $\hk$:

\begin{assumption}
  \label{ass:lowpass_kernel}
The kernels $\hk$ are band-limited such that $\h_k(i) = 0$ for $i >{\Kcut }$ for all $k \in \{1,2, \dots, K\}$, where $\Kcut $ is a cutoff parameter with $\Kcut \in \{1,2,\dots, N\}$.
\end{assumption}
 Before proceeding to our result, we also define the following parameters related to the topology of $\G$ and the process characteristics:
 \begin{definition}
 Let $D(i,j)$ denote the unweighted geodesic distance between two vertices $i,j \in \mathcal{V}$ given by 
 \begin{equation*}
 \begin{split}
 &D(i,j) \triangleq \min 
  \{ n: 
 \exists \ (l_0, l_1, \dots, l_n) \text{ such that } l_k \sim  l_{k+1} \\ 
 & \text{ for } k=0, \dots, n-1 ;  
  l_k\in \{1, \dots, N  \}; \  l_0=i, \  l_n=j \}.
\end{split}
 \end{equation*}
 Also let $w_{min} \triangleq \underset{i \sim j}{\min} \, \W(i,j)$ denote the minimum edge weight on $\G$, let $\tv_n \triangleq \sum_{i \sim j} \W(i,j) (\uv_n(i) - \uv_n(j))^2$ denote the total variation of $\uv_n$ on $\G$,   and let  $\sigma_k^2 \triangleq \sum_{i \in  \V_k} \Cxk (i,i)$ represent the total variance of the process $\xk$ on the subgraph $\G_k$.
 \end{definition}

We can now present our lower bound on the average magnitude of the process covariance on the individual subgraphs. 
 \begin{theorem}
 \label{theorem:cov_avg_low}
Let Assumptions \ref{ass:vertex_separation} and \ref{ass:lowpass_kernel} hold. Then the average magnitude of the covariance of the process $\x$ on the individual subgraphs $\{ \G_k \} $ is lower bounded as
      \begin{equation}
      \label{eq_thm_lb_cov}
      \begin{split}
             & \frac{1}{N^2 \mu^2}\sum_{k = 1}^K \sum_{(i,j) \in \mathcal{V}_k \times \mathcal{V}_k} | \Cx(i, j) |
             \geq             
            \frac{1}{N^2}\sum_{k = 1}^K | \mathcal{V}_k |   \sigma_k^2 \\
             & - \frac{1}   {2 \, N^2 \, w_{min}  }  \
           \sum_{k = 1}^K \ \sum_{(i,j) \in \mathcal{V}_k \times \mathcal{V}_k} D^2(i, j) \
           \sum_{n= 1}^{\Kcut } \tv_n \\
             & -  \frac{1}{N^2}  \left|\bigsqcup_{k = 1}^K \mathcal{V}_k \times \mathcal{V}_k \right|  \left(2(K-1) \frac{\delta}{\mu} + (K-1)^2 \left( \frac{\delta}{\mu}\right)^2 \right). \\
      \end{split}
     \end{equation}
 \end{theorem}
The first term in the right hand side of \eqref{eq_thm_lb_cov} sets a reference value for the average covariance magnitude of $\x$ when restricted to the subgraphs $\{\G_k\}$. The average covariance magnitude has limited deviation from this reference value if the second and the third terms have restricted magnitudes. The second term improves as the bandwidth parameter $\Kcut$ of the kernels decreases, which also depends on the diameters of the subgraphs. The third term is bounded by the localization ratio $\delta/\mu$ of the membership functions as in Theorem \ref{cor:average_cov_bound}.

Before proving Theorem \ref{theorem:cov_avg_low}, we first present the following lemma, which will be useful in the proof.

\begin{lemma}
\label{claim:combinatoric_smoothness}
The following inequality holds for all edges $ (i,j) \in \mathcal{E}$ and all $n\ \in \{1, \dots, N \}$. 
\begin{equation}
\label{eq:smooth_upper_edge}
    \frac{\tv_n}{w_{min}} \geq (\uv_n(i)-\uv_n(j))^2
\end{equation}
Also, for all vertex pairs $(i, j) \in \mathcal{V} \times \mathcal{V}$,

\begin{equation}
\label{eq:smooth_upper_v2}
    D^2(i,j) \frac{\tv_n}{w_{min}} \geq ( \uv_n(i)- \uv_n(j))^2.
\end{equation}

\end{lemma}

\begin{proof}
We have 
\begin{equation}
   \begin{split}
       \tv_n & =  \sum_{i \sim j} \W(i,j) (\uv_n(i) - \uv_n(j))^2 \\
       & \geq w_{min} \sum_{i \sim j} ( \uv_n(i) - \uv_n(j))^2 \\
       & \geq w_{min} \, \underset{i \sim j}{\max} \  ( \uv_n(i) - \uv_n(j))^2. \\
   \end{split}
\end{equation}
Hence, \eqref{eq:smooth_upper_edge} is proved. For showing \eqref{eq:smooth_upper_v2}, consider a simple path $(i, l_1, \dots, l_{q-1}, j)$ of length $q=D(i,j)$ between nodes $i$ and $j$. Then,
\begin{equation}
  \begin{split}
     &D(i,j) \sqrt{\frac{\tv_n}{w_{min}}} \\
     & \geq |\uv_n(i) - \uv_n(l_1)| + \dots + |\uv_n(l_{q-1}) - \uv_n(j)| \\ 
     & \geq |\uv_n(i) - \uv_n(j)| .\\
  \end{split}
\end{equation}
Taking the square of both sides, we get the inequality in \eqref{eq:smooth_upper_v2}.
\end{proof}

We can now prove Theorem \ref{theorem:cov_avg_low}.

\begin{proof}
We first obtain an expression for the total covariances of the component processes on their corresponding subgraphs as
\begin{equation}
\label{eq_Cxij_lb_v1}
\begin{split}
    &\sum_{k = 1}^K \ \sum_{(i,j) \in \mathcal{V}_k \times \mathcal{V}_k} \Cxk(i,j) \\
    &=  \sum_{k = 1}^K \ \sum_{(i,j) \in \mathcal{V}_k \times \mathcal{V}_k}
     \sum_{n = 1}^{\Kcut}
      \h_k(n)^2 \, \uv_n(i) \, \uv_n(j) \\
     & = \frac{1}{2} 
     \sum_{k = 1}^K \ \sum_{(i,j) \in \mathcal{V}_k \times \mathcal{V}_k}
     \sum_{n = 1}^{\Kcut}
      \h_k(n)^2  \left( \uv_n(i)^2 + \uv_n(j)^2  \right) \\
     & - \frac{1}{2}
      \sum_{k = 1}^K \ \sum_{(i,j) \in \mathcal{V}_k \times \mathcal{V}_k}
     \sum_{n = 1}^{\Kcut}
      \h_k(n)^2  \left( \uv_n(i)  -  \uv_n(j) \right)^2
\end{split}
\end{equation}
where the first equality is due to Assumption \ref{ass:lowpass_kernel}. We first obtain the following relation 
\begin{equation}
\label{eq:lower1}
\begin{split}
     &
     \sum_{k = 1}^K \ \sum_{(i,j) \in \mathcal{V}_k \times \mathcal{V}_k}
     \sum_{n = 1}^{\Kcut}     
     \h_k(n)^2  \uv_n(i)^2 \\
    & =  \sum_{k = 1}^K  \sum_{i \in \V_k}    \sum_{n = 1}^{\Kcut}  
      |\mathcal{V}_k | \ \h_k(n)^2  \uv_n(i)^2 \\
     & =\sum_{k = 1}^K  \sum_{i \in \V_k}
      |\mathcal{V}_k|   \Cxk(i,i) 
      = \sum_{k = 1}^K  |\mathcal{V}_k | \, \sigma_k^2.
\end{split}
\end{equation}

Next, from Lemma \ref{claim:combinatoric_smoothness}, we get 
\begin{equation}
\label{eq:lower2}
\begin{split}
       -\sum_{n= 1}^{\Kcut} \h_k(n)^2 (\uv_n(i)-\uv_n(j))^2 & \geq -\sum_{n = 1}^{\Kcut}  (\uv_n(i)-\uv_n(j))^2 \\
       & \geq - \frac{D^2(i,j)}{w_{min}} \sum_{n = 1}^{\Kcut} \tv_n.
\end{split}
\end{equation}
Using \eqref{eq:lower1} and \eqref{eq:lower2} in \eqref{eq_Cxij_lb_v1}, we obtain 

 \begin{equation}
 \begin{split}
& \sum_{k = 1}^K \ \sum_{(i,j) \in \mathcal{V}_k \times \mathcal{V}_k} \Cxk(i,j)   
    \geq \sum_{k = 1}^K |\mathcal{V}_k| \sigma_k^2 \\
   & - \frac{1}{2 w_{min}} \sum_{k = 1}^K \ \sum_{(i,j) \in \mathcal{V}_k \times \mathcal{V}_k} D^2(i, j) \sum_{n = 1}^{\Kcut} \tv_n.
 \end{split}
 \end{equation}
Finally, from Assumption \ref{ass:vertex_separation} it follows that
 \begin{equation}
 \label{eq:reverse_tri_ALT}
 \begin{split}
  &\frac{1}{\mu^2} | \Cx(i,j) | \geq | \Gmat_k^{\dag}(i,i) \Cx(i,j) \Gmat_k^{\dag}(j,j) | \\
 & \geq  \Gmat_k^{\dag}(i,i) \Cx(i,j) \Gmat_k^{\dag}(j,j) \\
&  \geq  \Cxk(i,j) - |\Gmat_k^{\dag}(i,i) \Cx(i,j) \Gmat_k^{\dag}(j,j) - \Cxk(i,j)|.
 \end{split}
 \end{equation}
Using this result together with the bound in  Theorem \ref{theorem:arbitrary_clusters}, we get the inequality stated in Theorem  \ref{theorem:cov_avg_low}.

\end{proof}

\section*{Appendix I. Additional performance analyses for the proposed algorithms }

\subsection{Performance Analysis of Algorithm \ref{alg:lsp_learning}}

\subsubsection{Effect of model complexity}
\label{sec:effect_model}
Here we examine the effect of the model order parameters $K$ (number of process components) and $Q$ (polynomial order) on the estimation performance of the LSGP method proposed in Algorithm \ref{alg:lsp_learning}. A synthetic  $7$-NN graph $\G$ with $N=36$ nodes is formed as in Section \ref{sec:perf_noise} by combining $K$ subgraphs $\{\G_k\}_{k=1}^K$ each of which consists of $36/K$ nodes. The component processes $\xk$ are blended in $\x$ by setting the membership functions $\gk$ to $1$ within each subgraph $\G_k$ and to $0.1$ outside $\G_k$.

The variation of the covariance discrepancy is plotted in Fig.~\ref{fig_cd_variation_constant_Q} for variable $K$ by fixing $Q = 4$, and for variable $Q$ in Fig.~\ref{fig_cd_variation_constant_K} by fixing $K = 2$. As the model complexity increases, the number of realizations required to attain a target covariance discrepancy level increases in both cases as expected. The CD converges to 0 with increasing number of realizations in all cases, which confirms that Algorithm \ref{alg:lsp_learning} recovers the true process model.  The algorithm performance is more sensitive to the $K$ parameter than $Q$, which is expected since increasing $K$ by 1 increases the model dimension by $N$.

\begin{figure}[t]
  \begin{center}
    \hspace{-0.7cm}
       \subfloat[]
       {\label{fig_cd_variation_constant_Q}\includegraphics[height=3.5cm]{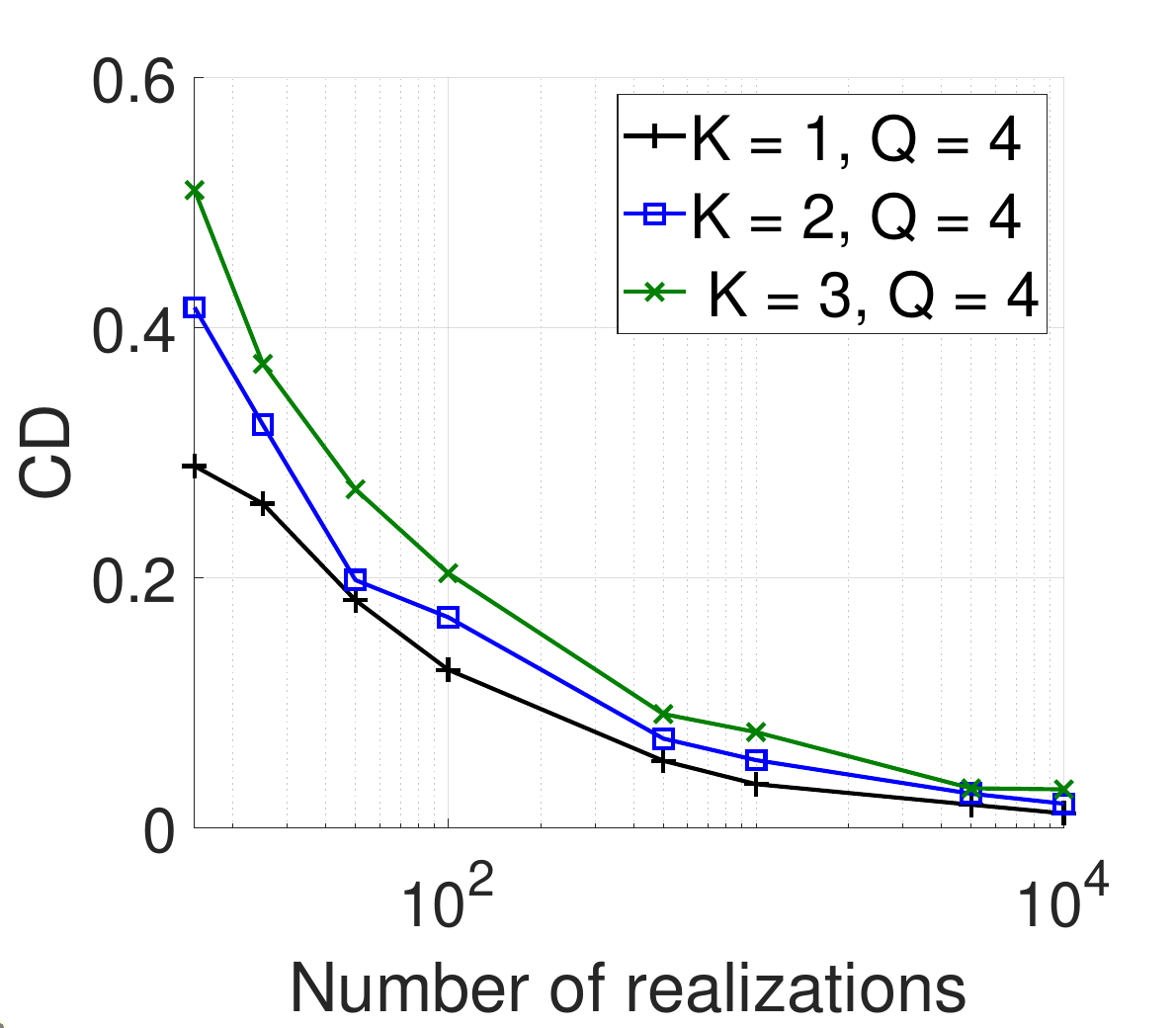}}
       \hspace{0.1cm}
      \subfloat[]
      {\label{fig_cd_variation_constant_K}\includegraphics[height=3.5cm]{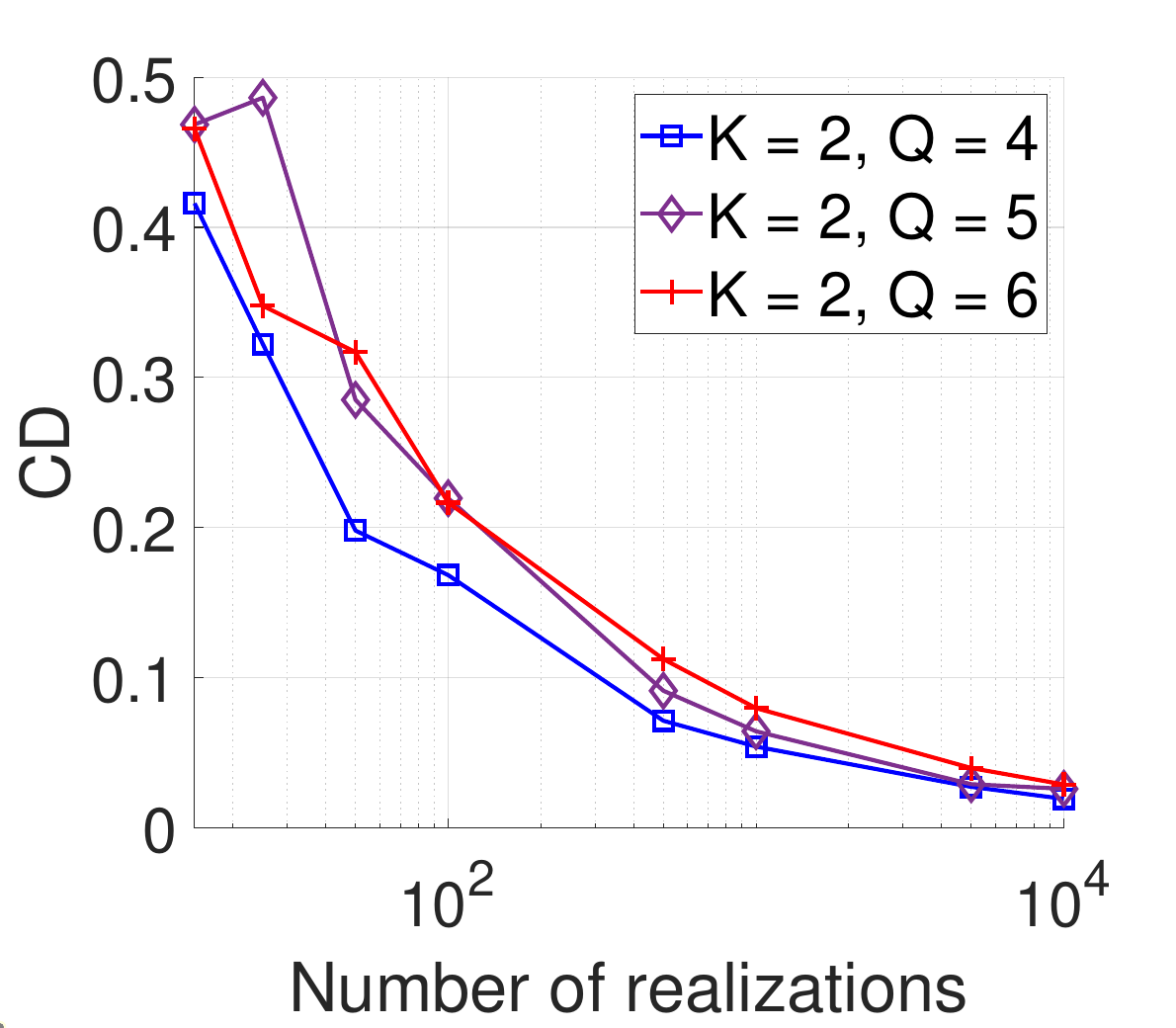}}\\
   \end{center}
    \caption{Variation of the covariance discrepancy with the model complexity}

   \label{fig_model_complexity}
  \end{figure}

\subsubsection{Sensitivity to regularization parameters}
\label{sec:mu_sensitivity}

We then investigate the effect of the weight parameters $\mu_1, \mu_2, \mu_3$ in the optimization problem \eqref{eq:optimization_poly} on the algorithm performance. We experiment on the COVID-19 and the Mol\`ene data sets described in Section \ref{sec:comparative_experiments}. The ratio of missing observations is fixed to $80\%$ and the NME and MAPE metrics for the LMMSE estimates are reported in Table \ref{table:mu_1_nme} for varying $\mu_1$ values and in Table \ref{table:mu_23_nme} for  varying $\mu_2, \mu_3$ combinations. The non-tested weight parameters are fixed to $0$ in each experiment, in order to focus on the effect of the tested ones.

In Table \ref{table:mu_1_nme}, the performance of the algorithm is seen to be stable with respect to the variations in $\mu_1$  over a rather large interval,  which controls the smoothness of the membership functions. The COVID-19 data set favors slightly lower  $\mu_1$ values compared to Mol\`ene, thus imposing the smoothness less strictly. This is in line with the finding that the COVID-19 data has weaker vertex stationarity than Mol\`ene \cite{Guneyi2023}.  A suitable choice for $\mu_1$ would lie in the interval $[10^{-8}, 10^{-6}]$, finding a trade-off between the NME and the MAPE metrics. Next, the results in Table \ref{table:mu_23_nme} show that relatively small  $(\mu_2,\mu_3)$ values lead to smaller NME. Recalling that these parameters impose the low-rank constraints in \eqref{eq:optimization_poly}, the decrease in the MAPE values for increasing  $(\mu_2,\mu_3)$ is misleading as the algorithm tends to compute a zero process model under too heavy regularization.  Taking into account both metrics, one may select $\mu_2$ in $[10^{-4}, 10^{-5}]$ and $\mu_3$ in $[0, 10^{-5}]$.

\begin{table}[t]
  \tiny
  \center
\begin{tabular}{|c|c|c|c|c|c|c|c|c|}
\hline
$\mu_1$ & 0&$10^{-8}$&$10^{-7}$&$10^{-6}$&$10^{-5}$&$10^{-4}$&$10^{-3}$&$10^{-2}$\\
\hline
Dataset & \multicolumn{8}{c|}{COVID-19}\\
\hline
NME & 0.8091 & 0.8091 & 0.8091 & 0.8091 & \underline{\textbf{0.8091}} & 0.8092 & 0.8130 & 0.8210 \\
\hline
MAPE & 2.1454 &\underline{\textbf{2.1454}} &2.1455 & 2.1459 & 2.1503 & 2.2075 & 2.4162 & 2.5038 \\
\hline
Dataset & \multicolumn{8}{c|}{Mol\`ene}\\
\hline
NME & 0.4470 & 0.4470 & 0.4470 & 0.4470 & 0.4470  & 0.4470  & 0.4465 & \underline{\textbf{0.4438}}\\
\hline
MAPE & 1.0257 &\underline{\textbf{1.0257}} & 1.0257 & 1.0257 & 1.0257 & 1.0258 & 1.0267 & 1.0369 \\
\hline
\end{tabular}
\vspace{0.1cm}
\caption{\label{table:mu_1_nme} Variation of the NME and MAPE with  $\mu_1$}
\end{table}

\begin{table}[t]
  \tiny
  \center
\begin{tabular}{|c|c|c|c|c|c|c|c|c|}
\hline
\diagbox[width = 1cm, height = 0.5cm]{$\mu_2$}{$\mu_3$}&0&$10^{-7}$&$10^{-6}$&$10^{-5}$&$10^{-4}$&$10^{-3}$&$10^{-2}$&$10^{-1}$\\
\hline
Metric & \multicolumn{8}{c}{NME (COVID-19)} \vline\\
\hline
0&0.8091&0.8091&0.8091&0.8091&0.8091&0.8091&9.5685&0.8091\\ 
\hline
$10^{-7}$&0.8091&0.8091&0.8091&0.8091&0.8091&0.8091&0.8093&0.8096\\ 
\hline
$10^{-6}$&0.8091&0.8091&0.8091&0.8091&0.8091&0.8091&0.8089&0.8118\\
\hline
$10^{-5}$&0.8091&0.8091&0.8091&0.8091&0.8093&0.8097&0.8089&0.8112\\ 
\hline
$10^{-4}$&0.8090&0.8090&0.8090&0.8090&0.8095&0.8101&0.8104&0.8328\\ 
\hline
$10^{-3}$&0.8096&0.8096&0.8096&0.8101&0.8098&0.8116&0.8133&0.9125\\ 
\hline 
$10^{-2}$&\textbf{0.8084}&\textbf{0.8084}&\underline{\textbf{0.8084}}&\textbf{0.8085}&0.8091&0.8126&0.9107&0.9392\\ 
\hline
$10^{-1}$&0.8105&0.8093&0.8095&0.8113&0.8126&0.9116&0.9533&0.9950\\ 
\hline
Metric & \multicolumn{8}{c}{MAPE (COVID-19)} \vline\\
\hline
0&2.1454&2.1454&2.1454&2.1454&2.1454&2.1454&2.6417&2.1456\\ 
\hline
$10^{-7}$&2.1455&2.1455&2.1455&2.1455&2.1456&2.1458&2.1298&2.2688\\
\hline
$10^{-6}$&2.1460&2.1460&2.1460&2.1460&2.1469&2.1492&2.1489&2.2752\\
\hline
$10^{-5}$&2.1508&2.1508&2.1508&2.1512&2.1530&2.1384&2.1703&2.3944\\ 
\hline
$10^{-4}$&2.1959&2.1959&2.1960&2.2140&2.2051&2.2234&2.2491&2.1158\\  
\hline
$10^{-3}$&2.3035&2.3034&2.3029&2.2797&2.2344&2.3254&2.3276&1.6488\\
\hline 
$10^{-2}$&2.2116&2.2118&2.2136&2.2237&2.2798&2.4312&1.7388&1.8048\\ 
\hline
$10^{-1}$&2.1697&2.1727&2.2483&2.2779&2.4244&1.7243&1.8460&\underline{\textbf{1.0557}}\\
\hline
\end{tabular}
\vspace{0.1cm}
\caption{\label{table:mu_23_nme} Variation of the NME and MAPE with $\mu_2$ and $\mu_3$}
\end{table}

\subsection{Performance Analysis of Algorithm \ref{alg:subgraph_clustering}}
\label{sec:exp_clustering}
We next verify the validity of our theoretical findings in Section \ref{sec:wss_lsp} by conducting a performance analysis of Algorithm \ref{alg:subgraph_clustering}. We construct a synthetic graph similar to the one in the model complexity experiments, consisting of $K = 5$ subgraphs and a total of $N = 300$ nodes. The subgraphs $\G_k$ are built with a $7$-NN connectivity pattern within themselves and combined with each other via extra edges. The membership functions $\gk$ to the component processes are also chosen as in the model complexity experiments. The spectral kernels $\hk$ are set by shifting and scaling a compactly supported bump function according to the intended spectral separation, which is defined as $h(\lambda)=\exp( (\lambda^{2n}-1)^{-1})$ for $\lambda \in (-1,1)$ and 0 elsewhere.

We study the problem of graph partitioning for locally approximating LSGPs  and examine how the membership ratio $\delta/\mu$ and the spectral separation parameter $\epsh$ influence the performance of Algorithm \ref{alg:subgraph_clustering}. In each instance of the experiment, LSGPs with the investigated $\delta/\mu$ and $\epsh$ parameters are generated;  Algorithm \ref{alg:subgraph_clustering} is provided the true covariance matrix $\Cx$ of the process as input; and the subgraphs $\{ \hat{\G}_k \}_{k=1}^K$ returned by the algorithm are compared to the true subgraphs $\{ \G_k \}$.  We evaluate the agreement between the true and the estimated subgraphs using the normalized mutual information (NMI) measure defined as
\begin{equation}
  \label{eq:nmi}
  \text{NMI} = \frac{1}{\text{max}(H(\mathcal{P}_\V),H(\mathcal{P}_{\hat{\V}})} \sum_{k,m} p(\V_k, \hat{\V}_m) \text{log}_2 \frac{p(\V_k, \hat{\V}_m)}{p(\V_k) p(\hat{\V}_m)}
\end{equation}
where $\mathcal{P}_{\V} = \{\V_1, \dots, \V_K\}$ is the true partitioning of the vertices and $\mathcal{P}_{\hat\V} = \{\hat\V_1, \dots, \hat\V_K\}$ denotes its estimate. In \eqref{eq:nmi}, $p(\V_k)$ and $p(\V_k, \V_m)$  represent the probability that a vertex chosen uniformly at random lies in $\V_k$ and in $\V_k \cap \V_m$, respectively. Also,
$
  H(\mathcal{P}_\V) = -\sum_{k} p(\V_k) \text{log}_2 p(\V_k)
$
denotes the entropy of a partitioning. A higher value of the NMI thus indicates a stronger agreement between the two partitions.

Table \ref{table:nmi_values} shows the variation of the NMI with the parameters $\delta/\mu$ and $\epsh$. As the $\delta/\mu$ ratio increases, the dominance of each membership function on the corresponding subgraph weakens, leading to a decrease in the NMI. This is coherent with the findings of Theorems \ref{cor:average_cov_bound} and \ref{theorem:cov_avg_low}, stating that $\delta/\mu$ must be low for ensuring weak between-subgraph and strong within-subgraph covariances. Similarly, the NMI decreases as $\epsh$ increases, which reduces the separation between the kernels and affects the partitioning performance, in coherence with Theorem \ref{cor:average_cov_bound}.

\begin{table}[t]
  \scriptsize
  \center
\begin{tabular}{|c|c||c|c|}
\hline
\multicolumn{2}{|c||}{\textbf{Membership ratio}} & 
\multicolumn{2}{c|}{\textbf{Spectral separation}} \\ 
\hline
 $\frac{\delta}{\mu}$& NMI & $\epsh$ & NMI \\
\hline
0 &  0.9764& 0.1 & 0.9822\\
\hline
0.13 &  0.979& 0.2 & 0.9853\\
\hline
0.27 &  0.965& 0.3 & 0.9740\\
\hline
0.4 &  0.962& 0.4 &  0.9488\\
\hline
0.53 &  0.918& 0.5 &  0.9559\\
\hline
0.67 &  0.917& 0.6 & 0.9489\\
\hline
0.8 &  0.853& 0.7 & 0.9306\\
\hline
\end{tabular}
\vspace{0.1cm}
\caption{\label{table:nmi_values} Variation of  the NMI with $\frac{\delta}{\mu}$ and $\epsh$}
\end{table}

\section*{Appendix J. Additional Comparative Experiments}

In this section, we provide some additional results complementary to the comparative experiments in  Section \ref{sec:comparative_experiments}. In Figures \ref{fig_maes} and \ref{fig_mapes}, we provide a comparison of the tested algorithms with respect to the MAE and the MAPE metrics on the COVID-19 and the Mol\`ene data sets.

\begin{figure}[ht]
  \begin{center}
    \hspace{-0.7cm}
       \subfloat[COVID-19 (Random data loss)]
       {\label{fig_mae_covid_sp}\includegraphics[height=3.5cm]{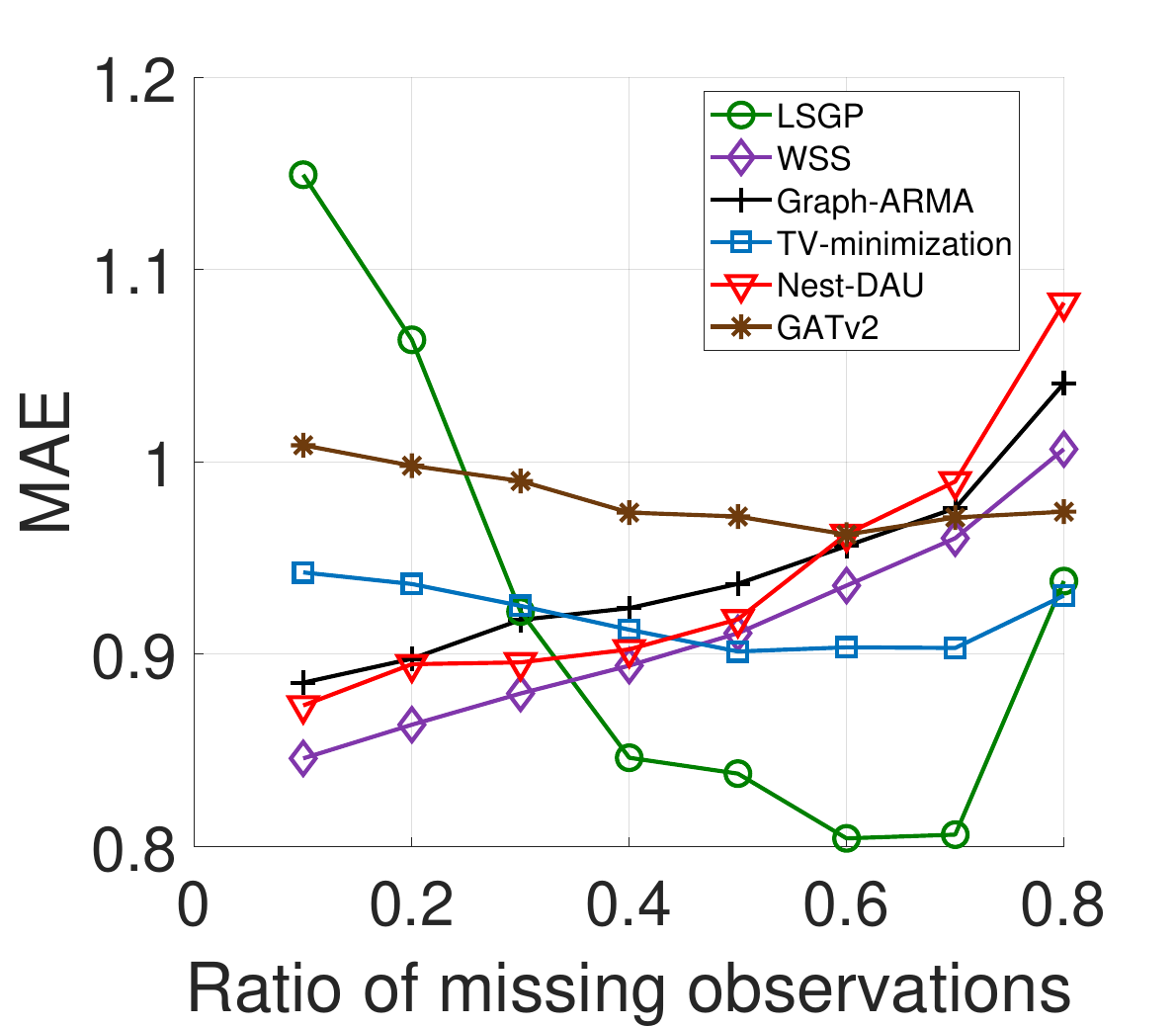}}
       \hspace{0.1cm}
      \subfloat[Mol\`ene (Random data loss)]
      {\label{fig_mae_molene_sp}\includegraphics[height=3.5cm]{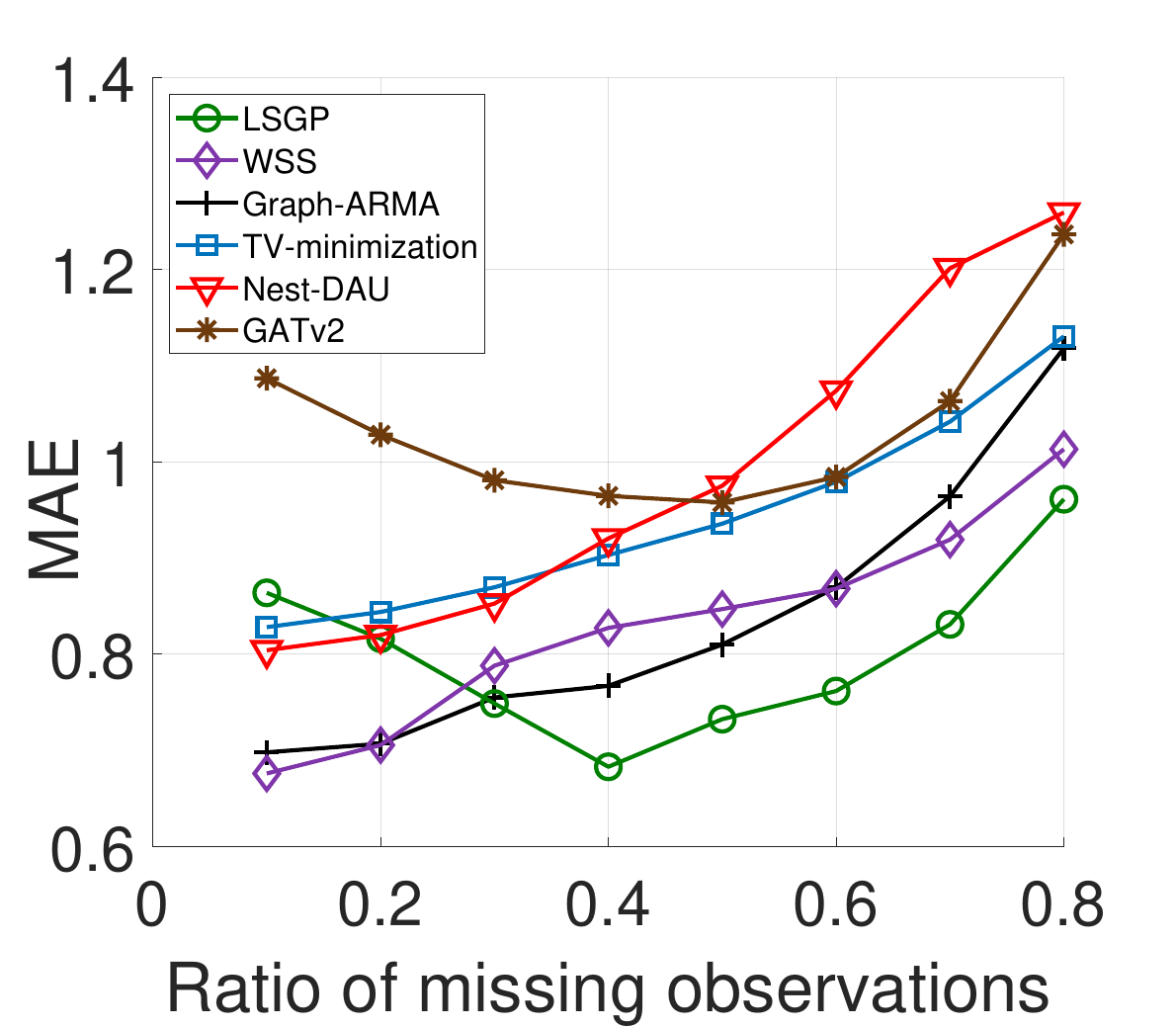}}\\
        \subfloat[COVID-19 (Structured data loss)]
         {\label{fig_mae_covid_nm}\includegraphics[height=3.5cm]{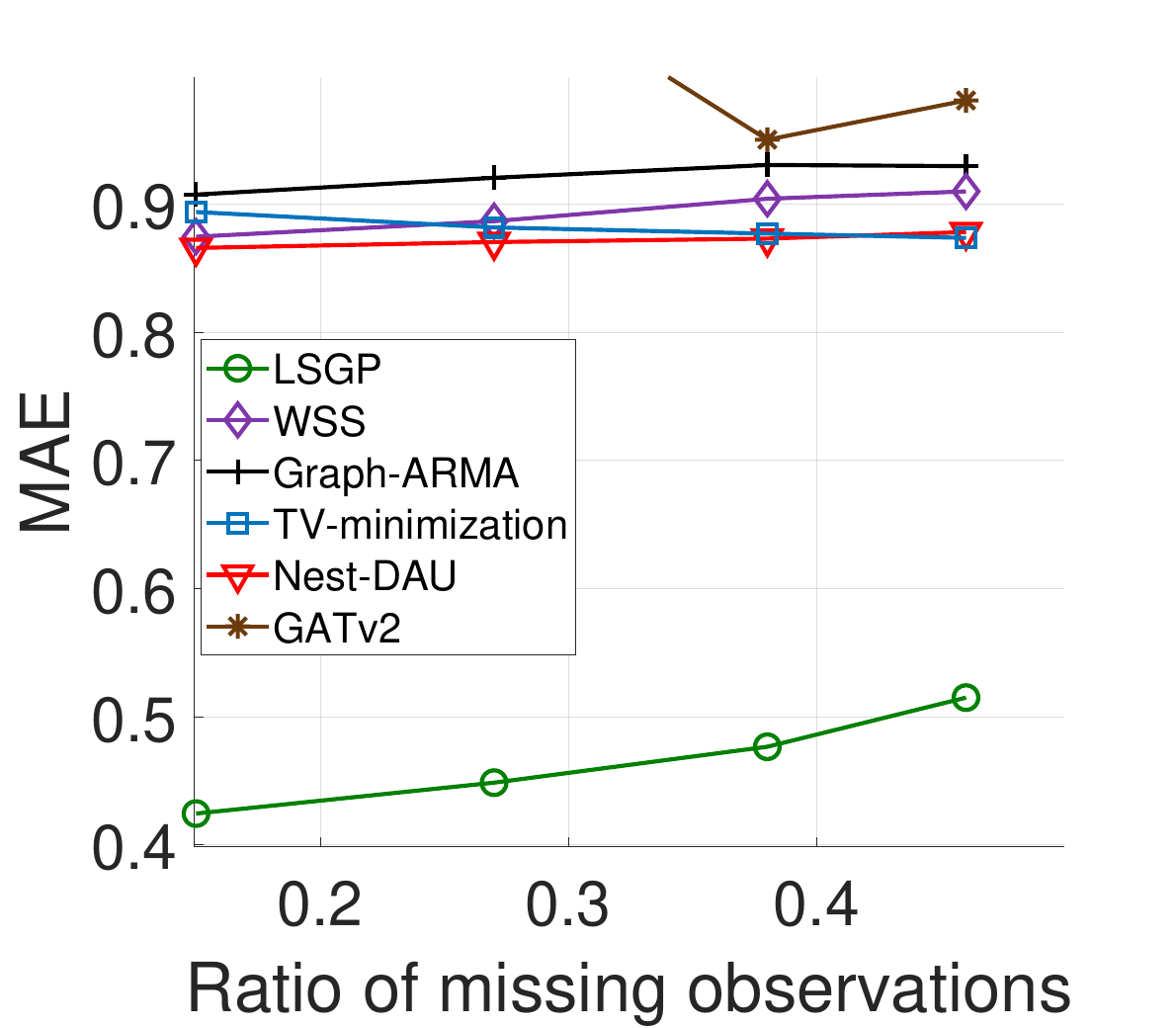}}
         \hspace{0.1cm}
        \subfloat[Mol\`ene (Structured data loss)]
        {\label{ffig_mae_molene_nm}\includegraphics[height=3.5cm]{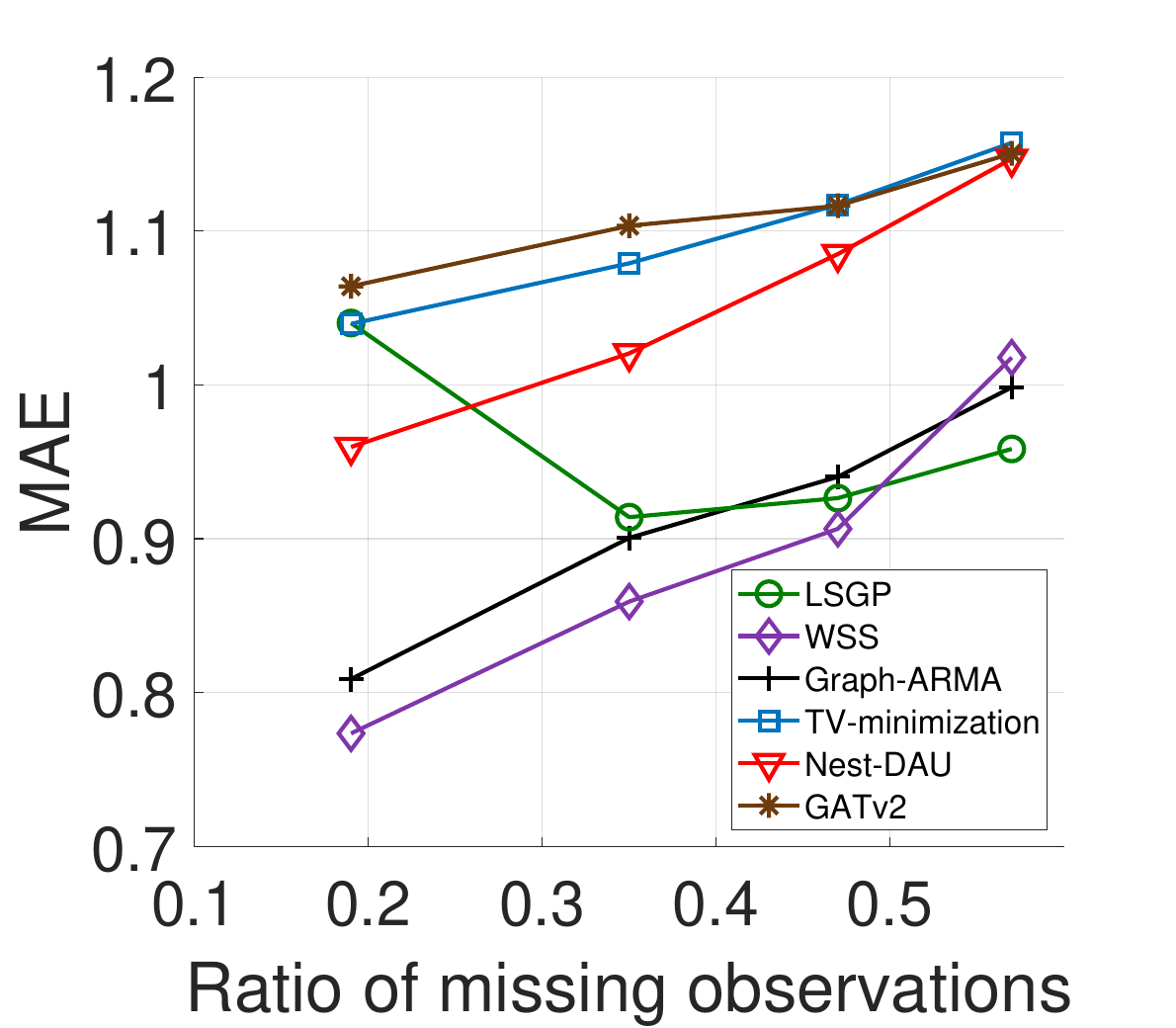}}\\
   \end{center}
    \caption{MAE of compared algorithms on COVID-19 and Mol\`ene data sets}   
   \label{fig_maes}
\end{figure}

\begin{figure}[ht]
  \begin{center}
    \hspace{-0.7cm}
       \subfloat[COVID-19 (Random data loss)]
       {\label{fig_mape_covid_sp}\includegraphics[height=3.5cm]{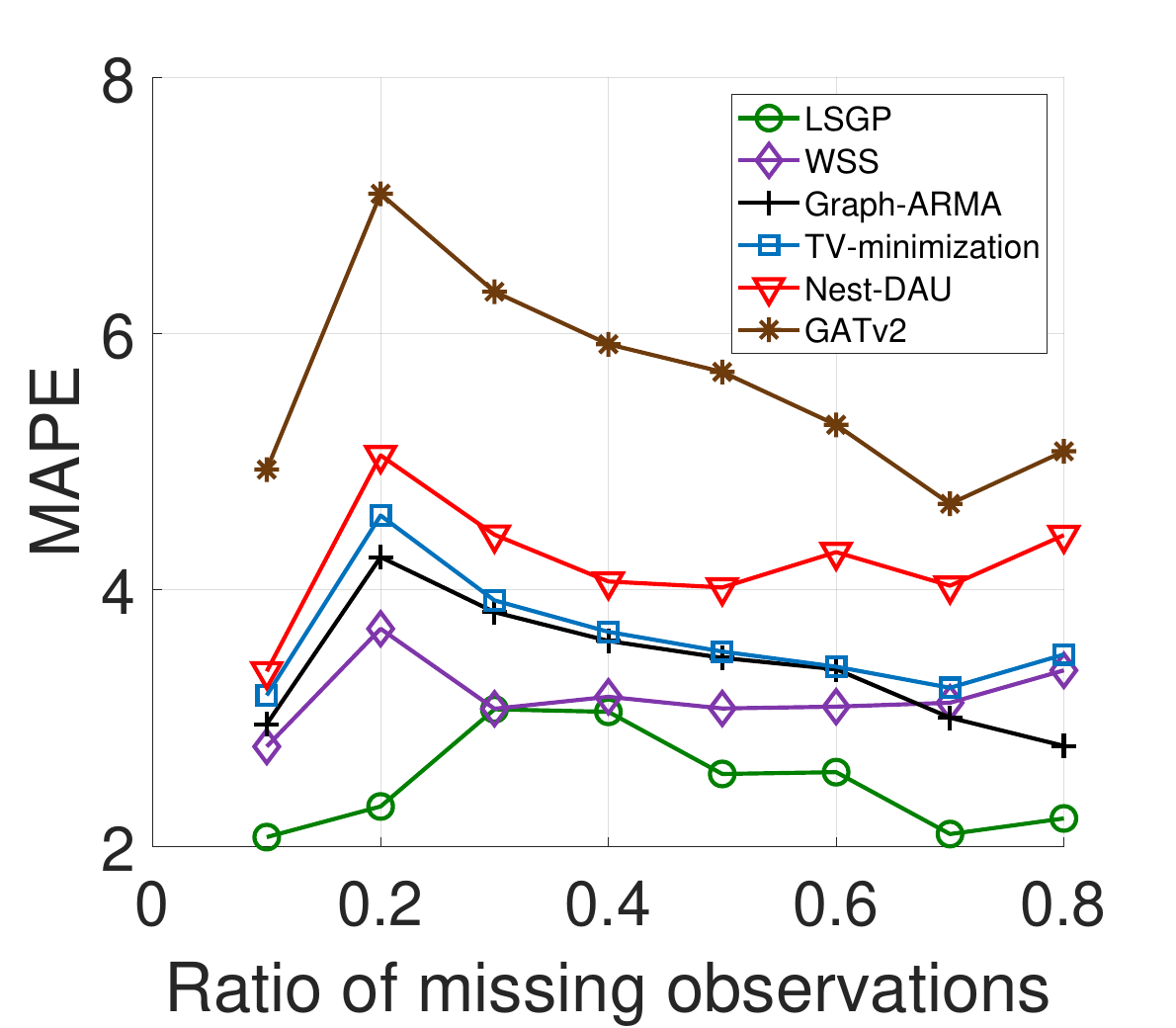}}
       \hspace{0.1cm}
      \subfloat[Mol\`ene (Random data loss)]
      {\label{fig_mape_molene_sp}\includegraphics[height=3.5cm]{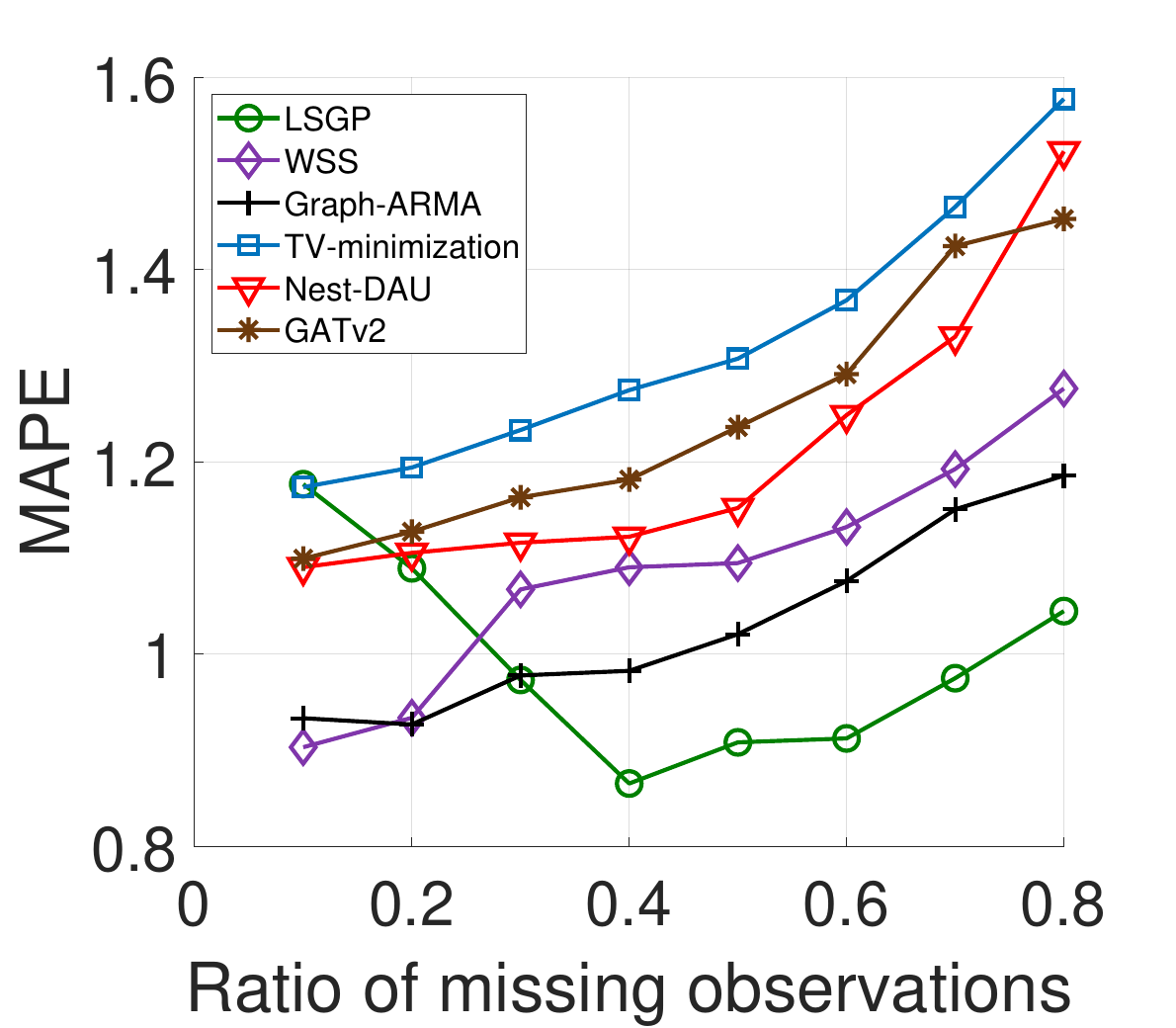}}\\
      \subfloat[COVID-19 (Structured data loss)]
 {\label{fig_mape_covid_nm}\includegraphics[height=3.5cm]{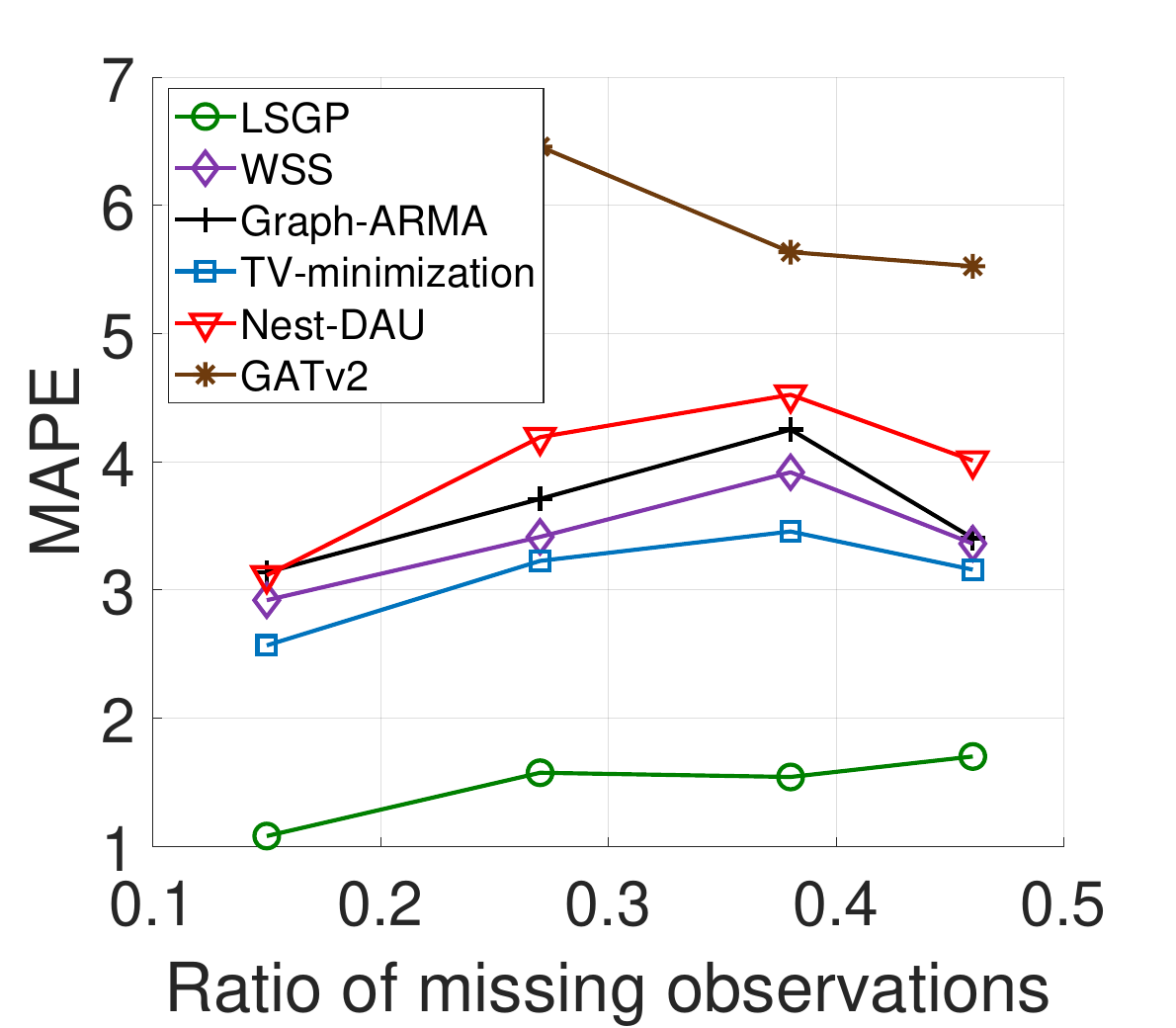}}
 \hspace{0.1cm}
\subfloat[Mol\`ene (Structured data loss)]
{\label{fig_mape_molene_nm}\includegraphics[height=3.5cm]{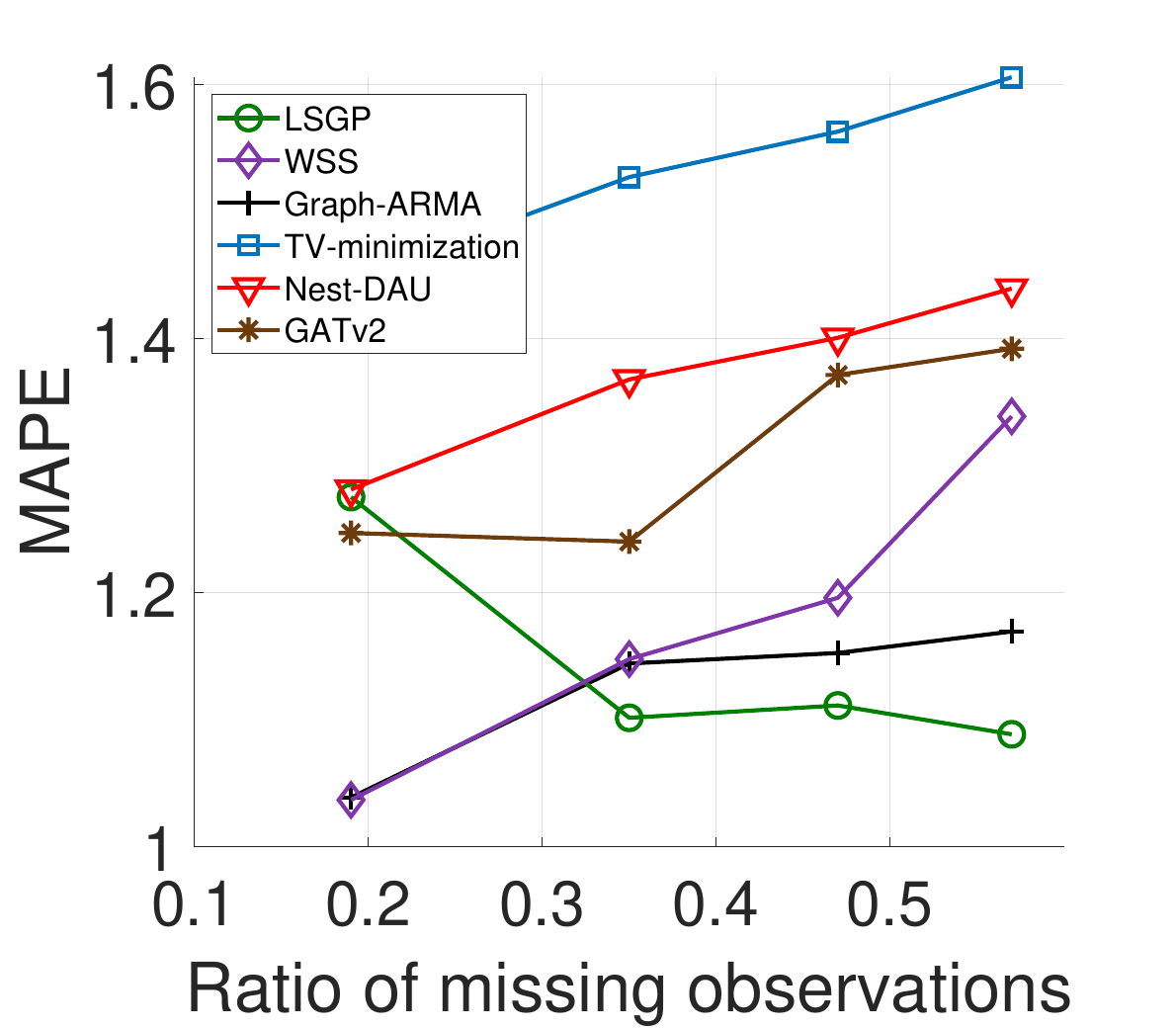}}\\
   \end{center}
    \caption{MAPE of compared algorithms on COVID-19 and Mol\`ene data sets}       
   \label{fig_mapes}
  \end{figure}

Next, in Figures \ref{fig_maes_partitioning} and \ref{fig_mapes_partitioning}, the methods are compared with respect to the MAE and the MAPE metrics on the NOAA and the USA COVID-19 data sets.

\begin{figure}[h!]
  \begin{center}
       \subfloat[NOAA data set]
      {\label{fig_mae_noaa_sp}\includegraphics[height=4.0cm]{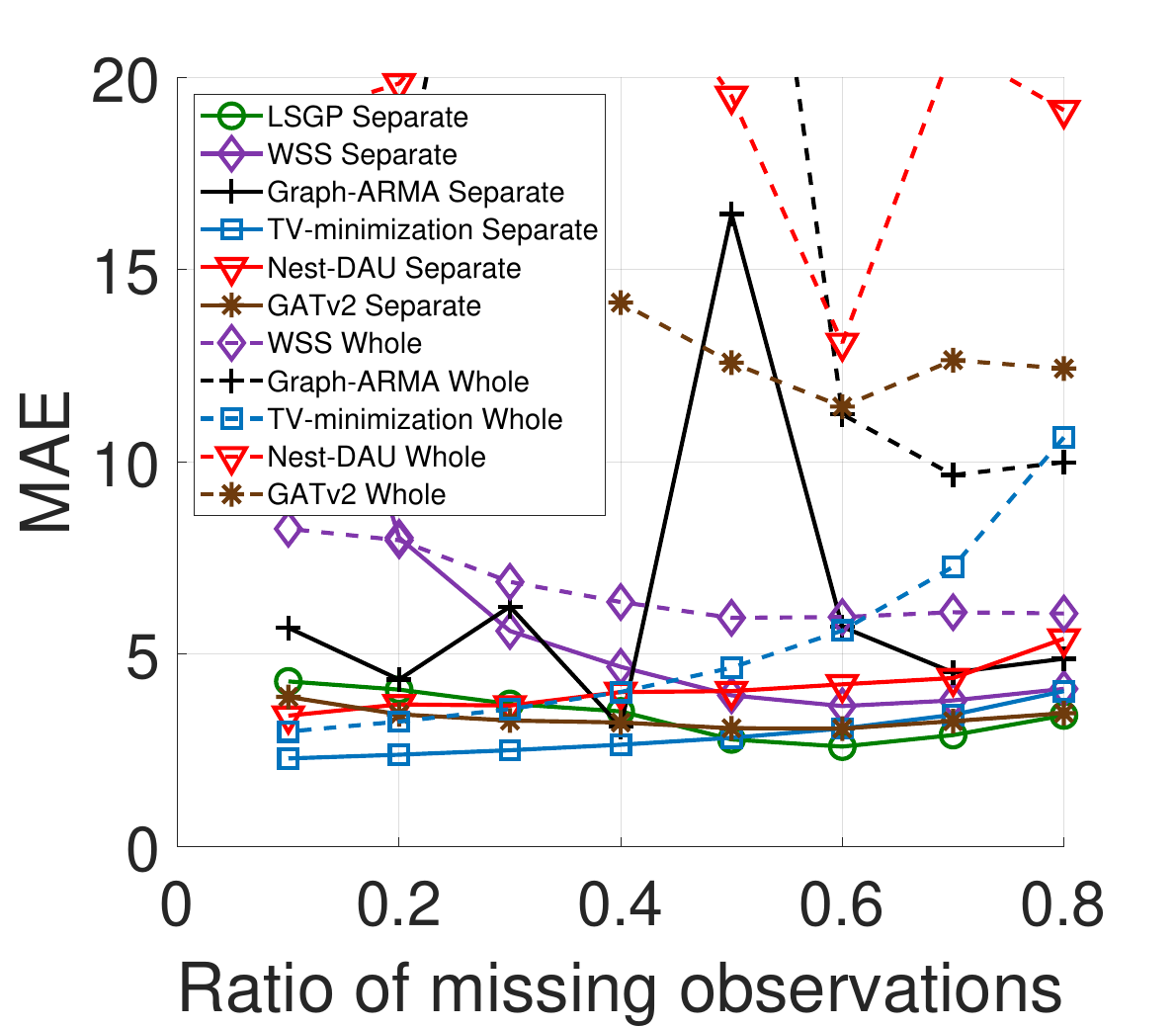}}
      \subfloat[USA COVID-19 data set]
      {\label{fig_mae_eus_sp}\includegraphics[height=4.0cm]{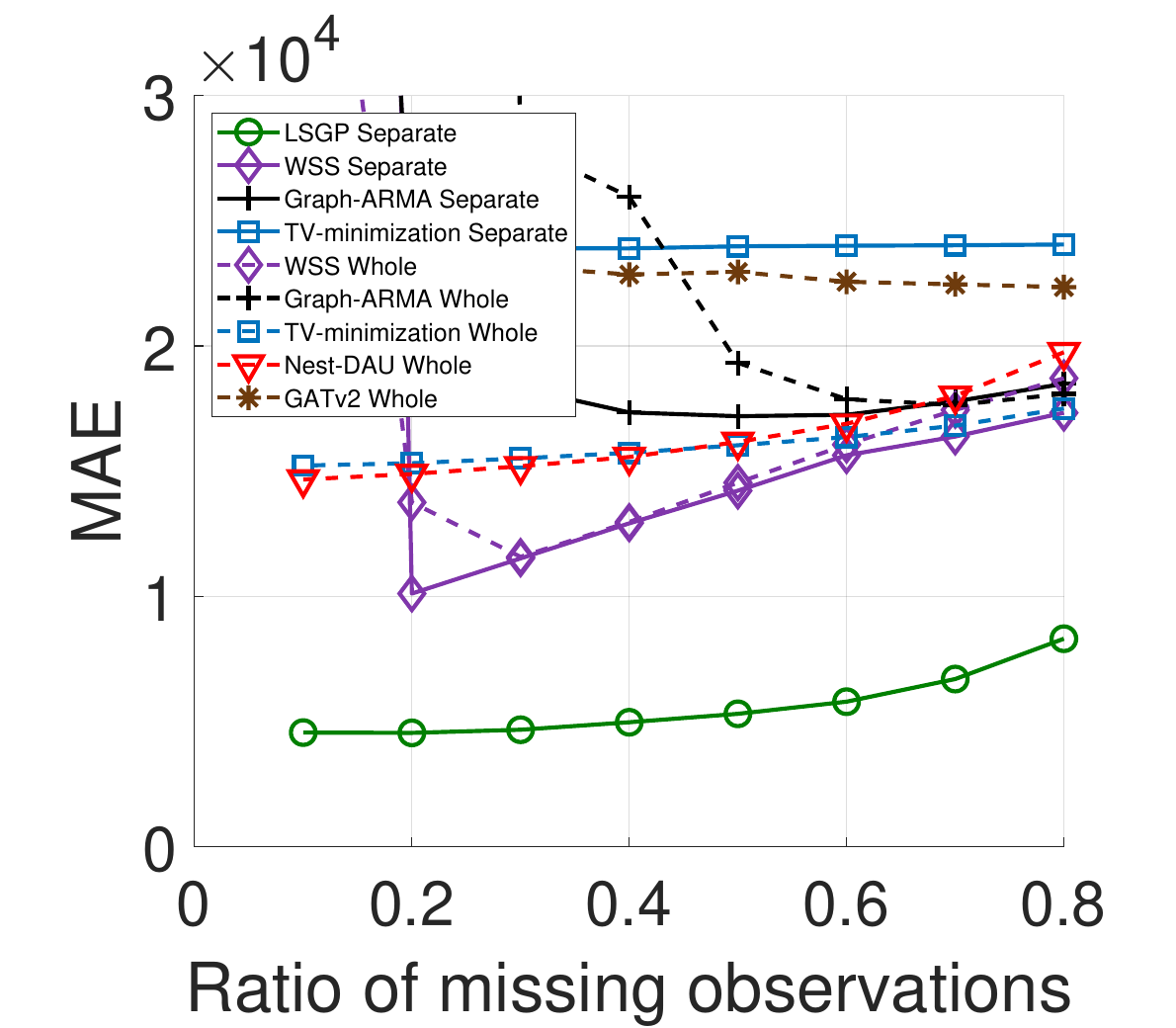}}
   \end{center}
    \caption{MAE of compared algorithms on NOAA and USA COVID-19 data sets}
   \label{fig_maes_partitioning}
  \end{figure}

\begin{figure}[h!]
  \begin{center}
       \subfloat[NOAA data set]
      {\label{fig_mape_noaa_sp}\includegraphics[height=4.0cm]{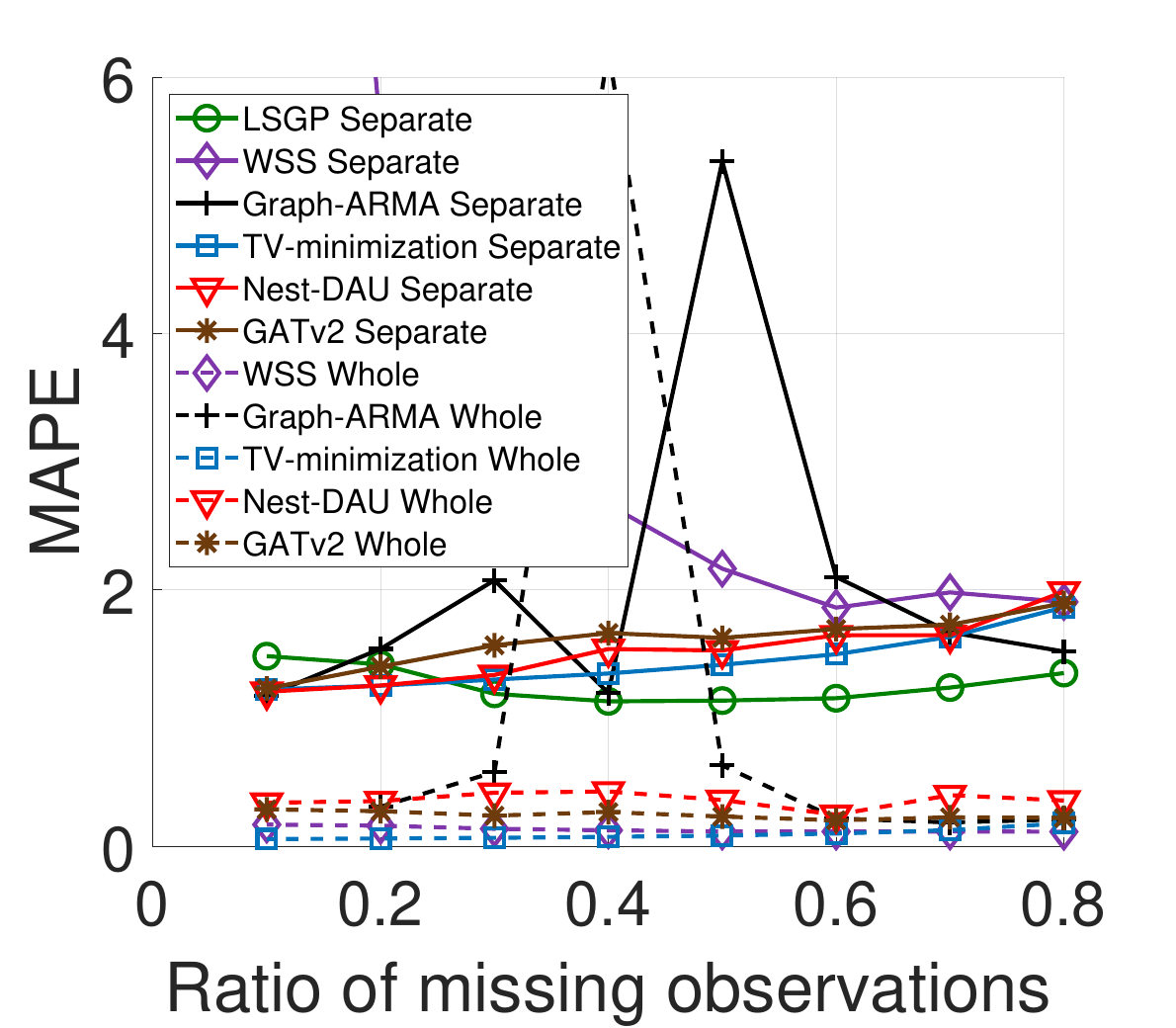}}
      \subfloat[USA COVID-19 data set]
      {\label{fig_mape_eus_sp}\includegraphics[height=4.0cm]{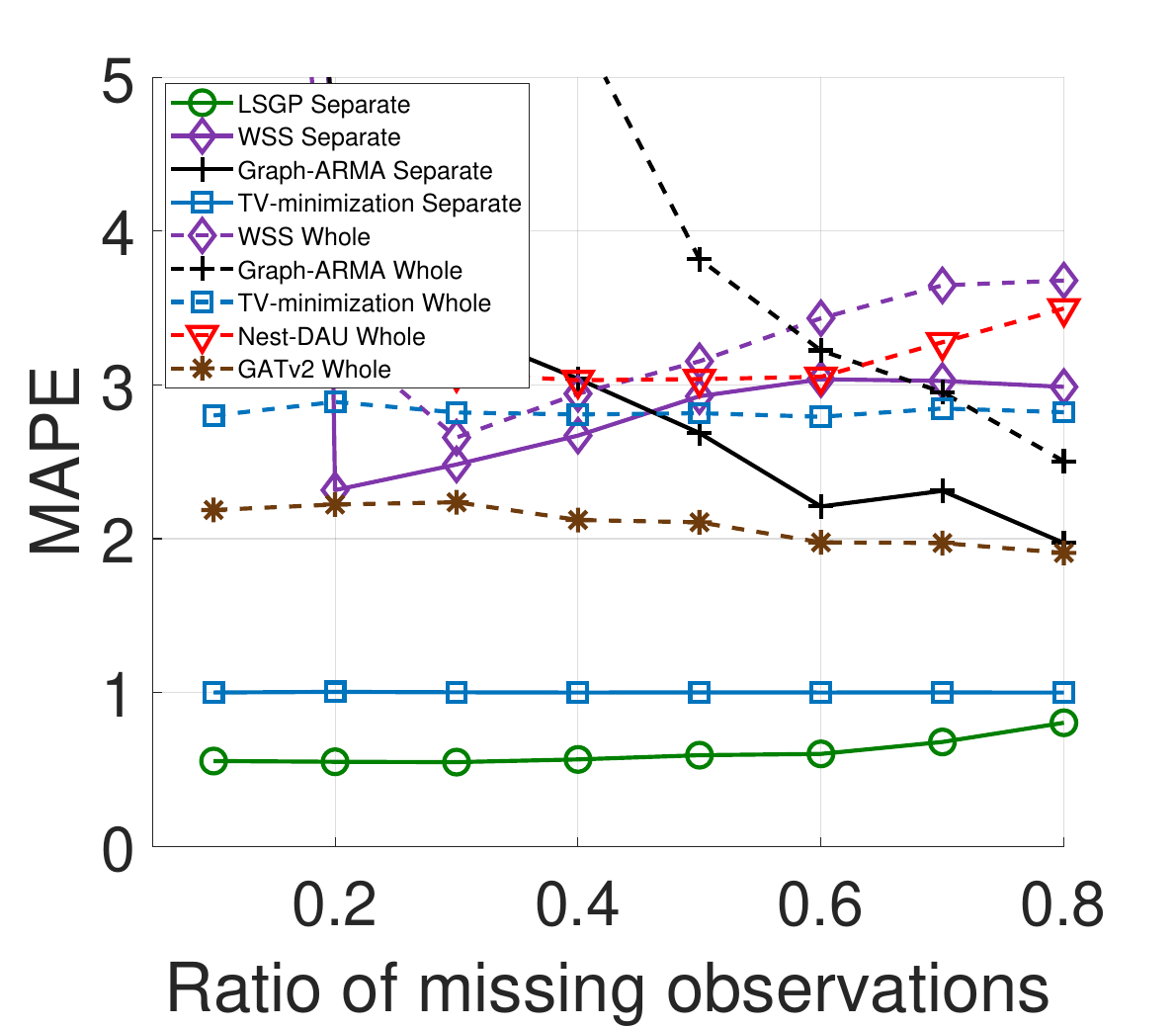}}
   \end{center}
    \caption{MAPE of compared algorithms on NOAA and USA COVID-19 data sets}
   \label{fig_mapes_partitioning}
  \end{figure}

Lastly, Table \ref{table:time_complexity} presents a comparison of the runtimes of the tested methods. The runtimes are measured on a laptop computer with 16 GB DDR5 RAM and 3.2 GHz CPU for the LSGP, WSS, Graph-ARMA, TV-minimization and the Nest-DAU methods. The experiments with the GATv2 (marked with \dag) method are conducted with the PyTorch library on Windows within the WSL environment on the same machine having RTX 3060 GPU. The tests of the Nest-DAU method on the USA COVID-19 data (marked with $*$) are done on a laptop computer with 32 GB DDR4 RAM and 2.6 GHz CPU and Linux operating system. The Nest-DAU method has a higher runtime in the separate graphs setting since the validation procedure has resulted in a larger number of layers. The difference in the runtimes of the GATv2 method between the separate and the whole graph settings is due to the memory management between the CPU and the GPU.

\begin{table}[th]
  \tiny
  \center
\begin{tabular}{|c||c|c||c|c|}
  \hline
  \textbf{Data Set} & \multicolumn{2}{c||}{\textbf{\textit{NOAA}}} & \multicolumn{2}{c|}{\textbf{\textit{USA COVID-19}}}\\
  \hline
  \textbf{Algorithm} & \textbf{Separate subgraphs} & \textbf{Whole graph} & \textbf{Separate subgraphs} & \textbf{Whole graph}\\
  \hline
  LSGP & 62.39 & - & 864.41& -\\ 
  \hline
  WSS & 5.14 & 20.30 & 10.80 & 101.92\\ 
  \hline
  Graph-ARMA & 4.90 & 18.34 & 10.18 & 90.12\\
  \hline
  TV-minimization & 338.84 & 408.20 & 184.29 & 239.16 \\
  \hline
  Nest-DAU  & 6145.51 & 1015.38 & - & 7269.87*\\
  \hline
  GATv2  & 1752.76 \dag& 1663.41 \dag& - & 518.40 \dag\\
  \hline
\end{tabular}
\vspace{0.1cm}
\caption{\label{table:time_complexity} Runtimes of compared methods (in seconds) on NOAA and USA COVID-19 data sets at 50\% missing observation ratio}
\end{table}

}

\end{document}